\documentclass{article}
\usepackage{iclr2027_conference,times}
\usepackage[T1]{fontenc}
\usepackage{amsmath,amssymb,amsthm,booktabs,graphicx,microtype}
\usepackage{algorithm,algpseudocode}
\usepackage{hyperref,url}
\usepackage{etoolbox}
\iclrfinalcopy
\apptocmd{\maketitle}{\lhead{Preprint}}{}{\PackageError{subgoal-preprint}{Cannot set the preprint header}{Check the ICLR style file.}}
\newtheorem{proposition}{Proposition}

\theoremstyle{definition}

\theoremstyle{remark}

\newcommand{\DiagN}{28}

\newcommand{\DiagSamples}{512}
\newcommand{\DiagSparseCosN}{0.841}
\newcommand{\DiagSparseCosR}{0.306}
\newcommand{\HorizonLogIdentities}{1377}
\newcommand{\MazeFixedDense}{73.75}
\newcommand{\MazeFixedSparse}{19.38}
\newcommand{\PFConfigs}{74}
\newcommand{\PFExactResamples}{3125}

\newcommand{\PFPairs}{36}

\newcommand{\PFSeeds}{5}

\newcommand{\PFTargetMillions}{10}
\newcommand{\PFUniquePairs}{32}
\newcommand{\ReacherConfigs}{64}

\newcommand{\ReacherPairedEntries}{512}

\newcommand{\ReacherSeeds}{8}
\newcommand{\ReacherWins}{60}

\newcommand{\MechCorrectionCheckpoints}{16}

\newcommand{\MechMatchedGradientPairs}{5}
\newcommand{\MechMatchedGradientPolicies}{10}
\newcommand{\MechMatchedRetentionCorrelation}{-0.800}

\newcommand{\MechRetentionMean}{27.9}

\newcommand{\VThreeFixedAAAAMean}{94.3}
\newcommand{\VThreeFixedAAAASd}{2.2}
\newcommand{\VThreeFixedAAABMean}{92.6}
\newcommand{\VThreeFixedAAABSd}{4.0}
\newcommand{\VThreeFixedAAACMean}{95.4}
\newcommand{\VThreeFixedAAACSd}{3.2}
\newcommand{\VThreeFixedABAAMean}{32.1}
\newcommand{\VThreeFixedABAASd}{32.1}
\newcommand{\VThreeFixedABABMean}{54.3}
\newcommand{\VThreeFixedABABSd}{37.8}
\newcommand{\VThreeFixedABACMean}{48.6}
\newcommand{\VThreeFixedABACSd}{41.8}
\newcommand{\VThreeFixedACAAMean}{96.5}
\newcommand{\VThreeFixedACAASd}{2.2}
\newcommand{\VThreeFixedACABMean}{97.1}
\newcommand{\VThreeFixedACABSd}{1.7}
\newcommand{\VThreeFixedACACMean}{97.3}
\newcommand{\VThreeFixedACACSd}{2.3}
\newcommand{\VThreeFixedADAAMean}{76.0}
\newcommand{\VThreeFixedADAASd}{30.3}
\newcommand{\VThreeFixedADABMean}{95.4}
\newcommand{\VThreeFixedADABSd}{7.9}
\newcommand{\VThreeFixedADACMean}{96.1}
\newcommand{\VThreeFixedADACSd}{3.4}
\newcommand{\VThreeFixedAEABMean}{60.6}
\newcommand{\VThreeFixedAEABSd}{18.8}
\newcommand{\VThreeFixedAEACMean}{73.8}
\newcommand{\VThreeFixedAEACSd}{9.5}
\newcommand{\VThreeFixedAFABMean}{35.0}
\newcommand{\VThreeFixedAFABSd}{39.9}
\newcommand{\VThreeFixedAFACMean}{19.4}
\newcommand{\VThreeFixedAFACSd}{31.8}
\newcommand{\VThreeFixedAGABMean}{26.9}
\newcommand{\VThreeFixedAGABSd}{6.3}
\newcommand{\VThreeFixedAGACMean}{29.3}
\newcommand{\VThreeFixedAGACSd}{6.5}
\newcommand{\VThreeFixedAHABMean}{14.8}
\newcommand{\VThreeFixedAHABSd}{15.1}
\newcommand{\VThreeFixedAHACMean}{10.4}
\newcommand{\VThreeFixedAHACSd}{11.4}
\newcommand{\VThreeAucAAAAMean}{-84.1}

\newcommand{\VThreeAucAAABMean}{508.3}

\newcommand{\VThreeAucABAAMean}{-82.2}

\newcommand{\VThreeAucABABMean}{75.5}

\newcommand{\VThreeAucACAAMean}{69.4}

\newcommand{\VThreeAucACABMean}{875.4}

\newcommand{\VThreeAucADAAMean}{226.6}

\newcommand{\VThreeAucADABMean}{560.4}

\newcommand{\VThreeSharingRows}{180}
\newcommand{\VThreeSharingAAAAMean}{87.7}
\newcommand{\VThreeSharingAAAASd}{4.5}
\newcommand{\VThreeSharingAAABMean}{81.6}
\newcommand{\VThreeSharingAAABSd}{8.2}
\newcommand{\VThreeSharingAAACMean}{86.7}
\newcommand{\VThreeSharingAAACSd}{5.5}
\newcommand{\VThreeSharingABAAMean}{12.7}
\newcommand{\VThreeSharingABAASd}{10.6}
\newcommand{\VThreeSharingABABMean}{24.3}
\newcommand{\VThreeSharingABABSd}{27.3}
\newcommand{\VThreeSharingABACMean}{9.8}
\newcommand{\VThreeSharingABACSd}{11.2}
\newcommand{\VThreeSharingACAAMean}{33.0}
\newcommand{\VThreeSharingACAASd}{22.3}
\newcommand{\VThreeSharingACABMean}{17.0}
\newcommand{\VThreeSharingACABSd}{21.6}
\newcommand{\VThreeSharingACACMean}{39.0}
\newcommand{\VThreeSharingACACSd}{27.6}
\newcommand{\VThreeSharingADAAMean}{27.5}
\newcommand{\VThreeSharingADAASd}{32.6}
\newcommand{\VThreeSharingADABMean}{9.0}
\newcommand{\VThreeSharingADABSd}{22.3}
\newcommand{\VThreeSharingADACMean}{7.5}
\newcommand{\VThreeSharingADACSd}{22.0}
\newcommand{\VThreeSharingAEAAMean}{21.1}
\newcommand{\VThreeSharingAEAASd}{7.2}
\newcommand{\VThreeSharingAEABMean}{14.8}
\newcommand{\VThreeSharingAEABSd}{7.0}
\newcommand{\VThreeSharingAEACMean}{22.6}
\newcommand{\VThreeSharingAEACSd}{6.7}
\newcommand{\VThreeSharingAFAAMean}{6.8}
\newcommand{\VThreeSharingAFAASd}{9.2}
\newcommand{\VThreeSharingAFABMean}{7.3}
\newcommand{\VThreeSharingAFABSd}{7.7}
\newcommand{\VThreeSharingAFACMean}{9.5}
\newcommand{\VThreeSharingAFACSd}{10.6}
\newcommand{\VThreeSummaryAAMean}{-0.1}

\newcommand{\VThreeSummaryAAN}{16}
\newcommand{\VThreeSummaryABMean}{-151.3}

\newcommand{\VThreeSummaryABN}{8}
\newcommand{\VThreeSummaryACMean}{1136.0}

\newcommand{\VThreeSummaryACN}{8}
\newcommand{\VThreeSummaryADMean}{-0.1}

\newcommand{\VThreeSummaryADN}{16}
\newcommand{\VThreeSummaryAEMean}{2.8}

\newcommand{\VThreeSummaryAEN}{16}
\newcommand{\VThreeSummaryAFMean}{-2.0}

\newcommand{\VThreeSummaryAFN}{16}
\newcommand{\VThreeSummaryAGMean}{77.1}

\newcommand{\VThreeSummaryAGN}{8}
\newcommand{\VThreeSummaryAHMean}{-209.4}

\newcommand{\VThreeSummaryAHN}{8}
\newcommand{\VThreeSummaryAIMean}{-1.8}

\newcommand{\VThreeSummaryAIN}{16}
\newcommand{\VThreeSummaryAJMean}{5.3}

\newcommand{\VThreeSummaryAJN}{16}
\newcommand{\VThreeSummaryAKMean}{-16.8}

\newcommand{\VThreeSummaryAKN}{16}
\newcommand{\VThreeSummaryALMean}{-25.6}

\newcommand{\VThreeSummaryALN}{16}
\newcommand{\VThreeSummaryAMMean}{-10.2}

\newcommand{\VThreeSummaryAMN}{16}
\newcommand{\VThreeSummaryANMean}{-1.1}

\newcommand{\VThreeSummaryANN}{4}
\newcommand{\VThreeSummaryAOMean}{187.0}

\newcommand{\VThreeSummaryAON}{2}
\newcommand{\VThreeSummaryAPMean}{584.0}

\newcommand{\VThreeSummaryAPN}{2}
\newcommand{\VThreeSummaryAQMean}{-0.5}

\newcommand{\VThreeSummaryAQN}{4}
\newcommand{\VThreeSummaryARMean}{-16.7}

\newcommand{\VThreeSummaryARN}{4}
\newcommand{\VThreeShortAAMean}{13.1}

\newcommand{\VThreeShortABMean}{-15.6}

\newcommand{\VThreeShortACMean}{2.4}

\newcommand{\VThreeShortADMean}{-4.4}

\newcommand{\VThreeShortAEMean}{344.7}

\newcommand{\VThreeShortAFMean}{19.6}

\newcommand{\VThreeShortAGMean}{350.5}

\newcommand{\VThreeShortAHMean}{19.2}

\newcommand{\VThreeShortAIMean}{29.3}

\newcommand{\VThreeShortAJMean}{0.8}

\newcommand{\VThreeShortAKMean}{817.6}

\newcommand{\VThreeShortALMean}{0.8}

\newcommand{\VThreeShortAMMean}{2.8}

\newcommand{\VThreeShortANMean}{-2.9}

\newcommand{\VThreeShortAOMean}{-5.6}

\newcommand{\VThreeShortAPMean}{-2.9}

\newcommand{\VThreeShortAQMean}{0.1}

\newcommand{\VThreeShortARMean}{-0.6}

\newcommand{\VThreeShortASMean}{0.8}

\newcommand{\VThreeShortATMean}{-60.4}

\newcommand{\VThreeFullProjectionReacherAAMean}{0.1}

\newcommand{\VThreeFullProjectionReacherABMean}{-2.5}

\newcommand{\VThreeFullProjectionReacherACMean}{2.9}

\newcommand{\VThreeFullProjectionReacherADMean}{-0.9}

\newcommand{\VThreeFullProjectionReacherAEMean}{18.5}

\newcommand{\VThreeFullProjectionReacherAFMean}{34.5}

\newcommand{\VThreeFullProjectionReacherAGMean}{-11.1}

\newcommand{\VThreeFullProjectionReacherAHMean}{-8.8}

\newcommand{\VThreeFullProjectionReacherAIMean}{-4.0}

\newcommand{\VThreeFullProjectionReacherAJMean}{11.3}

\newcommand{\VThreeFullProjectionReacherAKMean}{6.6}

\newcommand{\VThreeFullProjectionReacherALMean}{-17.5}

\newcommand{\VThreeFullProjectionReacherAMMean}{0.1}

\newcommand{\VThreeFullProjectionReacherANMean}{-1.5}

\newcommand{\VThreeFullProjectionReacherAOMean}{6.1}

\newcommand{\VThreeFullProjectionReacherAPMean}{-1.6}

\newcommand{\VThreeFullProjectionReacherAQMean}{0.0}

\newcommand{\VThreeFullProjectionReacherARMean}{0.0}

\newcommand{\VThreeFullProjectionReacherASMean}{0.0}

\newcommand{\VThreeFullProjectionReacherATMean}{-1.0}

\newcommand{\VThreeFullProjectionReacherAUMean}{38.4}

\newcommand{\VThreeFullProjectionReacherAVMean}{14.1}

\newcommand{\VThreeFullProjectionReacherAWMean}{8.0}

\newcommand{\VThreeFullProjectionReacherAXMean}{-1.9}

\newcommand{\VThreeFullProjectionReacherAYMean}{-25.3}

\newcommand{\VThreeFullProjectionReacherAZMean}{-1.5}

\newcommand{\VThreeFullProjectionReacherBAMean}{7.3}

\newcommand{\VThreeFullProjectionReacherBBMean}{-4.6}

\newcommand{\VThreeFullProjectionReacherBCMean}{3.3}

\newcommand{\VThreeFullProjectionReacherBDMean}{-4.6}

\newcommand{\VThreeFullProjectionReacherBEMean}{6.4}

\newcommand{\VThreeFullProjectionReacherBFMean}{-15.0}

\newcommand{\VThreeFullProjectionMazeavailableAAMean}{-63.1}

\newcommand{\VThreeFullProjectionMazeavailableABMean}{21.3}

\newcommand{\VThreeFullProjectionMazeavailableACMean}{13.8}

\newcommand{\VThreeFullProjectionMazeavailableADMean}{8.1}

\newcommand{\VThreeFullProjectionMazeavailableAEMean}{-8.1}

\newcommand{\VThreeFullProjectionMazeavailableAFMean}{-0.6}

\newcommand{\VThreeFullProjectionMazeavailableAGMean}{5.0}

\newcommand{\VThreeFullProjectionMazeavailableAHMean}{8.8}

\newcommand{\VThreeFullProjectionMazeavailableAIMean}{-4.3}

\newcommand{\VThreeFullProjectionMazeavailableAJMean}{0.8}

\newcommand{\VThreeFullProjectionMazeavailableAKMean}{2.4}

\newcommand{\VThreeFullProjectionMazeavailableALMean}{-6.4}

\newcommand{\VThreeFullProjectionMazeavailableAMMean}{-0.6}

\newcommand{\VThreeFullProjectionMazeavailableANMean}{0.4}

\newcommand{\VThreeFullProjectionMazeavailableAOMean}{-7.0}

\newcommand{\VThreeFullProjectionMazeavailableAPMean}{-1.9}

\newcommand{\VThreeFullProjectionPushfallAAMean}{112.2}

\newcommand{\VThreeFullProjectionPushfallABMean}{174.9}

\newcommand{\VThreeFullProjectionPushfallACMean}{276.3}

\newcommand{\VThreeFullProjectionPushfallADMean}{421.1}

\newcommand{\VThreeFullProjectionPushfallAEMean}{-345.2}

\newcommand{\VThreeFullProjectionPushfallAFMean}{47.4}

\newcommand{\VThreeFullProjectionPushfallAGMean}{-1891.3}

\newcommand{\VThreeFullProjectionPushfallAHMean}{-2218.2}

\newcommand{\VThreeFullProjectionPushfallAIMean}{243.9}

\newcommand{\VThreeFullProjectionPushfallAJMean}{438.7}

\newcommand{\VThreeFullProjectionPushfallAKMean}{-905.7}

\newcommand{\VThreeFullProjectionPushfallALMean}{-144.9}

\newcommand{\VThreeFullProjectionPushfallAMMean}{164.0}

\newcommand{\VThreeFullProjectionPushfallANMean}{2043.5}

\newcommand{\VThreeFullProjectionPushfallAOMean}{650.9}

\newcommand{\VThreeFullProjectionPushfallAPMean}{-126.5}

\newcommand{\VThreeFullNextgapReacherAAMean}{2.0}

\newcommand{\VThreeFullNextgapReacherABMean}{0.3}

\newcommand{\VThreeFullNextgapReacherACMean}{0.9}

\newcommand{\VThreeFullNextgapReacherADMean}{2.1}

\newcommand{\VThreeFullNextgapReacherAEMean}{-5.0}

\newcommand{\VThreeFullNextgapReacherAFMean}{10.1}

\newcommand{\VThreeFullNextgapReacherAGMean}{26.9}

\newcommand{\VThreeFullNextgapReacherAHMean}{16.4}

\newcommand{\VThreeFullNextgapReacherAIMean}{-12.6}

\newcommand{\VThreeFullNextgapReacherAJMean}{12.5}

\newcommand{\VThreeFullNextgapReacherAKMean}{-3.9}

\newcommand{\VThreeFullNextgapReacherALMean}{-4.4}

\newcommand{\VThreeFullNextgapReacherAMMean}{0.6}

\newcommand{\VThreeFullNextgapReacherANMean}{1.3}

\newcommand{\VThreeFullNextgapReacherAOMean}{-3.0}

\newcommand{\VThreeFullNextgapReacherAPMean}{1.3}

\newcommand{\VThreeFullNextgapReacherAQMean}{0.4}

\newcommand{\VThreeFullNextgapReacherARMean}{-0.6}

\newcommand{\VThreeFullNextgapReacherASMean}{-0.6}

\newcommand{\VThreeFullNextgapReacherATMean}{1.0}

\newcommand{\VThreeFullNextgapReacherAUMean}{-0.8}

\newcommand{\VThreeFullNextgapReacherAVMean}{-1.3}

\newcommand{\VThreeFullNextgapReacherAWMean}{1.9}

\newcommand{\VThreeFullNextgapReacherAXMean}{1.5}

\newcommand{\VThreeFullNextgapReacherAYMean}{-13.3}

\newcommand{\VThreeFullNextgapReacherAZMean}{6.8}

\newcommand{\VThreeFullNextgapReacherBAMean}{3.8}

\newcommand{\VThreeFullNextgapReacherBBMean}{3.6}

\newcommand{\VThreeFullNextgapReacherBCMean}{-3.1}

\newcommand{\VThreeFullNextgapReacherBDMean}{-2.0}

\newcommand{\VThreeFullNextgapReacherBEMean}{-20.9}

\newcommand{\VThreeFullNextgapReacherBFMean}{21.3}

\newcommand{\VThreeFullNextgapMazeavailableAAMean}{15.6}

\newcommand{\VThreeFullNextgapMazeavailableABMean}{-8.8}

\newcommand{\VThreeFullNextgapMazeavailableACMean}{4.4}

\newcommand{\VThreeFullNextgapMazeavailableADMean}{-2.5}

\newcommand{\VThreeFullNextgapMazeavailableAEMean}{0.0}

\newcommand{\VThreeFullNextgapMazeavailableAFMean}{5.0}

\newcommand{\VThreeFullNextgapMazeavailableAGMean}{-6.3}

\newcommand{\VThreeFullNextgapMazeavailableAHMean}{2.5}

\newcommand{\VThreeFullNextgapMazeavailableAIMean}{-1.8}

\newcommand{\VThreeFullNextgapMazeavailableAJMean}{-1.0}

\newcommand{\VThreeFullNextgapMazeavailableAKMean}{1.3}

\newcommand{\VThreeFullNextgapMazeavailableALMean}{-3.1}

\newcommand{\VThreeFullNextgapMazeavailableAMMean}{-12.5}

\newcommand{\VThreeFullNextgapMazeavailableANMean}{3.1}

\newcommand{\VThreeFullNextgapMazeavailableAOMean}{-3.0}

\newcommand{\VThreeFullNextgapMazeavailableAPMean}{4.6}

\newcommand{\VThreeFullNextgapPushfallAAMean}{-1226.5}

\newcommand{\VThreeFullNextgapPushfallABMean}{-1053.1}

\newcommand{\VThreeFullNextgapPushfallACMean}{3.5}

\newcommand{\VThreeFullNextgapPushfallADMean}{-332.5}

\newcommand{\VThreeFullNextgapPushfallAEMean}{2189.9}

\newcommand{\VThreeFullNextgapPushfallAFMean}{268.7}

\newcommand{\VThreeFullNextgapPushfallAGMean}{3169.4}

\newcommand{\VThreeFullNextgapPushfallAHMean}{3872.1}

\newcommand{\VThreeFullNextgapPushfallAIMean}{360.6}

\newcommand{\VThreeFullNextgapPushfallAJMean}{15.1}

\newcommand{\VThreeFullNextgapPushfallAKMean}{871.8}

\newcommand{\VThreeFullNextgapPushfallALMean}{150.5}

\newcommand{\VThreeFullNextgapPushfallAMMean}{-529.9}

\newcommand{\VThreeFullNextgapPushfallANMean}{-34.4}

\newcommand{\VThreeFullNextgapPushfallAOMean}{149.2}

\newcommand{\VThreeFullNextgapPushfallAPMean}{2.8}

\newcommand{\VThreeFullQueryReacherAAMean}{-68.4}

\newcommand{\VThreeFullQueryReacherABMean}{-13.1}

\newcommand{\VThreeFullQueryReacherACMean}{-12.1}

\newcommand{\VThreeFullQueryReacherADMean}{-4.5}

\newcommand{\VThreeFullQueryReacherAEMean}{7.9}

\newcommand{\VThreeFullQueryReacherAFMean}{-39.6}

\newcommand{\VThreeFullQueryReacherAGMean}{-24.4}

\newcommand{\VThreeFullQueryReacherAHMean}{-11.4}

\newcommand{\VThreeFullQueryReacherAIMean}{-13.9}

\newcommand{\VThreeFullQueryReacherAJMean}{-39.3}

\newcommand{\VThreeFullQueryReacherAKMean}{-43.9}

\newcommand{\VThreeFullQueryReacherALMean}{-25.4}

\newcommand{\VThreeFullQueryReacherAMMean}{-1.9}

\newcommand{\VThreeFullQueryReacherANMean}{-0.4}

\newcommand{\VThreeFullQueryReacherAOMean}{-8.1}

\newcommand{\VThreeFullQueryReacherAPMean}{-2.6}

\newcommand{\VThreeFullQueryReacherAQMean}{-37.1}

\newcommand{\VThreeFullQueryReacherARMean}{-1.5}

\newcommand{\VThreeFullQueryReacherASMean}{-1.0}

\newcommand{\VThreeFullQueryReacherATMean}{-0.1}

\newcommand{\VThreeFullQueryReacherAUMean}{-6.9}

\newcommand{\VThreeFullQueryReacherAVMean}{-20.0}

\newcommand{\VThreeFullQueryReacherAWMean}{-2.4}

\newcommand{\VThreeFullQueryReacherAXMean}{-0.1}

\newcommand{\VThreeFullQueryReacherAYMean}{-57.0}

\newcommand{\VThreeFullQueryReacherAZMean}{-46.9}

\newcommand{\VThreeFullQueryReacherBAMean}{-46.5}

\newcommand{\VThreeFullQueryReacherBBMean}{0.6}

\newcommand{\VThreeFullQueryReacherBCMean}{-5.0}

\newcommand{\VThreeFullQueryReacherBDMean}{2.3}

\newcommand{\VThreeFullQueryReacherBEMean}{-21.1}

\newcommand{\VThreeFullQueryReacherBFMean}{-28.8}

\newcommand{\VThreeFullQueryMazeavailableAAMean}{-6.3}

\newcommand{\VThreeFullQueryMazeavailableABMean}{-65.6}

\newcommand{\VThreeFullQueryMazeavailableACMean}{-42.5}

\newcommand{\VThreeFullQueryMazeavailableADMean}{-21.9}

\newcommand{\VThreeFullQueryMazeavailableAEMean}{0.0}

\newcommand{\VThreeFullQueryMazeavailableAFMean}{-18.8}

\newcommand{\VThreeFullQueryMazeavailableAGMean}{-21.3}

\newcommand{\VThreeFullQueryMazeavailableAHMean}{-4.4}

\newcommand{\VThreeFullQueryMazeavailableAIMean}{-19.0}

\newcommand{\VThreeFullQueryMazeavailableAJMean}{-15.0}

\newcommand{\VThreeFullQueryMazeavailableAKMean}{-14.1}

\newcommand{\VThreeFullQueryMazeavailableALMean}{-4.8}

\newcommand{\VThreeFullQueryMazeavailableAMMean}{-17.3}

\newcommand{\VThreeFullQueryMazeavailableANMean}{-11.6}

\newcommand{\VThreeFullQueryMazeavailableAOMean}{-9.1}

\newcommand{\VThreeFullQueryMazeavailableAPMean}{3.1}

\newcommand{\VThreeHiroSparseMean}{20.1}
\newcommand{\VThreeHiroSparseLo}{2.5}
\newcommand{\VThreeHiroSparseHi}{42.1}
\newcommand{\VThreeHiroSparseLeft}{96.1}
\newcommand{\VThreeHiroSparseRight}{76.0}

\newcommand{\VThreeLongSparseLeft}{95.4}

\newcommand{\VThreeProjectionPositiveMean}{38.4}
\newcommand{\VThreeProjectionPositiveLo}{15.0}
\newcommand{\VThreeProjectionPositiveHi}{63.9}

\newcommand{\VThreeMazeRows}{768}
\newcommand{\VThreeLiteralRows}{32}

\newcommand{\VFourPushHiroSuccess}{0.0}

\newcommand{\VFourPushCorrSuccess}{14.0}

\newcommand{\VFourPushCorrLo}{0.0}
\newcommand{\VFourPushCorrHi}{32.4}

\newcommand{\VFourFallHiroSuccess}{9.6}

\newcommand{\VFourFallCorrSuccess}{32.8}

\newcommand{\VFourFallCorrDelta}{23.2}
\newcommand{\VFourFallCorrLo}{8.8}
\newcommand{\VFourFallCorrHi}{39.2}

\newcommand{\VFourNextFallMean}{32.8}
\newcommand{\VFourNextFallLo}{14.4}
\newcommand{\VFourNextFallHi}{47.6}

\newcommand{\VFourNextReacherMean}{26.9}
\newcommand{\VFourNextReacherLo}{12.0}
\newcommand{\VFourNextReacherHi}{40.8}

\newcommand{\VFourShortFallCorrMean}{19.2}
\newcommand{\VFourShortFallCorrLo}{8.0}
\newcommand{\VFourShortFallCorrHi}{31.6}

\newcommand{\VFourShortFallNoCorrMean}{19.6}
\newcommand{\VFourShortFallNoCorrLo}{1.2}
\newcommand{\VFourShortFallNoCorrHi}{43.2}

\newcommand{\VFourSparseCorrectionMean}{63.75}
\newcommand{\VFourSparseCorrectionLo}{49.38}
\newcommand{\VFourSparseCorrectionHi}{74.25}
\newcommand{\VFourSparseCorrectionLeft}{96.13}
\newcommand{\VFourSparseCorrectionRight}{32.38}
\newcommand{\VFourFullFallInteractionMean}{4611.5}
\newcommand{\VFourFullFallInteractionLo}{1681.1}
\newcommand{\VFourFullFallInteractionHi}{7276.2}
\newcommand{\VFourFullPushInteractionMean}{4414.8}
\newcommand{\VFourFullPushInteractionLo}{2517.6}
\newcommand{\VFourFullPushInteractionHi}{6176.2}

\newcommand{\VSevenDenseDirectAllMean}{70.60}
\newcommand{\VSevenDenseDirectAllSd}{48.07}
\newcommand{\VSevenDenseSharedAllMean}{60.22}
\newcommand{\VSevenDenseSharedAllSd}{36.09}
\newcommand{\VSevenDenseDirectWithinStateMean}{50.45}
\newcommand{\VSevenDenseDirectWithinStateSd}{35.64}
\newcommand{\VSevenDenseSharedWithinStateMean}{19.29}
\newcommand{\VSevenDenseSharedWithinStateSd}{18.79}
\newcommand{\VSevenDenseDirectActionMean}{66.30}
\newcommand{\VSevenDenseDirectActionSd}{38.61}
\newcommand{\VSevenDenseDirectBiasMean}{-14.21}
\newcommand{\VSevenDenseDirectBiasSd}{30.58}
\newcommand{\VSevenDenseDirectRegretMean}{22.87}
\newcommand{\VSevenDenseDirectRegretSd}{21.86}
\newcommand{\VSevenDenseSharedActionMean}{60.38}
\newcommand{\VSevenDenseSharedActionSd}{36.12}
\newcommand{\VSevenDenseSharedBiasMean}{-11.30}
\newcommand{\VSevenDenseSharedBiasSd}{17.37}
\newcommand{\VSevenDenseSharedRegretMean}{23.54}
\newcommand{\VSevenDenseSharedRegretSd}{23.28}
\newcommand{\VSevenDenseOutputTargetAllMean}{72.75}
\newcommand{\VSevenDenseOutputTargetAllSd}{48.04}
\newcommand{\VSevenDenseOutputTargetActionMean}{73.00}
\newcommand{\VSevenDenseOutputTargetActionSd}{48.09}
\newcommand{\VSevenDenseOutputTargetBiasMean}{30.72}
\newcommand{\VSevenDenseOutputTargetBiasSd}{12.15}
\newcommand{\VSevenDenseOutputTargetRegretMean}{19.67}
\newcommand{\VSevenDenseOutputTargetRegretSd}{17.42}
\newcommand{\VSevenDenseOutputOnlyAllMean}{73.25}
\newcommand{\VSevenDenseOutputOnlyAllSd}{47.82}
\newcommand{\VSevenDenseOutputOnlyActionMean}{73.51}
\newcommand{\VSevenDenseOutputOnlyActionSd}{47.88}
\newcommand{\VSevenDenseOutputOnlyBiasMean}{32.75}
\newcommand{\VSevenDenseOutputOnlyBiasSd}{12.38}
\newcommand{\VSevenDenseOutputOnlyRegretMean}{20.81}
\newcommand{\VSevenDenseOutputOnlyRegretSd}{19.25}
\newcommand{\VSevenDenseTargetOnlyAllMean}{72.86}
\newcommand{\VSevenDenseTargetOnlyAllSd}{48.11}
\newcommand{\VSevenDenseTargetOnlyActionMean}{73.10}
\newcommand{\VSevenDenseTargetOnlyActionSd}{48.17}
\newcommand{\VSevenDenseTargetOnlyBiasMean}{30.48}
\newcommand{\VSevenDenseTargetOnlyBiasSd}{12.30}
\newcommand{\VSevenDenseTargetOnlyRegretMean}{21.23}
\newcommand{\VSevenDenseTargetOnlyRegretSd}{19.42}
\newcommand{\VSevenDenseSignedAllMean}{73.35}
\newcommand{\VSevenDenseSignedAllSd}{48.08}
\newcommand{\VSevenDenseSignedActionMean}{73.60}
\newcommand{\VSevenDenseSignedActionSd}{48.15}
\newcommand{\VSevenDenseSignedBiasMean}{32.17}
\newcommand{\VSevenDenseSignedBiasSd}{12.25}
\newcommand{\VSevenDenseSignedRegretMean}{21.59}
\newcommand{\VSevenDenseSignedRegretSd}{20.04}
\newcommand{\VSevenDenseSeedASlotDifference}{-3.15}
\newcommand{\VSevenDenseSeedADedupDifference}{-3.13}
\newcommand{\VSevenDenseSeedA}{777}
\newcommand{\VSevenDenseSeedBSlotDifference}{-24.92}
\newcommand{\VSevenDenseSeedBDedupDifference}{-25.01}
\newcommand{\VSevenDenseSeedB}{778}
\newcommand{\VSevenDenseSeedCSlotDifference}{-3.08}
\newcommand{\VSevenDenseSeedCDedupDifference}{-3.00}
\newcommand{\VSevenDenseSeedC}{779}

\newcommand{\VSevenSparseDirectAllMean}{9.49}
\newcommand{\VSevenSparseDirectAllSd}{0.36}
\newcommand{\VSevenSparseSharedAllMean}{8.86}
\newcommand{\VSevenSparseSharedAllSd}{0.91}
\newcommand{\VSevenSparseDirectWithinStateMean}{8.50}
\newcommand{\VSevenSparseDirectWithinStateSd}{0.35}
\newcommand{\VSevenSparseSharedWithinStateMean}{1.13}
\newcommand{\VSevenSparseSharedWithinStateSd}{0.03}
\newcommand{\VSevenSparseDirectActionMean}{10.41}
\newcommand{\VSevenSparseDirectActionSd}{0.19}
\newcommand{\VSevenSparseDirectBiasMean}{1.77}
\newcommand{\VSevenSparseDirectBiasSd}{1.39}
\newcommand{\VSevenSparseDirectRegretMean}{1.29}
\newcommand{\VSevenSparseDirectRegretSd}{0.20}
\newcommand{\VSevenSparseSharedActionMean}{8.85}
\newcommand{\VSevenSparseSharedActionSd}{0.91}
\newcommand{\VSevenSparseSharedBiasMean}{-0.86}
\newcommand{\VSevenSparseSharedBiasSd}{1.65}
\newcommand{\VSevenSparseSharedRegretMean}{1.33}
\newcommand{\VSevenSparseSharedRegretSd}{0.34}
\newcommand{\VSevenSparseOutputTargetAllMean}{12.60}
\newcommand{\VSevenSparseOutputTargetAllSd}{0.35}
\newcommand{\VSevenSparseOutputTargetActionMean}{12.60}
\newcommand{\VSevenSparseOutputTargetActionSd}{0.36}
\newcommand{\VSevenSparseOutputTargetBiasMean}{8.18}
\newcommand{\VSevenSparseOutputTargetBiasSd}{0.38}
\newcommand{\VSevenSparseOutputTargetRegretMean}{1.29}
\newcommand{\VSevenSparseOutputTargetRegretSd}{0.26}
\newcommand{\VSevenSparseOutputOnlyAllMean}{12.60}
\newcommand{\VSevenSparseOutputOnlyAllSd}{0.34}
\newcommand{\VSevenSparseOutputOnlyActionMean}{12.60}
\newcommand{\VSevenSparseOutputOnlyActionSd}{0.36}
\newcommand{\VSevenSparseOutputOnlyBiasMean}{8.18}
\newcommand{\VSevenSparseOutputOnlyBiasSd}{0.39}
\newcommand{\VSevenSparseOutputOnlyRegretMean}{1.27}
\newcommand{\VSevenSparseOutputOnlyRegretSd}{0.28}
\newcommand{\VSevenSparseTargetOnlyAllMean}{12.28}
\newcommand{\VSevenSparseTargetOnlyAllSd}{0.36}
\newcommand{\VSevenSparseTargetOnlyActionMean}{12.27}
\newcommand{\VSevenSparseTargetOnlyActionSd}{0.38}
\newcommand{\VSevenSparseTargetOnlyBiasMean}{7.32}
\newcommand{\VSevenSparseTargetOnlyBiasSd}{0.45}
\newcommand{\VSevenSparseTargetOnlyRegretMean}{1.30}
\newcommand{\VSevenSparseTargetOnlyRegretSd}{0.24}
\newcommand{\VSevenSparseSignedAllMean}{12.52}
\newcommand{\VSevenSparseSignedAllSd}{0.36}
\newcommand{\VSevenSparseSignedActionMean}{12.52}
\newcommand{\VSevenSparseSignedActionSd}{0.37}
\newcommand{\VSevenSparseSignedBiasMean}{7.91}
\newcommand{\VSevenSparseSignedBiasSd}{0.43}
\newcommand{\VSevenSparseSignedRegretMean}{1.27}
\newcommand{\VSevenSparseSignedRegretSd}{0.24}
\newcommand{\VSevenSparseSeedASlotDifference}{-0.77}
\newcommand{\VSevenSparseSeedADedupDifference}{-0.77}
\newcommand{\VSevenSparseSeedA}{777}
\newcommand{\VSevenSparseSeedBSlotDifference}{-0.03}
\newcommand{\VSevenSparseSeedBDedupDifference}{-0.03}
\newcommand{\VSevenSparseSeedB}{778}
\newcommand{\VSevenSparseSeedCSlotDifference}{-1.11}
\newcommand{\VSevenSparseSeedCDedupDifference}{-1.12}
\newcommand{\VSevenSparseSeedC}{779}

\newcommand{\VSevenMdpTwoZeroCandidates}{2}
\newcommand{\VSevenMdpTwoZeroNoise}{0.0}
\newcommand{\VSevenMdpTwoZeroSignedSelectedBiasMean}{0.002}
\newcommand{\VSevenMdpTwoZeroSignedSelectedBiasSd}{0.017}
\newcommand{\VSevenMdpTwoZeroProjectedSelectedBiasMean}{-0.268}
\newcommand{\VSevenMdpTwoZeroProjectedSelectedBiasSd}{0.019}
\newcommand{\VSevenMdpTwoZeroSignedMaeMean}{0.019}
\newcommand{\VSevenMdpTwoZeroSignedMaeSd}{0.006}
\newcommand{\VSevenMdpTwoZeroProjectedMaeMean}{0.270}
\newcommand{\VSevenMdpTwoZeroProjectedMaeSd}{0.016}
\newcommand{\VSevenMdpTwoLowCandidates}{2}
\newcommand{\VSevenMdpTwoLowNoise}{0.3}
\newcommand{\VSevenMdpTwoLowSignedSelectedBiasMean}{0.006}
\newcommand{\VSevenMdpTwoLowSignedSelectedBiasSd}{0.028}
\newcommand{\VSevenMdpTwoLowProjectedSelectedBiasMean}{-0.466}
\newcommand{\VSevenMdpTwoLowProjectedSelectedBiasSd}{0.032}
\newcommand{\VSevenMdpTwoLowSignedMaeMean}{0.032}
\newcommand{\VSevenMdpTwoLowSignedMaeSd}{0.010}
\newcommand{\VSevenMdpTwoLowProjectedMaeMean}{0.464}
\newcommand{\VSevenMdpTwoLowProjectedMaeSd}{0.029}
\newcommand{\VSevenMdpTwoHighCandidates}{2}
\newcommand{\VSevenMdpTwoHighNoise}{1.0}
\newcommand{\VSevenMdpTwoHighSignedSelectedBiasMean}{0.059}
\newcommand{\VSevenMdpTwoHighSignedSelectedBiasSd}{0.071}
\newcommand{\VSevenMdpTwoHighProjectedSelectedBiasMean}{-1.241}
\newcommand{\VSevenMdpTwoHighProjectedSelectedBiasSd}{0.083}
\newcommand{\VSevenMdpTwoHighSignedMaeMean}{0.096}
\newcommand{\VSevenMdpTwoHighSignedMaeSd}{0.034}
\newcommand{\VSevenMdpTwoHighProjectedMaeMean}{1.245}
\newcommand{\VSevenMdpTwoHighProjectedMaeSd}{0.086}
\newcommand{\VSevenMdpFourZeroCandidates}{4}
\newcommand{\VSevenMdpFourZeroNoise}{0.0}
\newcommand{\VSevenMdpFourZeroSignedSelectedBiasMean}{0.004}
\newcommand{\VSevenMdpFourZeroSignedSelectedBiasSd}{0.015}
\newcommand{\VSevenMdpFourZeroProjectedSelectedBiasMean}{-0.247}
\newcommand{\VSevenMdpFourZeroProjectedSelectedBiasSd}{0.017}
\newcommand{\VSevenMdpFourZeroSignedMaeMean}{0.017}
\newcommand{\VSevenMdpFourZeroSignedMaeSd}{0.005}
\newcommand{\VSevenMdpFourZeroProjectedMaeMean}{0.246}
\newcommand{\VSevenMdpFourZeroProjectedMaeSd}{0.014}
\newcommand{\VSevenMdpFourLowCandidates}{4}
\newcommand{\VSevenMdpFourLowNoise}{0.3}
\newcommand{\VSevenMdpFourLowSignedSelectedBiasMean}{0.010}
\newcommand{\VSevenMdpFourLowSignedSelectedBiasSd}{0.022}
\newcommand{\VSevenMdpFourLowProjectedSelectedBiasMean}{-0.448}
\newcommand{\VSevenMdpFourLowProjectedSelectedBiasSd}{0.027}
\newcommand{\VSevenMdpFourLowSignedMaeMean}{0.031}
\newcommand{\VSevenMdpFourLowSignedMaeSd}{0.007}
\newcommand{\VSevenMdpFourLowProjectedMaeMean}{0.446}
\newcommand{\VSevenMdpFourLowProjectedMaeSd}{0.027}
\newcommand{\VSevenMdpFourHighCandidates}{4}
\newcommand{\VSevenMdpFourHighNoise}{1.0}
\newcommand{\VSevenMdpFourHighSignedSelectedBiasMean}{0.053}
\newcommand{\VSevenMdpFourHighSignedSelectedBiasSd}{0.074}
\newcommand{\VSevenMdpFourHighProjectedSelectedBiasMean}{-1.240}
\newcommand{\VSevenMdpFourHighProjectedSelectedBiasSd}{0.091}
\newcommand{\VSevenMdpFourHighSignedMaeMean}{0.096}
\newcommand{\VSevenMdpFourHighSignedMaeSd}{0.021}
\newcommand{\VSevenMdpFourHighProjectedMaeMean}{1.239}
\newcommand{\VSevenMdpFourHighProjectedMaeSd}{0.083}
\newcommand{\VSevenMdpEightZeroCandidates}{8}
\newcommand{\VSevenMdpEightZeroNoise}{0.0}
\newcommand{\VSevenMdpEightZeroSignedSelectedBiasMean}{0.003}
\newcommand{\VSevenMdpEightZeroSignedSelectedBiasSd}{0.014}
\newcommand{\VSevenMdpEightZeroProjectedSelectedBiasMean}{-0.247}
\newcommand{\VSevenMdpEightZeroProjectedSelectedBiasSd}{0.015}
\newcommand{\VSevenMdpEightZeroSignedMaeMean}{0.017}
\newcommand{\VSevenMdpEightZeroSignedMaeSd}{0.004}
\newcommand{\VSevenMdpEightZeroProjectedMaeMean}{0.248}
\newcommand{\VSevenMdpEightZeroProjectedMaeSd}{0.014}
\newcommand{\VSevenMdpEightLowCandidates}{8}
\newcommand{\VSevenMdpEightLowNoise}{0.3}
\newcommand{\VSevenMdpEightLowSignedSelectedBiasMean}{0.022}
\newcommand{\VSevenMdpEightLowSignedSelectedBiasSd}{0.027}
\newcommand{\VSevenMdpEightLowProjectedSelectedBiasMean}{-0.438}
\newcommand{\VSevenMdpEightLowProjectedSelectedBiasSd}{0.029}
\newcommand{\VSevenMdpEightLowSignedMaeMean}{0.033}
\newcommand{\VSevenMdpEightLowSignedMaeSd}{0.008}
\newcommand{\VSevenMdpEightLowProjectedMaeMean}{0.437}
\newcommand{\VSevenMdpEightLowProjectedMaeSd}{0.024}
\newcommand{\VSevenMdpEightHighCandidates}{8}
\newcommand{\VSevenMdpEightHighNoise}{1.0}
\newcommand{\VSevenMdpEightHighSignedSelectedBiasMean}{0.153}
\newcommand{\VSevenMdpEightHighSignedSelectedBiasSd}{0.069}
\newcommand{\VSevenMdpEightHighProjectedSelectedBiasMean}{-1.145}
\newcommand{\VSevenMdpEightHighProjectedSelectedBiasSd}{0.073}
\newcommand{\VSevenMdpEightHighSignedMaeMean}{0.135}
\newcommand{\VSevenMdpEightHighSignedMaeSd}{0.036}
\newcommand{\VSevenMdpEightHighProjectedMaeMean}{1.160}
\newcommand{\VSevenMdpEightHighProjectedMaeSd}{0.061}
\newcommand{\VSevenPolicySeeds}{3}
\newcommand{\VSevenUnits}{6}

\newcommand{\VSevenUpdates}{5000}

\newcommand{\VSevenTrainEpisodes}{16}
\newcommand{\VSevenCalEpisodes}{8}
\newcommand{\VSevenTestEpisodes}{8}
\newcommand{\VSevenEpisodes}{32}
\newcommand{\VSevenSnapshots}{960}
\newcommand{\VSevenGoalSlots}{8}
\newcommand{\VSevenActions}{3}
\newcommand{\VSevenBankRows}{23040}

\newcommand{\VSevenUniqueInputs}{22944}

\newcommand{\VSevenDuplicateGroups}{96}

\newcommand{\VSevenDirectParams}{3055340}
\newcommand{\VSevenSharedParams}{3057030}
\newcommand{\VSevenMAEImprovedUnits}{6}

\newcommand{\VSevenMDPNoiseSeeds}{30}

\newcommand{\VSevenMDPRuns}{1092}

\title{Hierarchical Time-aware Bootstrapping\\for Off-Policy Subgoal Value Learning}
\author{Bingyun Liu\thanks{Equal contribution.}\\
Institute of Automation\\
Chinese Academy of Sciences\\
\texttt{liubingyun2021@ia.ac.cn}
\And
Yuheng Jing\footnotemark[1]\kern0.4em\thanks{Corresponding author.}\\
Institute of Automation\\
Chinese Academy of Sciences\\
\texttt{jingyuheng2022@ia.ac.cn}}
\hypersetup{colorlinks=true,linkcolor=blue,citecolor=blue,urlcolor=blue,
pdftitle={Hierarchical Time-aware Bootstrapping for Off-Policy Subgoal Value Learning},
pdfauthor={Bingyun Liu; Yuheng Jing}}
\begin{document}
\maketitle
\begin{abstract}
\IfFileExists{sections/abstract.tex}{Off-policy hierarchical reinforcement learning must estimate the values of
high-level decisions while the low-level policy changes. HIRO adapts
replay data through subgoal relabeling, but after a label change, the value
update targets the relabeled subgoal instead of the subgoal the high-level
policy originally needed to update. We propose Hierarchical
Time-aware Bootstrapping (HTB), which evaluates specified subgoals under
the current low-level policy while retaining accumulated task rewards.
Remaining execution time distinguishes subgoal continuation from a new
high-level decision. Together with primitive-action conditioning, it
enables off-policy Bellman updates based on the stationary environment
transition law. HTB combines these one-step updates with multi-step
suffix returns and truncated relabeling, reducing dependence on
intermediate value estimates. A shared value component supports learning
across actions, while nonnegative residuals constrain upward corrections
relative to that component. At a fixed mixture weight of $.95$, HTB
achieves \VFourFallCorrSuccess{}\% AntFall success versus
\VFourFallHiroSuccess{}\% for matched local HIRO over five paired seeds
at \PFTargetMillions{}M environment steps. Ablations identify contributions
from recursive continuation and mixed supervision; fixed-policy tests
show more accurate predictions for actions whose returns were excluded
from fitting.
}{}
\end{abstract}
\section{Introduction}
Goal-conditioned hierarchical reinforcement learning (HRL) selects
subgoals with a high-level policy and executes them with a low-level
policy. As the latter changes, the same subgoal induces different
transitions and returns, making high-level replay nonstationary.

Hierarchical Actor-Critic (HAC) corrects replay with hindsight actions
under idealized lower-level execution. Its correction relies on a
restrictive assumption: the high-level reward must be independent of the
low-level execution path, as explicitly required by
\citet[Sec.~4.3]{levy2019hac}. The standard goal-reaching cost is $-1$
before success and $0$ at the goal. Endpoint hindsight preserves this
path-independent objective, but fails to preserve general accumulated
trajectory rewards. For example, two policies may reach the same spatial
subgoal with accumulated rewards of $1$ and $10$; substituting the
endpoint does not make their returns interchangeable. This mechanism is
therefore not a general correction for hierarchical nonstationarity.

HIRO addresses general trajectory rewards through likelihood-based
relabeling: it chooses a subgoal that explains the recorded primitive
actions under the current low-level policy \citep{nachum2018hiro}.
This also changes which subgoal receives value supervision.
When the label changes, that regression supplies no direct target for
the original high-level decision. Yet the high-level policy needs to
evaluate its recorded decisions and future proposals, including subgoals
that do not explain the historical trajectory.
We ask: \emph{how can replay update the value of a specified subgoal
under the current low-level policy?}

We introduce \emph{Hierarchical Time-aware Bootstrapping} (HTB).
Conditioning on the recorded primitive action exposes the stationary
environment transition law. Remaining execution time specifies whether
the current subgoal continues or a new high-level decision follows.
Together, they provide primitive-step Bellman updates for a chosen
subgoal while preserving the hierarchy's decision period.
HTB mixes these updates with multi-step suffix returns and truncated
relabeling, replacing intermediate predictions with observed rewards.
A shared subgoal value transfers updates across primitive actions;
action-value residuals supply action-specific corrections, with
nonnegativity constraining upward amplification.

Our contributions are the time-aware Bellman formulation, its mixed-backup
and shared-value implementation, and evidence from task success,
component ablations and action-value prediction under limited supervision.

\begin{figure}[t]
  \centering
  \includegraphics[width=\textwidth]{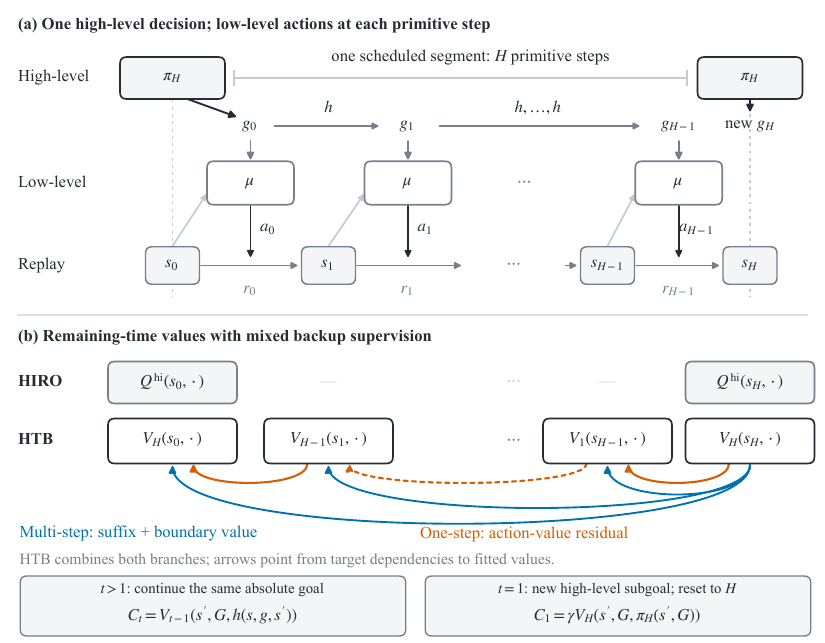}
  \caption{\textbf{Subgoal values can be learned at every primitive step.}
  The high-level policy acts every $H$ steps. HTB learns
  $V_{H-i}(s_i,G,g_i)$ throughout each segment using a mixture of its
  remaining return (blue) and a one-step backup (orange).
  The current subgoal continues until the next high-level decision.
  Arrows run from target dependencies to the value being updated;
  displayed continuations assume no termination.}
  \label{fig:architecture}
\end{figure}

\section{Subgoal evaluation under a changing low-level policy}
\label{sec:problem}
Let $s$ denote the environment state, $G$ the task goal, $g$ a subgoal,
and $a$ a primitive action. A high-level policy $\pi_H(s,G)$ selects a
subgoal every $H$ steps; the low-level policy $\mu(a\mid s,g)$ acts at
each step. For achieved-goal coordinates $\phi(s)$, the goal-transition
function $h(s,g,s')=g+\phi(s)-\phi(s')$ keeps the absolute destination
unchanged. Replay stores a segment
$\tau=(s_0,a_0,r_0,\ldots,s_L)$, its original subgoal $g^{\rm orig}$,
and its observed length $L\leq H$. A segment ends at the scheduled
boundary or episode termination, so $L<H$ indicates early termination.
Here $r$ is the scaled task reward,
not the low-level intrinsic reward.

HIRO chooses a relabeled subgoal $\widetilde g$ by approximately
maximizing the likelihood of the recorded action sequence under the
current low-level policy. Denoting its high-level critic by $Q^{\rm hi}$,
the resulting regression has the form
\begin{equation}
 \mathcal L_{\rm corr}
 =\mathbb E_{\tau\sim\mathcal D}
 \left[\ell\!\left(Q^{\rm hi}(s_0,G,\widetilde g),Y_\tau\right)\right],
 \label{eq:relabel-subtask-loss}
\end{equation}
where $\mathcal D$ is the replay distribution, $\ell$ is a regression
loss, and $Y_\tau$ is a stopped-gradient return target.
The supervised input is therefore the relabeled subgoal. When
$\widetilde g\neq g^{\rm orig}$, the original high-level decision receives
no direct target from this regression, even though parameter sharing may
change its prediction indirectly.

Our objective is to evaluate an explicitly chosen subgoal $g$ under the
current low-level policy, including the original replay subgoal or a
proposal from the current high-level policy. Its value covers the remaining
execution duration and subsequent high-level decisions. Because different
subgoals induce different future actions, we learn primitive-action values
from recorded transitions and evaluate them at the current policy's action.
Remaining time identifies whether the continuation executes the chosen
subgoal or returns control to the high-level policy. This construction
separates the subgoal receiving supervision from the subgoal selected by
likelihood-based correction.

\section{Hierarchical Time-aware Bootstrapping}
\label{sec:method}

\subsection{Remaining-time Bellman updates}
Let $t\in\{1,\ldots,H\}$ be the number of primitive steps until the next
scheduled high-level decision. For fixed policies, $V_t(s,G,g)$ is the
expected task return from executing subgoal $g$ for these $t$ steps and
following the hierarchy afterward. Define $Q_t(s,G,g,a)$ by fixing the
first primitive action to $a$. Thus $V_t$ is a subgoal value, while $Q_t$
also conditions on the primitive action. Both estimate the external task
return, distinct from the intrinsic-reward objective used to train the low-level policy.
At a decision boundary,
$V_H$ evaluates the high-level action.

Within a high-level segment, rewards are summed without discounting.
A factor $\gamma$ is applied at the next high-level decision, matching
the training objective. With terminal indicator $d$, the Bellman equations are
\begin{align}
 Q_t(s,G,g,a)
 &=\mathbb E\!\left[r+(1-d)C_t(s',G,h(s,g,s'))\mid s,a,G\right],
 \label{eq:primitive-bellman}\\
 C_t(s',G,g')
 &=\begin{cases}
 V_{t-1}(s',G,g'), & t>1,\\
 \gamma V_H(s',G,\pi_H(s',G)), & t=1,
 \end{cases}
 \qquad
 V_t=\mathbb E_{a\sim\mu(\cdot\mid s,g)}Q_t(s,G,g,a).
 \label{eq:countdown-bellman}
\end{align}
The two cases distinguish continuing the current subgoal from choosing
the next one (Figure~\ref{fig:architecture}). For example, with $H=10$,
a transition at step $i=4$ updates $V_6$, whose one-step continuation is
$V_5$. A value function that cannot distinguish these remaining durations
can assign incompatible targets to the same input \citep{pardo2018time}
(Appendix~\ref{app:time-aliasing}).

\begin{proposition}[Off-policy primitive-step evaluation]
\label{prop:conditional-replay}
Assume a stationary Markov environment with joint transition law
$P(s',r,d\mid s,a,G)$ independent of the commanded subgoal.
Suppose replay sampling preserves this conditional law, and the
subgoal to be evaluated is fixed or selected independently of the
transition outcome given $(s,a,G)$.
For any fixed bounded continuation, a replay transition has the
conditional Bellman expectation in Equation~\eqref{eq:primitive-bellman},
independently of the data-collection policy.
\end{proposition}

The same recorded transition can therefore support evaluation of a
specified subgoal different from its collection subgoal. Current-policy
evaluation uses $Q_t$ on the current policy's action support. The conditional
transition kernel remains stationary while the continuation tracks the
improving hierarchy. Appendix~\ref{app:theory} develops the signed residual
recursion and the sampling and coverage conditions behind this construction.

\subsection{Multi-step returns with truncated goal relabeling}
The implementation uses a deterministic low-level actor, denoted $\mu(s,g)$.
A one-step recursion repeatedly bootstraps through intermediate value
estimates. A multi-step target replaces these intermediate estimates
with observed rewards and bootstraps once at the next high-level boundary.
For a suffix starting at step $i$, let $t=H-i$ and define
\begin{equation}
 y^{\mathrm L}_{t,i}=\sum_{k=i}^{L-1}r_k
 +(1-d_{L-1})\gamma\bar V_H(s_L,G,\widehat g_H(s_L)).
 \label{eq:long-target}
\end{equation}
The implementation uses twin value networks $V_{t,j}$, $j\in\{1,2\}$.
Bars denote target networks, $\bar V_t=\min_j\bar V_{t,j}$, and
$\widehat g_H$ is the target high-level action with target-policy smoothing
\citep{fujimoto2018td3}.

Because the suffix was collected by an earlier low-level policy,
we adapt its subgoal using HIRO-style likelihood correction.
The likelihood must use the same suffix as the return:
\begin{equation}
 \widetilde g_i=\arg\min_{g\in\mathcal C_i}
 \sum_{k=i}^{L-1}
 \left\|a_k-\mu\!\left(s_k,g+\phi(s_i)-\phi(s_k)\right)\right\|_2^2,
 \label{eq:suffix-correction-main}
\end{equation}
where $\mathcal C_i$ contains the transported replay subgoal, the
achieved displacement, and Gaussian candidates around that displacement.
This is \emph{truncated goal relabeling}: actions before $i$ are excluded.
Figures abbreviate this operation as suffix correction.
Each suffix thus selects a goal for its own action window; the full
segment defines a different likelihood objective.
The multi-step loss fits $V_t(s_i,G,\widetilde g_i)$ to
$y^{\mathrm L}_{t,i}$, while the one-step update retains its explicitly
specified subgoal.

The observed reward sum removes the learned intermediate continuations
from this target. Its remaining estimation errors include the boundary
value and approximate likelihood correction; Appendix~\ref{app:multistep-error}
separates these sources.
Early termination removes the boundary bootstrap. Remaining time is
scheduled time $H-i$, even when $L<H$; only observed steps contribute.

\subsection{Shared subgoal values and action-value residuals}
Primitive-action conditioning introduces an additional input dimension.
When the low-level policy changes, the new action may have little direct
training data. Inspired by dueling value decomposition
\citep{wang2016dueling}, we share a subgoal value across primitive actions:
\begin{align}
 \widehat Q_{t,j}(s,G,g,a)&=V_{t,j}(s,G,g)-A_t(s,G,g,a),
 \label{eq:query-reconstruction-main}\\
 \Delta\widehat Q_{t,j}(s,G,g,a)&=\Delta V_{t,j}(s,G,g)
 \quad\text{for every }a\text{ when }A_t\text{ is fixed}.
 \label{eq:shared-value-update}
\end{align}
The network $A_t$ estimates an \emph{action-value residual}. For fixed
policies, its exact signed form is
$A_t^{\mu,\pi}=V_t^{\mu,\pi}-Q_t^{\mu,\pi}$, the negative of the
conventional primitive-action advantage. The learned $A_t$ is a
nonnegative approximation. The shared $V_t(s,G,g)$ acts as an
action-independent value offset (bias), while $A_t$ supplies action-specific
corrections. Updating this offset contributes the same shift to every
action value, including actions not supervised by that update
(Equation~\eqref{eq:shared-value-update}; Figure~\ref{fig:backups}c).

For a specified subgoal $g_i$, put
$g_i^+=h(s_i,g_i,s_{i+1})$. HTB corrects the next-step continuation
using the next residual:
\begin{align}
 \bar C_{t,i}
 &=\begin{cases}
 \bar V_{t-1}(s_{i+1},G,g_i^+)
 -\bar A_{t-1}(s_{i+1},G,g_i^+,\mu(s_{i+1},g_i^+)), & t>1,\\
 \gamma\bar V_H(s_{i+1},G,\widehat g_H(s_{i+1})), & t=1,
 \end{cases}
 \label{eq:recursive-continuation}\\
 y^A_{t,i}
 &=\left[\bar V_t(s_i,G,g_i)-r_i-(1-d_i)\bar C_{t,i}\right]_+,
 \label{eq:gap-target}
\end{align}
where $[x]_+=\max(x,0)$. At fixed $(s_i,G,g_i,t)$, the one-step
value target averages over the low-level policy's primitive action:
\begin{align}
 y^{\mathrm S}_{t,i}
 &=\mathbb E_{a\sim\mu(\cdot\mid s_i,g_i)}
 \!\left[\bar V_t(s_i,G,g_i)-\bar A_t(s_i,G,g_i,a)\right]\nonumber\\
 &=\bar V_t(s_i,G,g_i)-\bar A_t(s_i,G,g_i,\mu(s_i,g_i))
 \quad\text{for deterministic }\mu.
 \label{eq:short-target-expectation}
\end{align}
With target-policy smoothing, the implementation uses the sampled target
\begin{equation}
 \widehat y^{\mathrm S}_{t,i}
 =\bar V_t(s_i,G,g_i)-\bar A_t(s_i,G,g_i,\widehat a_i).
 \label{eq:short-target}
\end{equation}
Here $\widehat a_i$ adds clipped smoothing noise to the current low-level
action and clips to the action bounds. It targets the smoothed policy;
Equation~\eqref{eq:short-target-expectation} gives the noise-free case.
The residual network is fitted to $y^A_{t,i}$ at the \emph{recorded}
action $a_i$. It is then evaluated at the \emph{current-policy}
action to construct $\widehat y^{\mathrm S}_{t,i}$, the target for
$V_{t,j}(s_i,G,g_i)$.
This distinction is what allows recorded transitions to evaluate a
specified subgoal under the current policy.
All targets are stopped-gradient; the residual loss updates $A$ and
the value losses update $V$. The next-state action in $\bar C_{t,i}$
is unsmoothed.

\paragraph{Nonnegative residuals.}
The practical algorithm constrains both the residual output and target
to be nonnegative. For a scalar reference $q$ and Bellman target $b$,
exactly fitting the projected residual gives
$q-[q-b]_+=\min(q,b)$.
The reconstructed value therefore cannot exceed its current reference.
This one-sided correction filters upward target amplification relative
to the learned reference, including amplification introduced by
maximization. The cap is reference-relative: its effect on overestimation
and underestimation depends on the accuracy of that reference.
Appendix~\ref{app:theory} characterizes the projection and its signed counterpart.

\begin{figure}[t]
  \centering
  \includegraphics[width=\textwidth]{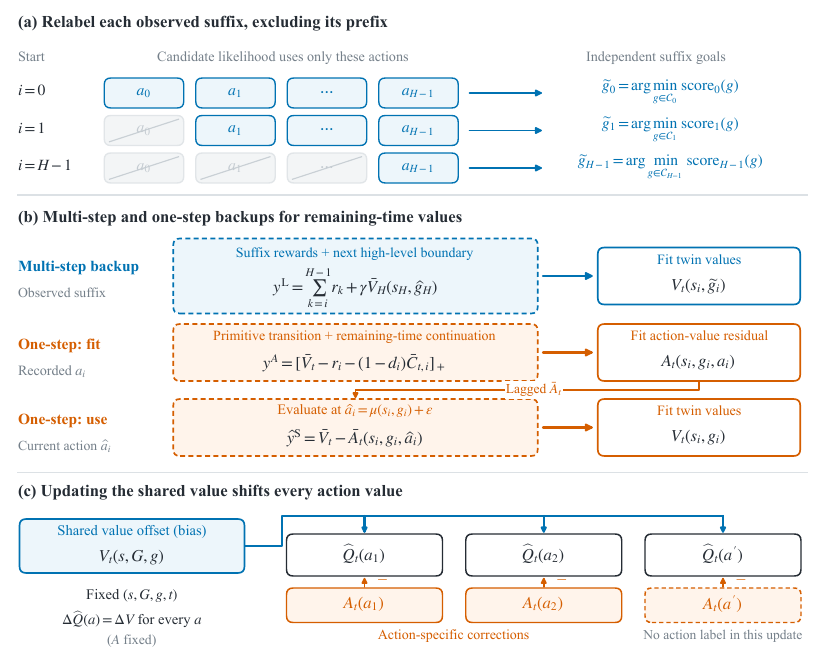}
  \caption{\textbf{Multi-step and one-step updates can supervise different subgoals.}
  (a) Each suffix is relabeled using only its recorded actions.
  (b) The multi-step target trains the relabeled subgoal; the one-step
  estimator fits the recorded action and evaluates the current-policy
  action for the specified subgoal. (c) Updating $V$ changes the shared
  value offset in every primitive-action value, including actions without
  direct supervision; for fixed $A$, $\Delta\widehat Q(a)=\Delta V$.
  Repeated task-goal
  inputs and action clipping are omitted; the diagram shows a nonterminal segment.}
  \label{fig:backups}
\end{figure}

\subsection{Combined objective and ablations}
Using smooth-$L_1$ loss $\ell$ and $\lambda\in[0,1]$, the value objective is
\begin{align}
 \mathcal L_V
 &=\lambda\,\mathbb E_{\mathcal D}\sum_{j=1}^2
 \ell\!\left(V_{t,j}(s_i,G,\widetilde g_i),y^{\mathrm L}_{t,i}\right)
 \nonumber\\
 &\quad +(1-\lambda)\,\mathbb E_{\mathcal D}\sum_{j=1}^2
 \ell\!\left(V_{t,j}(s_i,G,g_i),\widehat y^{\mathrm S}_{t,i}\right).
 \label{eq:mixed-loss}
\end{align}
We average valid samples within each remaining-time branch and sum
over branches. The weight $\lambda$ controls multi-step learning;
$\lambda=1$ removes one-step value supervision.
For $\lambda<1$, the one-step term directly updates the specified
subgoal whenever its prediction differs from the target, even if
relabeling selects another subgoal for the multi-step loss
(Appendix~\ref{app:specified-updates}).

Table~\ref{tab:method-variants} defines the ablations.
The default uses the transported replay subgoal in the one-step loss.
The current-policy-subgoal variant uses a stopped-gradient proposal
$\pi_t(s_i,G)$. Disabling relabeling changes only the subgoal used
in the multi-step loss. Only $\pi_H$ executes during interaction;
shorter-duration actors provide proposals for the corresponding replay
updates. Algorithm~\ref{alg:replay-update} and implementation details
appear in Appendix~\ref{app:implementation}.

\begin{table}[t]
\centering
\caption{\textbf{Controlled variants of the HTB update.}
All variants use remaining-time values and the same $V-A$ architecture.
The last two variants use the nonrecursive residual estimator.}
\label{tab:method-variants}
\small
\begin{tabular}{llll}
\toprule
Variant & One-step subgoal & Next residual & Output / target\\
\midrule
HTB & Replay & Included & Nonnegative / clipped\\
No next residual & Replay & Omitted & Nonnegative / clipped\\
Signed residual & Replay & Omitted & Signed / signed\\
Policy subgoal & Current proposal & Omitted & Nonnegative / clipped\\
\bottomrule
\end{tabular}
\end{table}

\IfFileExists{sections/results.tex}{\section{Experiments}
\label{sec:experiments}
The experiments test whether specified-subgoal value updates improve
hierarchical control, which backup components produce the gains, and how
the shared value representation supports action-conditioned estimation.
Push and Fall combine ant locomotion with object manipulation.
Reacher, Maze, and MazeW provide goal-reaching tasks with dense and
sparse rewards. Push/Fall comparisons use five paired seeds and
\PFTargetMillions{} million environment steps; Reacher/Maze comparisons
use eight paired seeds and one million steps.
Success is measured at the final episode step for Push/Fall, and averaged
over the last twenty evaluations for Reacher/Maze.
The main comparisons use a fixed multi-step weight $\lambda=.95$ and
nominal $95\%$ paired-bootstrap intervals. Highlighted component settings
were selected retrospectively; the appendix includes the full mixture
and relabeling sweeps, with both effect directions.
Appendices~\ref{app:training-protocol},
\ref{app:empirical}, and~\ref{app:vthree-evidence} give protocols,
complete results, and statistical procedures.

\subsection{Task success at a fixed backup configuration}
\label{sec:performance}
HTB improves Fall success at the same interaction budget as the local
HIRO implementation (Figure~\ref{fig:vthree-learning}).
With truncated relabeling, it reaches
\VFourFallCorrSuccess{}\% success, compared with
\VFourFallHiroSuccess{}\% for HIRO.
The paired improvement is \VFourFallCorrDelta{} percentage points
[\VFourFallCorrLo{}, \VFourFallCorrHi{}].
On Push, the respective means are \VFourPushCorrSuccess{}\% and
\VFourPushHiroSuccess{}\%; the interval
[\VFourPushCorrLo{}, \VFourPushCorrHi{}] includes zero.
The return curves and integrated-return comparisons are reported
separately in Figure~\ref{fig:vfour-return} and Table~\ref{tab:vthree-auc}.

\begin{figure}[t]
\centering
\includegraphics[width=\textwidth]{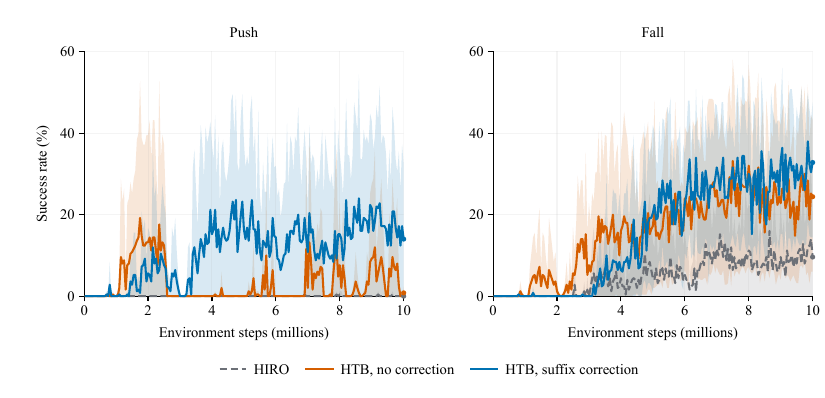}
\caption{\textbf{HTB improves success on Fall at a fixed configuration.}
Success over training, averaged across five paired seeds.
Both HTB conditions use $\lambda=.95$; HIRO uses its original
off-policy correction. Shading denotes pointwise $95\%$ empirical
bootstrap intervals. No temporal smoothing is applied.}
\label{fig:vthree-learning}
\end{figure}

The fixed-setting comparison also extends to Reacher and Maze
(Table~\ref{tab:vthree-fixed}).
On Reacher-v1 with sparse rewards, with relabeling disabled in all
compared algorithms, HTB achieves \VThreeHiroSparseLeft{}\% success
versus \VThreeHiroSparseRight{}\% for the archived HIRO baseline,
a difference of \VThreeHiroSparseMean{} points
[\VThreeHiroSparseLo{}, \VThreeHiroSparseHi{}].
The multi-step-only baseline reaches \VThreeLongSparseLeft{}\% on this
task, so high success is also attainable with multi-step-only supervision.
These fixed-setting results evaluate complete algorithms; the matched
controls below assess residual recursion and mixed supervision.

\subsection{Residual recursion and mixed backups improve success}
\label{sec:vthree-components}
To measure the contribution of recursive continuation, the first control
removes only the next-step residual from
Equation~\eqref{eq:recursive-continuation}.
It retains the same subgoal, value decomposition, nonnegative constraint,
relabeling condition, and mixture weight.
Restoring residual recursion improves success by
\VFourNextFallMean{} points [\VFourNextFallLo{}, \VFourNextFallHi{}]
on Fall and \VFourNextReacherMean{} points
[\VFourNextReacherLo{}, \VFourNextReacherHi{}] on Reacher-v0 dense
(Figure~\ref{fig:vfour-components}a).

The second control compares the mixed objective at $\lambda=.95$
with multi-step-only learning at $\lambda=1$, keeping the remaining-time
architecture and relabeling condition fixed.
The mixture improves Fall success by
\VFourShortFallCorrMean{} points
[\VFourShortFallCorrLo{}, \VFourShortFallCorrHi{}] with relabeling
and \VFourShortFallNoCorrMean{} points
[\VFourShortFallNoCorrLo{}, \VFourShortFallNoCorrHi{}] without it
(Figure~\ref{fig:vfour-components}b).
The intervention changes both objective weights and therefore measures
their combined effect. Complete controls, including tasks without
improvement, appear in
Tables~\ref{tab:vthree-short} and~\ref{tab:vthree-components-all}.

\begin{figure}[t]
\centering
\includegraphics[width=\textwidth]{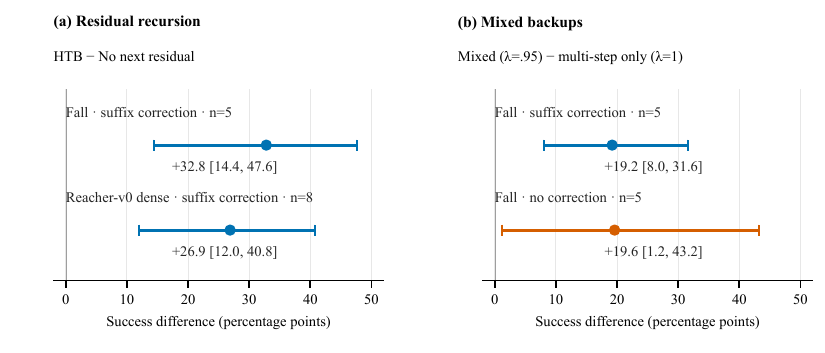}
\caption{\textbf{Residual recursion and mixed backups contribute to task success.}
(a) Full HTB minus the variant without the next-step residual,
both at $\lambda=.95$. (b) Mixed backups at $.95$ minus multi-step-only
learning at $1$. Points are paired mean success differences; bars are
nominal $95\%$ bootstrap intervals.}
\label{fig:vfour-components}
\end{figure}

The nonnegative residual provides a separate, one-sided constraint.
In the nonrecursive ablation, imposing nonnegative outputs and
targets increases Reacher-v1 dense success with relabeling at
$\lambda=.6$ by \VThreeProjectionPositiveMean{} points
[\VThreeProjectionPositiveLo{}, \VThreeProjectionPositiveHi{}].
The complete task grid shows both positive and negative effects
(Tables~\ref{tab:vthree-projection-reacher}--\ref{tab:vthree-projection-pushfall}).
The factorial and finite-MDP diagnostics in
Appendix~\ref{app:vseven-diagnostics} clarify this constraint's role:
lower estimates can reduce optimism while increasing underestimation
and total estimation error.

\subsection{Why separate subgoal supervision from behavior correction?}
\label{sec:routing-measurement}
\label{sec:gradient-measurement}
Figure~\ref{fig:gradients} tests two properties of the combined update:
which subgoal receives a target, and whether the two losses make
compatible parameter updates. We analyze \DiagSamples{} replay
transitions per trained policy instance, across \DiagN{} instances.

\paragraph{(a) Relabeling redirects supervision away from the original subgoal.}
Among the \MechCorrectionCheckpoints{} correction-trained instances,
only \MechRetentionMean{}\% of valid multi-step inputs retain the
subgoal recorded in replay. The other targets supervise different
subgoals. HTB's one-step loss instead uses its explicitly designated
subgoal; the hollow squares show this input rule, not measured optimizer
counts. Thus the one-step branch supplies direct supervision that
likelihood correction alone does not preserve.

\paragraph{(b) Preserving the subgoal does not ensure compatible updates.}
For \MechMatchedGradientPolicies{} policy instances, we compare the
unweighted multi-step and one-step losses $\mathcal L_L,\mathcal L_S$
on the same replay minibatch. Both gradients use the same online
parameters $\theta_V$, covering the twin values and all remaining-time heads:
\begin{equation}
 c_{L,S}=\frac{\langle\nabla_{\theta_V}\mathcal L_L,
                         \nabla_{\theta_V}\mathcal L_S\rangle}
 {\|\nabla_{\theta_V}\mathcal L_L\|_2
  \|\nabla_{\theta_V}\mathcal L_S\|_2}.
 \label{eq:value-gradient-cosine}
\end{equation}
Positive values indicate aligned descent directions; zero indicates
orthogonality and negative values indicate conflict.
For Reacher-v1 sparse HTB at $\lambda=.95$, the mean cosine is
\DiagSparseCosR{} with correction and \DiagSparseCosN{} without it.
The losses therefore align more strongly in the no-correction condition
for this setting. Panel (a) motivates explicit subgoal supervision;
panel (b) shows why its compatibility with the multi-step loss is a
separate design question. Task-success effects are tested by the
controlled training comparisons.

\begin{figure}[t]
\centering
\includegraphics[width=\textwidth]{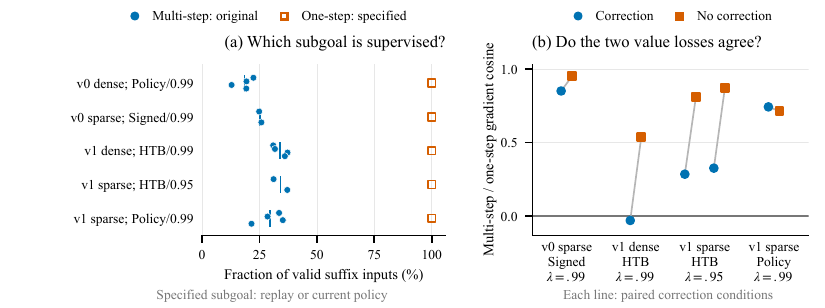}
\caption{\textbf{Direct subgoal supervision and loss compatibility are distinct.}
(a) Circles measure retention of the original subgoal after multi-step
correction; hollow squares mark the designated one-step input rule.
(b) Each line connects correction/no-correction measurements from one
paired repeat, grouped by task, variant and $\lambda$. Each point is
the multi-step--one-step gradient cosine within that model
(Equation~\eqref{eq:value-gradient-cosine}), averaged over two minibatches.
\emph{Signed} and \emph{Policy} are the nonrecursive signed-residual
and current-policy-subgoal variants.}
\label{fig:gradients}
\end{figure}

The performance ablations show that the effect of relabeling depends
on the backup mixture.
On Reacher-v1 sparse at $\lambda=.95$, disabling relabeling raises
HTB success from \VFourSparseCorrectionRight{}\% to
\VFourSparseCorrectionLeft{}\%, a paired improvement of
\VFourSparseCorrectionMean{} points
[\VFourSparseCorrectionLo{}, \VFourSparseCorrectionHi{}].
Across the interior Reacher grid, \ReacherWins{} of
\ReacherConfigs{} configurations have higher mean success without
relabeling. In Push/Fall, the benefit of correction reverses as $\lambda$
changes (Figures~\ref{fig:pushfall}--\ref{fig:reacher}).
This interaction supports HTB's separation of multi-step behavior
correction from one-step supervision at the specified subgoal: their
weights and correction settings control distinct learning effects.

\subsection{Can shared values estimate actions without direct supervision?}
\label{sec:shared-estimation}
We test whether fitting a few actions at a state--subgoal input improves
predictions for another action whose return was excluded from fitting.
The high- and low-level policies are frozen. From states sampled along
their trajectories, we execute a candidate primitive action for one
step, then follow the fixed hierarchy until the episode ends.
The accumulated task reward gives a reference for $Q_t(s,G,g,a)$,
with the training reward scale and boundary discount and no critic bootstrap.

We compare a direct-$Q$ critic with a shared $V-A$ critic. The latter
uses an unconstrained residual to isolate value sharing. Parameter
counts are nearly equal; batches and supervision are identical.
At each $(s,G,g,t)$,
the policy action and a Gaussian-perturbed action supply training returns.
Both critics also receive the same return from following the policy,
supervising $V$ or direct-$Q$ at the policy action.
A third, uniformly sampled action supplies a test return that never
enters a loss. Table~\ref{tab:vseven-representation} reports two tests:
all three action types on separate evaluation trajectories, and the
excluded actions at the training inputs.

After \VSevenUpdates{} updates, sharing lowers mean absolute error
(MAE) in both tests for all \VSevenMAEImprovedUnits{} policy instances.
For actions excluded from fitting, dense-task MAE decreases from
\VSevenDenseDirectWithinStateMean{} to \VSevenDenseSharedWithinStateMean{},
and sparse-task MAE from \VSevenSparseDirectWithinStateMean{} to
\VSevenSparseSharedWithinStateMean{}.
This supports the intended benefit of the shared value: learning from
other actions improves estimates beyond the directly supervised choices.

\begin{table}[t]
\centering
\caption{\textbf{Shared values improve estimates for actions without direct labels.}
MAE against complete rollout returns (mean $\pm$ SD across
\VSevenPolicySeeds{} policy seeds; lower is better).
Left: all three action types on trajectories excluded from fitting and
calibration. Right: uniformly sampled actions at training inputs
$(s,G,g,t)$, whose return labels enter no loss.}
\label{tab:vseven-representation}
\small
\begin{tabular}{lrrrr}
\toprule
 & \multicolumn{2}{c}{Separate test trajectories} & \multicolumn{2}{c}{Actions excluded from fitting}\\
Task & Direct-$Q$ & Shared $V-A$ & Direct-$Q$ & Shared $V-A$\\
\midrule
Reacher-v0 dense & $\VSevenDenseDirectAllMean{}\!\pm\!\VSevenDenseDirectAllSd{}$ & $\VSevenDenseSharedAllMean{}\!\pm\!\VSevenDenseSharedAllSd{}$ & $\VSevenDenseDirectWithinStateMean{}\!\pm\!\VSevenDenseDirectWithinStateSd{}$ & $\VSevenDenseSharedWithinStateMean{}\!\pm\!\VSevenDenseSharedWithinStateSd{}$ \\
Reacher-v1 sparse & $\VSevenSparseDirectAllMean{}\!\pm\!\VSevenSparseDirectAllSd{}$ & $\VSevenSparseSharedAllMean{}\!\pm\!\VSevenSparseSharedAllSd{}$ & $\VSevenSparseDirectWithinStateMean{}\!\pm\!\VSevenSparseDirectWithinStateSd{}$ & $\VSevenSparseSharedWithinStateMean{}\!\pm\!\VSevenSparseSharedWithinStateSd{}$ \\
\bottomrule
\end{tabular}
\end{table}

The test-trajectory result persists after exact-input deduplication.
Subgoal-selection regret varies across tasks, separating this estimation
benefit from decision quality. Appendix~\ref{app:vseven-diagnostics}
provides the full protocols, learning curves and sensitivity analyses.
}{}
\section{Related work}
\label{sec:related}
HIRO corrects recorded behavior \citep{nachum2018hiro}, while HAC uses
path-independent goal-reaching rewards under idealized lower-level
execution \citep{levy2019hac}. Other relabeling methods address preferences
and demonstrations \citep{singh2024piper,singh2025pear}.
Intra-option learning reuses primitive experience across extended actions
\citep{sutton1998intra,sutton1999options}; goal-conditioned values,
hindsight goals and subgoal generation provide complementary mechanisms
\citep{schaul2015uvfa,andrychowicz2017her,zhang2020hrac,kim2021higl}.
Time-aware values and timed subgoals motivate temporal conditioning
\citep{pardo2018time,gurtler2021hits}; HTB uses the next high-level
decision as its continuation boundary.
Dueling networks and normalized advantages share values across actions
\citep{wang2016dueling,gu2016naf}; HTB adds a separately supervised
action-value residual. Twin critics and smoothing follow TD3
\citep{fujimoto2018td3}; the loss mixture differs from Retrace's
importance-weighted return operator \citep{munos2016retrace}.

\section{Discussion and conclusion}
HTB combines specified-subgoal evaluation, time-aware backups and shared
action values while retaining task rewards. Its conditional identities
require Markov states and action coverage; joint actor--critic convergence
remains open. Sparse Maze and MazeW motivate better update-goal selection.
The task and estimation results support its central capability: replay
correction need not determine which subgoal receives value supervision.

\label{lastmain}
\clearpage
\section*{AI use statement}
We used generative AI tools for readability editing of the manuscript. We have not used generative AI tools for proposing or refining the hypotheses, designing or critiquing the methodology or experiments, implementing the methods, interpreting the results, formulating the mathematical claims, or writing the proofs; generating synthetic datasets and qualitative data analysis are not applicable to this work. We have reviewed all AI-assisted work: every number in the paper was read back from the released derived data files by an author, the figure-generation code was checked against those files, and every theorem statement and proof was written and verified by the authors. We take responsibility for the final content of this work, including text, claims, and artifacts produced with the aid of generative AI.

\section*{Ethics statement}
This work is a methodological study of reinforcement learning algorithms.
It involves no human subjects, personal data, or real-world deployment.
We believe that no additional ethical issues arise beyond the general
question of how empirical claims in this field are supported, which this
paper addresses directly.

\section*{Reproducibility statement}
Appendix~\ref{app:theory} gives the assumptions and proofs of the conditional
Bellman results. Appendix~\ref{app:implementation} specifies the backup and
termination conventions. The empirical appendices record sample counts,
measurement protocols, paired comparisons, and analysis limitations.
Appendix~\ref{app:vseven-diagnostics} specifies the frozen-policy
and finite-MDP diagnostics. Validation scripts regenerate the statistics
from retained seed-level results. 

\bibliography{references}

@inproceedings{nachum2018hiro,
  author = {Nachum, Ofir and Gu, Shixiang and Lee, Honglak and Levine, Sergey},
  title = {Data-Efficient Hierarchical Reinforcement Learning},
  booktitle = {Advances in Neural Information Processing Systems},
  volume = {31},
  pages = {3303--3313},
  publisher = {Curran Associates, Inc.},
  year = {2018},
  url = {https://proceedings.neurips.cc/paper/2018/hash/e6384711491713d29bc63fc5eeb5ba4f-Abstract.html}
}

@inproceedings{levy2019hac,
  author = {Levy, Andrew and Konidaris, George and Platt, Robert and Saenko, Kate},
  title = {Learning Multi-Level Hierarchies with Hindsight},
  booktitle = {International Conference on Learning Representations},
  year = {2019},
  url = {https://openreview.net/forum?id=ryzECoAcY7}
}

@inproceedings{gurtler2021hits,
  author = {G{\"u}rtler, Nico and B{\"u}chler, Dieter and Martius, Georg},
  title = {Hierarchical Reinforcement Learning with Timed Subgoals},
  booktitle = {Advances in Neural Information Processing Systems},
  volume = {34},
  pages = {21732--21743},
  publisher = {Curran Associates, Inc.},
  year = {2021},
  url = {https://proceedings.neurips.cc/paper/2021/hash/b59c21a078fde074a6750e91ed19fb21-Abstract.html}
}

@inproceedings{zhang2020hrac,
  author = {Zhang, Tianren and Guo, Shangqi and Tan, Tian and Hu, Xiaolin and Chen, Feng},
  title = {Generating Adjacency-Constrained Subgoals in Hierarchical Reinforcement Learning},
  booktitle = {Advances in Neural Information Processing Systems},
  volume = {33},
  pages = {21579--21590},
  publisher = {Curran Associates, Inc.},
  year = {2020},
  url = {https://proceedings.neurips.cc/paper/2020/hash/f5f3b8d720f34ebebceb7765e447268b-Abstract.html}
}

@inproceedings{kim2021higl,
  author = {Kim, Junsu and Seo, Younggyo and Shin, Jinwoo},
  title = {Landmark-Guided Subgoal Generation in Hierarchical Reinforcement Learning},
  booktitle = {Advances in Neural Information Processing Systems},
  volume = {34},
  pages = {28336--28349},
  publisher = {Curran Associates, Inc.},
  year = {2021},
  url = {https://proceedings.neurips.cc/paper/2021/hash/ee39e503b6bedf0c98c388b7e8589aca-Abstract.html}
}

@inproceedings{schaul2015uvfa,
  author = {Schaul, Tom and Horgan, Daniel and Gregor, Karol and Silver, David},
  title = {Universal Value Function Approximators},
  booktitle = {Proceedings of the 32nd International Conference on Machine Learning},
  series = {Proceedings of Machine Learning Research},
  volume = {37},
  pages = {1312--1320},
  publisher = {PMLR},
  year = {2015},
  url = {https://proceedings.mlr.press/v37/schaul15.html}
}

@inproceedings{sutton1998intra,
  author = {Sutton, Richard S. and Precup, Doina and Singh, Satinder},
  title = {Intra-Option Learning about Temporally Abstract Actions},
  booktitle = {Proceedings of the Fifteenth International Conference on Machine Learning},
  pages = {556--564},
  year = {1998},
  url = {https://icml.cc/Conferences/1998/papers/paper162.html}
}

@article{sutton1999options,
  author = {Sutton, Richard S. and Precup, Doina and Singh, Satinder},
  title = {Between {MDPs} and semi-{MDPs}: A Framework for Temporal Abstraction in Reinforcement Learning},
  journal = {Artificial Intelligence},
  volume = {112},
  pages = {181--211},
  year = {1999},
  doi = {10.1016/S0004-3702(99)00052-1},
  url = {https://doi.org/10.1016/S0004-3702(99)00052-1}
}

@inproceedings{fujimoto2018td3,
  author = {Fujimoto, Scott and van Hoof, Herke and Meger, David},
  title = {Addressing Function Approximation Error in Actor-Critic Methods},
  booktitle = {Proceedings of the 35th International Conference on Machine Learning},
  series = {Proceedings of Machine Learning Research},
  volume = {80},
  pages = {1587--1596},
  publisher = {PMLR},
  year = {2018},
  url = {https://proceedings.mlr.press/v80/fujimoto18a.html}
}

@inproceedings{agarwal2021precipice,
  author = {Agarwal, Rishabh and Schwarzer, Max and Castro, Pablo Samuel and Courville, Aaron and Bellemare, Marc G.},
  title = {Deep Reinforcement Learning at the Edge of the Statistical Precipice},
  booktitle = {Advances in Neural Information Processing Systems},
  volume = {34},
  pages = {29304--29320},
  publisher = {Curran Associates, Inc.},
  year = {2021},
  url = {https://proceedings.neurips.cc/paper/2021/hash/f514cec81cb148559cf475e7426eed5e-Abstract.html}
}

@inproceedings{andrychowicz2017her,
  author = {Andrychowicz, Marcin and Wolski, Filip and Ray, Alex and Schneider, Jonas and Fong, Rachel and Welinder, Peter and McGrew, Bob and Tobin, Josh and Abbeel, Pieter and Zaremba, Wojciech},
  title = {Hindsight Experience Replay},
  booktitle = {Advances in Neural Information Processing Systems},
  volume = {30},
  pages = {5049--5059},
  publisher = {Curran Associates, Inc.},
  year = {2017},
  url = {https://proceedings.neurips.cc/paper/2017/hash/453fadbd8a1a3af50a9df4df899537b5-Abstract.html}
}

@inproceedings{munos2016retrace,
  author = {Munos, R{\'e}mi and Stepleton, Thomas and Harutyunyan, Anna and Bellemare, Marc G.},
  title = {Safe and Efficient Off-Policy Reinforcement Learning},
  booktitle = {Advances in Neural Information Processing Systems},
  volume = {29},
  pages = {1054--1062},
  publisher = {Curran Associates, Inc.},
  year = {2016},
  url = {https://proceedings.neurips.cc/paper/2016/hash/c3992e9a68c5ae12bd18488bc579b30d-Abstract.html}
}

@inproceedings{singh2024piper,
  author = {Singh, Utsav and Suttle, Wesley A. and Sadler, Brian M. and Namboodiri, Vinay P. and Bedi, Amrit Singh},
  title = {{PIPER}: Primitive-Informed Preference-based Hierarchical Reinforcement Learning via Hindsight Relabeling},
  booktitle = {Proceedings of the 41st International Conference on Machine Learning},
  series = {Proceedings of Machine Learning Research},
  volume = {235},
  pages = {45681--45695},
  publisher = {PMLR},
  year = {2024},
  url = {https://proceedings.mlr.press/v235/singh24e.html}
}

@inproceedings{singh2025pear,
  author = {Singh, Utsav and Namboodiri, Vinay Purushothaman},
  title = {{PEAR}: Primitive Enabled Adaptive Relabeling for Boosting Hierarchical Reinforcement Learning},
  booktitle = {International Conference on Learning Representations},
  year = {2025},
  url = {https://proceedings.iclr.cc/paper_files/paper/2025/hash/a42fa12a321b3594317b1a60a6dbccae-Abstract-Conference.html}
}

@inproceedings{wang2016dueling,
  title = {Dueling Network Architectures for Deep Reinforcement Learning},
  author = {Wang, Ziyu and Schaul, Tom and Hessel, Matteo and van Hasselt, Hado and Lanctot, Marc and de Freitas, Nando},
  booktitle = {Proceedings of the 33rd International Conference on Machine Learning},
  series = {Proceedings of Machine Learning Research},
  volume = {48}, pages = {1995--2003}, year = {2016}, publisher = {PMLR},
  url = {https://proceedings.mlr.press/v48/wangf16.html}
}

@inproceedings{pardo2018time,
  title = {Time Limits in Reinforcement Learning},
  author = {Pardo, Fabio and Tavakoli, Arash and Levdik, Vitaly and Kormushev, Petar},
  booktitle = {Proceedings of the 35th International Conference on Machine Learning},
  series = {Proceedings of Machine Learning Research},
  volume = {80}, pages = {4045--4054}, year = {2018}, publisher = {PMLR},
  url = {https://proceedings.mlr.press/v80/pardo18a.html}
}

@inproceedings{gu2016naf,
  title = {Continuous Deep Q-Learning with Model-based Acceleration},
  author = {Gu, Shixiang and Lillicrap, Timothy and Sutskever, Ilya and Levine, Sergey},
  booktitle = {Proceedings of the 33rd International Conference on Machine Learning},
  series = {Proceedings of Machine Learning Research},
  volume = {48}, pages = {2829--2838}, year = {2016}, publisher = {PMLR},
  url = {https://proceedings.mlr.press/v48/gu16.html}
}

@inproceedings{vezhnevets2017fun,
  author = {Vezhnevets, Alexander Sasha and Osindero, Simon and Schaul, Tom and Heess, Nicolas and Jaderberg, Max and Silver, David and Kavukcuoglu, Koray},
  title = {{FeUdal} Networks for Hierarchical Reinforcement Learning},
  booktitle = {Proceedings of the 34th International Conference on Machine Learning},
  series = {Proceedings of Machine Learning Research},
  volume = {70},
  pages = {3540--3549},
  publisher = {PMLR},
  year = {2017},
  url = {https://proceedings.mlr.press/v70/vezhnevets17a.html}
}

@inproceedings{eysenbach2019diayn,
  author = {Eysenbach, Benjamin and Gupta, Abhishek and Ibarz, Julian and Levine, Sergey},
  title = {Diversity Is All You Need: Learning Skills without a Reward Function},
  booktitle = {International Conference on Learning Representations},
  year = {2019},
  url = {https://openreview.net/forum?id=SJx63jRqFm}
}

@inproceedings{sharma2020dads,
  author = {Sharma, Archit and Gu, Shixiang and Levine, Sergey and Kumar, Vikash and Hausman, Karol},
  title = {Dynamics-Aware Unsupervised Discovery of Skills},
  booktitle = {International Conference on Learning Representations},
  year = {2020},
  url = {https://openreview.net/forum?id=HJgLZR4KvH}
}

@inproceedings{bacon2017optioncritic,
  author = {Bacon, Pierre-Luc and Harb, Jean and Precup, Doina},
  title = {The Option-Critic Architecture},
  booktitle = {Proceedings of the AAAI Conference on Artificial Intelligence},
  volume = {31},
  number = {1},
  pages = {1726--1734},
  year = {2017},
  doi = {10.1609/aaai.v31i1.10916},
  url = {https://ojs.aaai.org/index.php/AAAI/article/view/10916}
}

@inproceedings{nachum2019representation,
  author = {Nachum, Ofir and Gu, Shixiang and Lee, Honglak and Levine, Sergey},
  title = {Near-Optimal Representation Learning for Hierarchical Reinforcement Learning},
  booktitle = {International Conference on Learning Representations},
  year = {2019},
  url = {https://openreview.net/forum?id=H1emus0qF7}
}

@inproceedings{park2023hiql,
  author = {Park, Seohong and Ghosh, Dibya and Eysenbach, Benjamin and Levine, Sergey},
  title = {{HIQL}: Offline Goal-Conditioned {RL} with Latent States as Actions},
  booktitle = {Advances in Neural Information Processing Systems},
  volume = {36},
  pages = {34866--34891},
  publisher = {Curran Associates, Inc.},
  year = {2023},
  url = {https://proceedings.neurips.cc/paper_files/paper/2023/hash/6d7c4a0727e089ed6cdd3151cbe8d8ba-Abstract-Conference.html}
}

@inproceedings{hafner2022director,
  author = {Hafner, Danijar and Lee, Kuang-Huei and Fischer, Ian and Abbeel, Pieter},
  title = {Deep Hierarchical Planning from Pixels},
  booktitle = {Advances in Neural Information Processing Systems},
  volume = {35},
  pages = {26091--26104},
  publisher = {Curran Associates, Inc.},
  year = {2022},
  url = {https://proceedings.neurips.cc/paper_files/paper/2022/hash/a766f56d2da42cae20b5652970ec04ef-Abstract-Conference.html}
}
\bibliographystyle{iclr2027_conference}
\clearpage
\appendix
\section{Remaining-time Bellman construction and its scope}
\label{app:theory}

\subsection{Fixed-policy value definition}

Fix a stationary low-level policy $\mu$ and high-level policy $\pi_H$, a
high-level decision interval $H$, a task goal $G$, and $0\leq\gamma<1$. The environment is Markov
in $s$; if termination depends on a clock, that clock is included in the
state for the statement below. The joint law of next state, scaled task
reward, and terminal indicator is $P(ds',dr,dd\mid s,a,G)$. This law does
not depend on the high-level query $g$. Assume bounded rewards. Relative
subgoals are transported by $h(s,g,s')=g+\phi(s)-\phi(s')$.

Starting from $(s,G,g,t)$, the low-level policy acts under the transported subgoal until
termination or the scheduled boundary after $t$ primitive steps. At a
nonterminal boundary, the high-level policy chooses a new subgoal using $\pi_H$ and
resets the countdown to $H$. If $b(k)$ counts boundaries crossed before
primitive step $k$, define
\begin{equation}
 V_t^{\mu,\pi}(s,G,g)
 =\mathbb E\!\left[\sum_{k=0}^{T-1}\gamma^{b(k)}r_k
 \;\middle|\;s_0=s,G,g_0=g,t_0=t\right],
 \label{eq:boundary-return}
\end{equation}
where $T$ is the terminal time, with an infinite sum for continuing tasks.
The bound $H\|r\|_\infty/(1-\gamma)$ ensures existence. This definition
evaluates the specified low-level policy, including during learning.
Define $Q_t^{\mu,\pi}(s,G,g,a)$ by fixing the first primitive action to $a$
and following the same hierarchy afterward. Superscripts are omitted when
the policies are understood.

\begin{proposition}[Remaining-time Bellman identity]
\label{prop:countdown}
The values in Equation~\eqref{eq:boundary-return} satisfy
Equations~\eqref{eq:primitive-bellman}--\eqref{eq:countdown-bellman}.
Unrolling from a countdown $t$ reaches the next high-level decision boundary with an
undiscounted sum of within-segment rewards and a single factor $\gamma$
on the continuation. On a trajectory that terminates earlier, the
continuation is zero.
\end{proposition}

\begin{proof}
Condition on the first action and the transition $(s',r,d)$. If $d=1$, the
return is $r$. If $d=0$ and $t>1$, no high-level decision boundary has been crossed, so
the remaining return is $V_{t-1}(s',G,h(s,g,s'))$ without a discount factor.
If $d=0$ and $t=1$, a boundary is crossed and the remaining return is
$\gamma V_H(s',G,\pi_H(s',G))$. Taking the conditional expectation gives
Equation~\eqref{eq:primitive-bellman}; integrating the first action under
$\mu$ gives Equation~\eqref{eq:countdown-bellman}. Repeat the same argument
until the first boundary or terminal transition. All intermediate
continuations have multiplier one. Only a nonterminal boundary contributes
$\gamma$, proving the unrolled form.
\end{proof}

The identity gives the implementation's boundary-discounted objective:
within-segment rewards have unit weights, and each subsequent high-level
segment contributes another factor $\gamma$. The long replay target in
Equation~\eqref{eq:long-target} has this unrolled form. A replay suffix
follows its historical low-level policy, so equality to the current-policy
expectation requires on-policy sampling or a valid off-policy estimator.
Likelihood-based candidate selection provides an approximate correction.

\subsection{Remaining time distinguishes subgoal values}
\label{app:time-aliasing}

\begin{proposition}[Remaining-time aliasing]
\label{prop:time-aliasing}
Suppose the retained input $x=(s,G,g,a)$ does not distinguish remaining
times $t$ and $u$. If $Q_t(x)\ne Q_u(x)$, every predictor $q(x)$ satisfies
$\max\{|q(x)-Q_t(x)|,|q(x)-Q_u(x)|\}\ge |Q_t(x)-Q_u(x)|/2$.
\end{proposition}

\begin{proof}[Proof of Proposition~\ref{prop:time-aliasing}]
For the same retained input $x$, the triangle inequality gives
$|Q_t(x)-Q_u(x)|\le |Q_t(x)-q(x)|+|q(x)-Q_u(x)|$.
The larger of the two errors is at least half their required total.
The midpoint $q(x)=(Q_t(x)+Q_u(x))/2$ attains the bound for this pair.
\end{proof}

A stationary Markov example realizes a nonzero gap. Let the physical
environment have one self-loop state $s$, two primitive actions
$a_1,a_0$, and rewards $r(s,a_1,G)=1$, $r(s,a_0,G)=0$ for fixed $G$.
There is no termination and $\phi(s)=0$, so relative-goal transport
leaves each goal unchanged. Use goals $g_1=1$ and $g_0=0$, with
deterministic low-level policy $\mu(s,g_1)=a_1$ and $\mu(s,g_0)=a_0$.
The stationary boundary high-level policy always chooses $g_0$.
For any $H\ge2$ and $0\le\gamma<1$, the boundary value under $g_0$ is
zero: $V_H(s,G,g_0)=\gamma V_H(s,G,g_0)$.

Start a queried goal $g_1$ with remaining time $t$ and first action $a_1$.
It yields exactly $t$ unit rewards before the high-level policy switches to $g_0$;
all subsequent rewards vanish. Hence
\begin{equation}
 V_t(s,G,g_1)=Q_t(s,G,g_1,a_1)=t,\qquad
 Q_1(s,G,g_1,a_1)=1,\quad Q_2(s,G,g_1,a_1)=2.
 \label{eq:aliasing-example}
\end{equation}
The physical state, task goal, queried subgoal, and first action are
identical, yet the values differ. Their distinction is the time at
which control returns to the high-level policy.

The representational requirement is that remaining time be recoverable
from the value input. An elapsed clock combined with a known schedule
can also satisfy it. The lower bound applies when two distinct phases
are aliased; it establishes the need to distinguish their boundary
semantics, independently of the neural encoding used to do so.

\subsection{Primitive-action conditioning permits off-policy primitive-step updates}

For any bounded fixed continuation family $F$, define
\begin{align}
 b^F_t(s,G,g,a,s',r,d)
 &=r+(1-d)C^F_t(s',G,h(s,g,s')),\\
 (\mathcal B_tF)(s,G,g,a)
 &=\int b^F_t(s,G,g,a,s',r,d)\,
 P(ds',dr,dd\mid s,a,G),
 \label{eq:conditional-operator}
\end{align}
where $C^F_t$ replaces $V$ by $F$ in
Equation~\eqref{eq:countdown-bellman}, with the high-level policy held fixed.
The continuation may depend on the query even though the physical transition
law does not.

\begin{proof}[Proof of Proposition~\ref{prop:conditional-replay}]
Let $\beta$ be any data-collection low-level policy, including a low-level policy from an earlier
training iteration. The sampling and query-selection assumptions preserve
the outcome law after conditioning on selection and on the query. Under
the Markov assumption, conditioning on $(s,a,G)$
removes the dependence of $(s',r,d)$ on $\beta$ and on the collection goal.
Its conditional law is $P(\cdot\mid s,a,G)$. Substituting this law into the
conditional expectation of $b^F_t$ gives $\mathcal B_tF$. This equality
holds for every fixed query $g$, including a query different from the
collection goal. It does not change the distribution of the sampled action
$a$. To integrate under a new low-level policy $\mu$, the conditional function must be
identified on actions to which $\mu$ assigns probability. That step needs
coverage or a justified generalization assumption.
\end{proof}

The invariant quantity is the conditional one-step sampling law; the
continuation value tracks the policies being evaluated. Replay-goal and
frozen high-level-policy proposal queries use pre-outcome information
and satisfy the query-selection condition. Goals selected by likelihood
correction from a realized suffix can depend on its outcomes, so their
multi-step estimator uses the separate approximate correction described
above. Evaluating a current low-level policy additionally requires coverage
of its actions or a justified generalization assumption, as in the proof.

\subsection{Shared value updates and action-dependent estimation errors}
\label{app:shared-value}

For exact fixed-policy values, put
$A_t^{\mu,\pi}=V_t^{\mu,\pi}-Q_t^{\mu,\pi}$. For one estimated
value head $V_{t,j}$ and residual network $A_t$, define
$e_V=V_{t,j}-V_t^{\mu,\pi}$, $e_A=A_t-A_t^{\mu,\pi}$, and
$e_Q=(V_{t,j}-A_t)-Q_t^{\mu,\pi}$, all at the same appropriate inputs. Then
\begin{equation}
 e_Q(s,G,g,a)=e_V(s,G,g)-e_A(s,G,g,a),\quad
 |e_Q(s,G,g,a)|\le |e_V(s,G,g)|+|e_A(s,G,g,a)|.
 \label{eq:shared-error-decomposition}
\end{equation}
Substitution of $Q_t^{\mu,\pi}=V_t^{\mu,\pi}-A_t^{\mu,\pi}$ proves both statements.
Holding $A_t$ fixed, a change in $V_{t,j}$ produces the same additive
change in every reconstructed action value at $(s,G,g,t)$, including
actions without direct supervision, as stated
in Equation~\eqref{eq:shared-value-update}. When both networks change,
$\Delta\widehat Q_{t,j}(a)=\Delta V_{t,j}-\Delta A_t(a)$.
Here the shared bias is an action-independent \emph{value offset}, not
a particular neural-layer bias parameter.

The decomposition exposes two contributions to action-value accuracy:
a shared value error and an action-specific residual error. Updating $V$
transfers its common component across actions, while accuracy at actions
absent from the data still depends on generalization of $A$. Thus improved
$V$ alone supplies no uniform action-value or high-level ranking guarantee.
A direct neural action-value network may also change unsampled predictions;
the distinction here is the explicit shared component.
The implementation assigns separate parameters to $V$ and $A$ and detaches
every target. The residual-regression loss therefore has zero gradient
with respect to online $V$; the long/short value losses train that component.

\subsection{A signed residual representation}

Let $q_t(s,G,g)$ be a fixed reference value and define
\begin{equation}
 A^F_t(s,G,g,a)
 =q_t(s,G,g)-(\mathcal B_tF)(s,G,g,a).
 \label{eq:signed-residual}
\end{equation}
Then the following reconstruction holds for any query and low-level policy:
\begin{equation}
 q_t(s,G,g)-\mathbb E_{a\sim\mu(\cdot\mid s,g)}A^F_t(s,G,g,a)
 =\mathbb E_{a\sim\mu(\cdot\mid s,g)}
 (\mathcal B_tF)(s,G,g,a).
 \label{eq:residual-reconstruction}
\end{equation}
For a deterministic low-level policy the action expectation is a single evaluation at
$\mu(s,g)$. Equation~\eqref{eq:residual-reconstruction} follows by
substitution, so it holds independently of how the query was selected.
In a population least-squares regression with sufficient coverage and a
realizable function class, the target $q_t-b^F_t$ identifies $A^F_t$.
Together with Proposition~\ref{prop:conditional-replay}, this gives an
explicit route from replay transitions to the queried low-level policy return.

The identification statement above applies to realizable population
least-squares regression. Practical HTB uses smooth-$L_1$ regression,
finite neural networks, target networks, and a nonnegative residual
except in the signed-residual ablation. Nonnegativity is a stabilizing
constraint on this estimator: the exact $V_t-Q_t$ can be negative when an
individual action outperforms the low-level policy's value. The signed-residual
ablation retains the signed target and output to test this constraint.

For a deterministic low-level policy with exact policy values, the signed residual at
its own action satisfies $A_t^{\mu,\pi}(s,G,g,\mu(s,g))=0$. Consequently, the extra
next-state signed-residual term used to motivate HTB leaves the exact
fixed-policy Bellman identity unchanged. Away from this limit, HTB changes
the fitted recursion by replacing the next reference critic with its
residual-corrected value. We assess that estimator choice empirically.

\begin{proposition}[Action-conditioned residual recursion]
\label{prop:residual-recursion}
Fix a deterministic low-level policy $\mu$, a bounded boundary value $B(s,G)$, and
bounded reference functions $q_t(s,G,g)$ for $t=1,\ldots,H$. Define
$Q^{\mu,B}_t$ as the conditional return that executes its first action $a$,
then follows $\mu$ until the next scheduled boundary or termination, and
receives continuation $\gamma B$ at a nonterminal boundary. There is a
unique family of exact signed residuals, denoted here by
$A_t=q_t-Q^{\mu,B}_t$, satisfying
\begin{align}
 A_1(s,G,g,a)
 &=q_1(s,G,g)-\mathbb E[r+(1-d)\gamma B(s',G)\mid s,a,G],\\
 A_t(s,G,g,a)
 &=q_t(s,G,g)-\mathbb E\!\left[
 r+(1-d)\left\{q_{t-1}(s',G,g')\right.\right.\nonumber\\
 &\hspace{5em}\left.\left.
 -A_{t-1}(s',G,g',\mu(s',g'))\right\}\mid s,a,G\right],
 \qquad t>1,
 \label{eq:residual-recursion}
\end{align}
where $g'=h(s,g,s')$. The reconstructed value
$q_t(s,G,g)-A_t(s,G,g,\mu(s,g))$ equals the value of the specified subgoal
under $\mu$ and boundary $B$. This holds for every query on which the
conditional expectations are identified.
\end{proposition}

\begin{proof}
For $t=1$, the definition of $Q^{\mu,B}_1$ gives the first equation and
determines $A_1$ uniquely. Suppose the statement holds at $t-1$.
Then
\begin{equation}
 q_{t-1}(s',G,g')-A_{t-1}(s',G,g',\mu(s',g'))
 =Q^{\mu,B}_{t-1}(s',G,g',\mu(s',g')).
\end{equation}
The right side is the continuation return of the low-level policy after the first
transition. Substitution into Equation~\eqref{eq:residual-recursion} gives
$A_t=q_t-Q^{\mu,B}_t$ and uniquely determines $A_t$ from $A_{t-1}$.
Induction proves the family and its reconstruction at the current low-level policy's
action. Proposition~\ref{prop:conditional-replay} applies to each conditional
expectation under its coverage assumptions.
\end{proof}

The proposition characterizes the signed conditional recursion.
Practical HTB instantiates it with frozen target references and a
target-high-level decision boundary value during each update. Its projected
targets and delayed residual estimates give an approximate, constrained
version of that recursion. The no-next-residual and signed-residual variants
use $q_{t-1}$ directly as their next-step continuation, omitting the
residual correction inside that continuation.

\subsection{Nonnegative projection gives a one-sided correction}

The nonnegative output guarantees $\widehat Q_{t,j}=V_{t,j}-A_t\le V_{t,j}$
pointwise, regardless of the quality of the learned reference. A stronger
scalar identity holds when the projected residual target is fitted exactly.

\begin{proposition}[One-sided residual correction]
\label{prop:projection}
For real values $q,b$, let $p(q,b)=q-[q-b]_+$. Then
\begin{equation}
 p(q,b)=\min(q,b),\qquad p(q,b)\leq q,\qquad
 |p(q,b)-p(q,b')|\leq|b-b'|.
 \label{eq:projection-properties}
\end{equation}
If an estimated residual $A$ approximates $[q-b]_+$ to error at most
$\epsilon$, then $|q-A-p(q,b)|\leq\epsilon$.
\end{proposition}

\begin{proof}
If $q\geq b$, subtraction of $q-b$ gives $b$; if $q<b$, the clipped
residual is zero and the result is $q$. The minimum with a fixed scalar is
nonexpansive in its other argument. Finally,
$|q-A-p(q,b)|=|A-[q-b]_+|\leq\epsilon$.
\end{proof}

The correction is nonexpansive with its reference held fixed. This
local bound leaves the coupled actor--critic iteration, which also updates
the reference, outside its scope. Projection before taking expectations differs from
projecting a conditional mean: for a random backup $B$,
$q-\mathbb E[(q-B)_+]=\mathbb E[\min(q,B)]$, which can be strictly below
$\min(q,\mathbb E[B])$. The practical target therefore has a conservative
component whose cost and benefit depend on the task and on $\lambda$.

\paragraph{Scope for maximization amplification.}
If noisy candidate maximization raises a scalar backup $b$ above $q$,
the projected reconstruction removes that upward increment. The remaining
error depends on the reference: $q=3$, $b=4$, and true value $1$ still give
the overestimate $\min(q,b)=3$, while a low reference can suppress a
legitimate improvement. The implementation's high-level policy maximizes
$V_{t,1}$, and its low-level policy uses the intrinsic-reward critic for
policy optimization. The local cap
therefore characterizes the correction, not the global bias of either
policy-optimization process.

\subsection{Intermediate error sources in primitive and suffix backups}
\label{app:multistep-error}

Consider a fixed low-level policy up to the next scheduled boundary and a fixed
boundary continuation $B$. Let $\widehat B$ satisfy
$\|\widehat B-B\|_\infty\le\epsilon_B$.
Suppose an approximate primitive recursion introduces an additive
operator residual $\eta_j$ at remaining time $j$, with
$\|\eta_j\|_\infty\le\epsilon_j$.
For $j>1$ the correct within-segment expectation operator is
nonexpansive in its continuation; at $j=1$ the boundary continuation
has coefficient $\gamma$ and a terminal mask. Consequently,
\begin{equation}
 \|\widehat V_t-V_t\|_\infty
 \le \sum_{j=1}^{t}\epsilon_j+\gamma\epsilon_B.
 \label{eq:primitive-error-budget}
\end{equation}
The base error is at most $\epsilon_1+\gamma\epsilon_B$.
Each earlier step adds at most $\epsilon_j$ because within-segment
discounting is one. Induction proves Equation~\eqref{eq:primitive-error-budget}.

For a suffix sampled under that same low-level policy, the observed reward sum
contains none of these learned intermediate continuations.
Replacing only $B$ by $\widehat B$ in its direct target gives
\begin{equation}
 \left|\mathbb E[\widehat Y^{\mathrm L}-Y^{\mathrm L}
       \mid s,G,g,t]\right|\le\gamma\epsilon_B.
 \label{eq:multistep-error-budget}
\end{equation}
The reward terms cancel pathwise, leaving the masked boundary error.
Under the stated operator-error bounds and common sampling policy,
Equations~\eqref{eq:primitive-error-budget}--\eqref{eq:multistep-error-budget}
isolate the dependence on learned intermediate continuations. Actual
estimator errors additionally depend on reward variance, fitting error,
and any mismatch between historical and current trajectory distributions;
the $V-A$ architecture alone supplies no bound on these terms. HTB's
truncated relabeling adapts historical suffixes approximately, while its
mixture retains the primitive route for an explicitly specified subgoal.

\subsection{Direct value updates for specified subgoals}
\label{app:specified-updates}

\begin{proposition}[Query routing]
\label{prop:query-routing}
Fix a replay suffix, detached query $g$, detached short target $\widehat y^{\mathrm S}$,
and $\lambda<1$. For each critic, the short term in
Equation~\eqref{eq:mixed-loss} supplies the output derivative
\begin{equation}
 (1-\lambda)\,\ell'_1\!\left(V_{t,j}(s,G,g),\widehat y^{\mathrm S}\right)
 \label{eq:query-gradient}
\end{equation}
at the queried input $(s,G,g)$, independently of the long-path goal selected
by correction. If its residual is nonzero, this derivative is nonzero for
the smooth-$L_1$ loss.
\end{proposition}

\begin{proof}
Differentiate the short term while holding its query and target fixed.
The chain rule gives Equation~\eqref{eq:query-gradient}. The corrected
goal is an argument only of the long term and cannot replace the short
term's input. The smooth-$L_1$ derivative vanishes only at zero residual.
\end{proof}

The proposition identifies which input receives a direct loss derivative.
The resulting parameter update combines both losses, whose gradients can
cancel, and parameter sharing can change predictions at other inputs.
A high-level policy proposal can serve as $g$ because the implementation
detaches it before forming the residual and critic losses. At $\lambda=1$,
the short derivative is zero, so this supervision mechanism is absent.

\section{Implementation of the remaining-time updates}
\label{app:implementation}

\begin{algorithm}[t]
\caption{Learning from a high-level replay segment}
\label{alg:replay-update}
\small
\begin{algorithmic}[1]
\State Sample a segment $\tau$ of observed length $L\leq H$.
\For{each observed step $i=0,\ldots,L-1$}
  \State Set $t=H-i$ and choose query $g_i$ (replay or current high-level policy).
  \State Set $\widetilde g_i$ by suffix correction, or keep the replay goal.
  \State Form detached $y^A_{t,i}$, $\widehat y^{\mathrm S}_{t,i}$ and $y^{\mathrm L}_{t,i}$;
  \Statex \hspace{\algorithmicindent} using masked returns and excluding unobserved padding.
\EndFor
\State Fit $A$ at recorded actions using $y^A$.
\State Fit twin critics with $\lambda\mathcal L_{\rm L}+(1-\lambda)\mathcal L_{\rm S}$.
\State Improve the time-indexed actors; Polyak-update target networks.
\end{algorithmic}
\end{algorithm}

\subsection{Networks, losses, and update order}

The Markov-state construction is implemented with each benchmark's
observation protocol. Push/Fall observations contain ant position and
velocity together with elapsed time and omit the movable object's state.
Because hidden object configurations can depend on collection history,
the resulting observation-based update is an approximation to
Proposition~\ref{prop:conditional-replay}; its conditional-law guarantee
requires a sufficient Markov state. All compared Push/Fall conditions use
the same observation protocol.

For $H=10$, each horizon has its own actor, twin values $V_{t,j}$, and
residual network, with width-300 ReLU hidden layers. Their inputs are
$(s,G)$, $(s,G,g)$, and $(s,G,g,a)$, respectively. Parameters are
stored in a batched horizon axis, but no trainable weights are shared across
horizons. The residual's final activation is ReLU except in the signed-residual
variant, which uses a linear output. The low-level policy uses the same observation preprocessing in
candidate correction and residual queries as it uses to select actions.

For each sampled high-level segment, we construct all observed suffixes.
Training proceeds in the following order:
\begin{enumerate}
\item Select each suffix's long-path goal, and compute target high-level policy
boundary values using the minimum of the two target critics.
\item Choose the short-path query, transport it by one primitive step, and
form the one-step residual target. Apply the terminal mask after selecting
the boundary or within-segment continuation.
\item Update the online residual network at the recorded primitive action
using smooth-$L_1$ loss.
\item Compute the long and short critic losses. The short target uses the
\emph{target} residual network, including in the iteration immediately after
the online residual update. Minimize their $\lambda$-weighted combination.
\item Update each high-level actor to maximize its first online critic and then
Polyak-update the target actors, critics, and residual networks.
\end{enumerate}
Within a horizon, we average loss over valid samples; we sum these averages
over active horizons. The rollout calls only the $H$-step actor. Actors for
shorter remaining durations are trained at replay suffix states and provide query
goals for the current-policy-subgoal variant; they do not resample the high-level policy's goal during a segment.

\subsection{Exact distinctions between variants}

All variants share the long target in Equation~\eqref{eq:long-target} and
short target in Equation~\eqref{eq:short-target}. For all nonrecursive variants,
replace the $t>1$ branch in Equation~\eqref{eq:recursive-continuation} by
$\bar V_{t-1}(s_{i+1},G,g_i^+)$. For the signed-residual variant, remove the positive-part
operation from Equation~\eqref{eq:gap-target} and use an unconstrained
residual output. For the current-policy-subgoal variant, set
\begin{equation}
 g_i=\operatorname{stopgrad}(\pi_t(s_i,G)),\qquad
 g_i^+=g_i+\phi(s_i)-\phi(s_{i+1}).
\end{equation}
The same detached query enters the residual and short critic paths. These
losses do not create an extra actor objective through the query. In all replay-query variants,
$g_i=g_0^{\mathrm{orig}}+\phi(s_0)-\phi(s_i)$.

The full HTB next-state residual is evaluated at the current low-level policy's deterministic
action with no smoothing noise. The short critic target evaluates the
target residual at the current low-level policy's action plus clipped Gaussian
smoothing noise, with the resulting action clipped to its bounds. Boundary
targets use the target high-level policy with its own clipped smoothing noise and
action bounds. Neither expression uses a target low-level actor.

When $\lambda=1$, the short critic term is skipped. The residual network
continues training but does not affect actor or critic targets. A seeded
implementation test verifies equality of actor and critic updates with the
independent remaining-time baseline. This is an implementation regression test,
not a claim that separately collected HIRO and countdown runs are identical.

\subsection{Suffix correction and early termination}

For each valid suffix starting at $i$, the candidate set contains its
transported original subgoal, the observed displacement
$\phi(s_L)-\phi(s_i)$, and $K$ Gaussian candidates centered on that
displacement. The Push/Fall protocol uses $K=8$, for ten total
candidates; the Reacher/Maze implementation defaults to $K=10$, for twelve
total candidates. Random candidates are clipped to the subgoal bounds. The
selection rule is
\begin{equation}
 \widetilde g_i=\arg\min_{g\in\mathcal C_i}
 \sum_{k=i}^{L-1}
 \left\|a_k-\mu\!\left(s_k,g+\phi(s_i)-\phi(s_k)\right)\right\|_2^2.
 \label{eq:suffix-selection}
\end{equation}
Each horizon independently samples and scores its candidates on its observed
suffix. Disabling correction returns the transported replay subgoal for
the long path; it leaves the short-path construction unchanged. Since
candidate selection consumes random numbers, paired seeds across correction
settings do not imply identical realized replay or exploration trajectories.

The countdown is scheduled time to a high-level policy decision, not observed time
to termination. If a segment terminates after $L<H$ primitive actions, only
branches $t=H,H-1,\ldots,H-L+1$ are valid. The last observed transition
has $d_i=1$, so its residual target contains only $r_i$ after the reference
value. All suffix long targets end at that same terminal state and have
zero continuation. Unobserved padding is excluded from losses and from
candidate likelihoods. If termination coincides with $t=1$, the terminal
mask overrides boundary bootstrapping. These conventions preserve the
semantics of a high-level policy that planned an $H$-step segment before termination.

\section{Training and evaluation protocol}
\label{app:training-protocol}

The Push/Fall experiments use a shared environment protocol and PyTorch
implementation. All compared methods use the same environment semantics
and optimizer/update conventions. This comparison does not imply an exact
reproduction of the original HIRO training software.
Table~\ref{tab:training-protocol} reports the common training settings.

\begin{table}[h]
\centering\small
\caption{\textbf{Push/Fall training settings.} Values are shared by HIRO and
the remaining-time variants unless the component exists only in the latter.}
\label{tab:training-protocol}
\begin{tabular}{ll}
\toprule
Setting & Value\\
\midrule
Environment steps & $10{,}000{,}000$\\
High-level period / update interval & $10$ / $10$ primitive steps\\
Batch size (both levels) & $128$\\
Replay capacity (both levels) & $200{,}000$ transitions each\\
Actor learning rate (both levels) & $10^{-4}$\\
Critic learning rate (both levels) & $10^{-3}$\\
Action-residual learning rate & $10^{-3}$\\
High-level boundary / low-level discount & $0.99$ / $0.95$\\
High-level / low-level reward scale & $0.1$ / $1.0$\\
Polyak coefficient & $0.005$\\
Target action noise standard deviation / clip & $0.2$ / $0.5$\\
Training exploration standard deviation & $1.0$\\
Random / total correction candidates & $8$ / $10$\\
Evaluation interval / episodes & $50{,}000$ steps / $50$\\
Episode duration & $500$ primitive steps\\
Warmup episodes & $0$\\
\bottomrule
\end{tabular}
\end{table}

Training uses Adam and smooth-$L_1$ critic losses. Updates are executed in
episode batches: the low-level policy receives one gradient iteration per collected
primitive step, and the high-level policy receives one per high-level policy update interval.
The actor and target networks update at every corresponding gradient
iteration. The high-level policy critic and action-residual optimizers use weight
decay $10^{-4}$.

Evaluation calls the high-level policy and low-level policy without exploration noise and
without parameter updates. Push/Fall episodes always run for the fixed
duration: reaching the task goal does not end evaluation early. Success is
measured on the final step, and return sums the environment rewards over
the episode. The environment uses dense distance reward, relative
subgoals, and no stochastic position perturbation.

The Reacher/Maze evaluation instead stops an episode at its first
successful step and reports success and steps to termination.
The comparisons use one million environment steps and eight paired seeds.
We do not pool the
Reacher/Maze success measure with Push/Fall terminal success. The
Reacher tail statistic averages the retained evaluation points within the
reported terminal training window before averaging over seeds.

\IfFileExists{sections/empirical_appendix.tex}{\section{Empirical protocol and complete paired results}
\label{app:empirical}

\paragraph{Metrics and analysis units.}
Push and Fall use the curve entries \texttt{eval0/avg\_ep\_rew} and
\texttt{eval0/perc\_env\_goal\_achieved} at exactly \PFTargetMillions{} million environment
steps. Each configuration is evaluated with \PFSeeds{} seeds
(777--781) at this evaluation point.
The comparison includes \PFConfigs{} logical configurations including
the two HIRO references, and \PFPairs{} correction/no-correction pairs.
The method variants have identical seed-level returns and success at
$\lambda=1$; removing the repeated method labels leaves \PFUniquePairs{}
distinct comparison rows. The figure shows this endpoint once per environment,
and its uncertainty calculation counts each distinct seed outcome once.

For Reacher, the metric is \texttt{eval0/perc\_env\_goal\_achieved}, averaged
over the last twenty evaluation entries ending at exactly one million
steps. Each eligible comparison contains eight paired seeds (777--784).
The primary interior-$\lambda$ set contains \ReacherConfigs{} configurations
and \ReacherPairedEntries{} configuration--seed pairs. It excludes the endpoints
because $\lambda=1$ removes short-path critic supervision and $\lambda=0$
removes long-path critic supervision. The latter may still execute correction
and consume random numbers, so endpoint differences need not measure a pure
change in the weighted critic objective. The mechanism analysis focuses on
the interior grid.

\paragraph{Uncertainty in figures.}
We report variation across seeds and configuration choices explicitly, in line
with the evaluation concerns discussed by \citet{agarwal2021precipice}.
The plotted intervals are percentile intervals with nominal coverage $95\%$.
For each Push/Fall configuration, we enumerate all \PFExactResamples{}
with-replacement resamples of its five paired seed differences and take the
$2.5$th and $97.5$th percentiles using linear interpolation. Enumeration
removes Monte Carlo variation from bootstrap computation; finite-seed
coverage remains nominal. For the Reacher correction-effect strata in
Figure~\ref{fig:reacher}, a fixed generator seed resamples
configurations and then paired seed differences within each chosen configuration
for ten thousand replicates. Configuration means receive equal weight.
These intervals describe this finite grid and have no simultaneous multiple-test
coverage. The full sweep characterizes variation across settings.

\paragraph{Analysis-set sensitivity.}
The full grid contains 88 logical comparisons. A single-setting exclusion
analysis removes Reacher-v1 sparse current-policy-subgoal/$\lambda=0$,
leaving 87 comparisons. This setting has eight paired seeds and is included
in the full-grid analysis.
Separately, HTB and its no-next-residual variant at $\lambda=1$ have identical paired seed values in every
reward--environment stratum, so collapsing those repeated endpoint rows gives
84 configurations. Table~\ref{tab:reacher-sensitivity} compares the full,
single-setting-exclusion, deduplicated, and interior analysis sets. The correction-effect
aggregate uses the interior grid, which excludes both endpoints.

\begin{table}[t]
\centering
\caption{\textbf{Reacher conclusions persist across analysis sets.}
Gains are mean success-rate differences in percentage points; wins count
positive configuration means.}
\label{tab:reacher-sensitivity}
\small
\begin{tabular}{lrrr}
\toprule
Analysis set & Settings & Wins & Mean gain\\
\midrule
Interior mixing weights & 64 & 60 & +26.82 \\
Full grid, deduplicated & 84 & 68 & +20.88 \\
Full logical grid & 88 & 71 & +20.49 \\
Single-setting exclusion, logical grid & 87 & 71 & +20.74 \\
Single-setting exclusion, deduplicated & 83 & 68 & +21.14 \\
\bottomrule
\end{tabular}
\end{table}

\paragraph{Reference setting and benchmark interpretation.}
Table~\ref{tab:fixed-reference} gives the same fixed HTB/$\lambda=.95$ setting in
both environments and correction modes. These are comparisons against this
repository's HIRO implementation under the local paper-environment protocol
specified in Appendix~\ref{app:training-protocol}.
Both sides of each Push/Fall comparison use the same five seed identifiers;
additional HIRO training seeds remain outside these paired analyses.

An additional within-Fall interaction for the no-next-residual variant is the difference between
the no-correction effects at $\lambda=.6$ and $.9$. Its empirical bootstrap
uses the same resampled seed indices in all four arms, enumerating the
five-seed resamples. This exploratory contrast has a nominal interval without
simultaneous multiple-test coverage.

\paragraph{Routing and gradient diagnostics.}
Subgoal-retention diagnostics evaluate each fitted policy instance separately.
Each analysis samples four batches of 128 high-level transitions,
without replacement inside each batch; batches can overlap. Valid suffix
queries use the implementation's residual subgoal at their own start state.
The retention indicator is
$\mathbb{1}[\|g^{\mathrm{orig}}-\widetilde g\|_2\leq10^{-6}]$.
We first compute retention within each trained model, then average over the
correction-enabled models. Applying correction to no-correction models
is a counterfactual diagnostic; their actual long-path update retains the original
goal. The diagnostic measures subgoal retention on reconstructed valid
suffixes; value generalization is evaluated separately in
Appendix~\ref{app:vseven-diagnostics}.

For gradient diagnostics, the unweighted multi-step loss $\mathcal L_L$
and one-step loss $\mathcal L_S$ use the same sampled replay minibatch
and valid suffix mask. Their Huber losses sum the two online value
critics, with the implementation's common remaining-time reduction.
Targets are detached, and each loss is differentiated separately with
respect to the same $\theta_V$: all parameters of both online value
critics across the ten independent remaining-time heads.
Neither the residual network nor either actor is part of this gradient
vector; target-network parameters are excluded as well.
The long target uses the recorded suffix reward and target-critic boundary
bootstrap; the short target uses the current target critic and residual
network. Target-high-level policy smoothing is sampled according to the diagnostic's
fixed random seed. The offline short-path query uses the current low-level policy's
unsmoothed action, while training adds action smoothing. The cosines therefore
describe this reconstructed-loss diagnostic, not the exact stochastic
training gradients. Equation~\eqref{eq:value-gradient-cosine} concatenates
these parameters before computing a cosine. Positive scalar mixture
weights do not change that cosine. Each point compares the two objectives
\emph{within one fitted model}; connecting correction and no-correction
points compares their alignments, not gradients from different models.
The horizon-specific diagnostic slices the independent horizon heads and
counts negative cosines over correlated minibatch--head combinations.
Each plotted value averages two minibatches. Figure~\ref{fig:gradients}b
contains \MechMatchedGradientPairs{} pairs matched by environment, reward,
method, mixture weight, seed, and training budget. Pair eligibility does
not depend on cosine or success. Gradient diagnostics cover twelve
policy instances, including two outside this matched display;
the other sixteen instances provide subgoal-retention measurements only.

\paragraph{Exploratory association and fixed-query drift.}
Across the same \MechMatchedGradientPairs{} equal-budget pairs, the
Spearman correlation between correction retention and the no-correction
success advantage is \MechMatchedRetentionCorrelation{}.
This exploratory association describes the selected settings and does
not identify a causal effect. We also hold the
later model's probe states and three goal sets fixed and evaluate earlier
and later critics on those inputs. Table~\ref{tab:qdrift} reports absolute
changes. The statistic measures aggregate prediction drift from all learning
between the two training stages. It neither isolates low-level-policy changes nor measures
value-estimation or Bellman error.

\begin{table}[t]
\centering
\caption{\textbf{Value changes at fixed subgoal inputs.}
Mean absolute critic change across earlier and later models. R0/R1 denote
Reacher-v0/v1; d/s denote dense/sparse. Values use each run's reward scale.}
\label{tab:qdrift}
\small
\begin{tabular}{lrrlrrr}
\toprule
Setting & Seed & Correction & Steps (millions) & Original & Corrected & Current\\
\midrule
Rv0/s/Signed & 781 & yes & 0.95--1.00 & 0.316 & 0.315 & 0.322 \\
Rv0/s/Signed & 784 & yes & 0.70--0.75 & 0.333 & 0.327 & 0.336 \\
Rv0/s/Signed & 781 & no & 0.90--1.00 & 0.214 & 0.205 & 0.199 \\
Rv0/s/Signed & 784 & no & 0.90--1.00 & 0.950 & 0.945 & 0.989 \\
Rv1/d/HTB & 783 & yes & 0.95--1.00 & 3.366 & 3.324 & 3.394 \\
Rv1/s/HTB & 779 & yes & 0.95--1.00 & 2.667 & 0.789 & 1.080 \\
Rv1/s/HTB & 781 & yes & 0.95--1.00 & 1.133 & 0.976 & 0.988 \\
Rv1/s/Policy & 781 & yes & 0.95--1.00 & 0.558 & 0.558 & 0.548 \\
Rv1/d/HTB & 783 & no & 0.90--1.00 & 10.097 & 10.041 & 9.999 \\
Rv1/s/HTB & 779 & no & 0.90--1.00 & 0.309 & 0.308 & 0.321 \\
Rv1/s/HTB & 781 & no & 0.90--1.00 & 0.520 & 0.519 & 0.539 \\
Rv1/s/Policy & 781 & no & 0.90--1.00 & 0.525 & 0.521 & 0.536 \\
\bottomrule
\end{tabular}
\end{table}

\begin{table}[p]
\centering
\caption{\textbf{Complete Push/Fall correction ablation.}
Differences are no correction minus correction. Success differences are in
percentage points. Entries average the paired seed differences.
Repeated $\lambda=1$ rows preserve the full grid and share the same seed outcomes.}
\label{tab:full-pushfall}
\small
\setlength{\tabcolsep}{3pt}
\begin{tabular}{llrrr}
\toprule
Env. & Rule & $\lambda$ & $\Delta$ return & $\Delta$ success\\
\midrule
Fall & No next residual & 0 & -971.5 & +0.0 \\
Fall & No next residual & 0.6 & +5938.8 & +0.0 \\
Fall & No next residual & 0.9 & -767.3 & -9.2 \\
Fall & No next residual & 0.95 & +2850.9 & +14.4 \\
Fall & No next residual & 0.99 & +4225.6 & +13.2 \\
Fall & No next residual & 1 & -309.2 & -8.8 \\
Fall & HTB & 0 & +699.5 & +0.0 \\
Fall & HTB & 0.6 & +2522.4 & +0.0 \\
Fall & HTB & 0.9 & -2089.1 & -25.2 \\
Fall & HTB & 0.95 & -315.0 & -8.4 \\
Fall & HTB & 0.99 & +21.0 & -20.8 \\
Fall & HTB & 1 & -309.2 & -8.8 \\
Fall & Signed residual & 0 & +1037.8 & +0.0 \\
Fall & Signed residual & 0.6 & +5481.4 & +4.4 \\
Fall & Signed residual & 0.9 & -894.8 & -5.2 \\
Fall & Signed residual & 0.95 & +683.3 & -10.0 \\
Fall & Signed residual & 0.99 & +1586.2 & -2.0 \\
Fall & Signed residual & 1 & -309.2 & -8.8 \\
Push & No next residual & 0 & -1057.7 & -0.8 \\
Push & No next residual & 0.6 & +3458.8 & +0.0 \\
Push & No next residual & 0.9 & -115.0 & -9.6 \\
Push & No next residual & 0.95 & -1155.4 & -16.8 \\
Push & No next residual & 0.99 & -144.9 & +0.0 \\
Push & No next residual & 1 & +355.6 & -13.2 \\
Push & HTB & 0 & -113.6 & +0.0 \\
Push & HTB & 0.6 & +4349.2 & +0.0 \\
Push & HTB & 0.9 & -65.6 & +5.2 \\
Push & HTB & 0.95 & -432.7 & -13.2 \\
Push & HTB & 0.99 & +2.8 & +6.4 \\
Push & HTB & 1 & +355.6 & -13.2 \\
Push & Signed residual & 0 & +381.6 & +0.0 \\
Push & Signed residual & 0.6 & +3378.9 & +0.0 \\
Push & Signed residual & 0.9 & +1489.8 & -1.2 \\
Push & Signed residual & 0.95 & +401.3 & -2.8 \\
Push & Signed residual & 0.99 & -126.6 & -16.8 \\
Push & Signed residual & 1 & +355.6 & -13.2 \\
\bottomrule
\end{tabular}
\end{table}

Table~\ref{tab:full-pushfall} lists every logical pair in the comparison
grid, including the endpoint controls and comparisons favoring correction.

\paragraph{Current-policy queries and additional task coverage.}
Table~\ref{tab:mfour-examples} shows two current-policy-subgoal/$\lambda=.99$ configurations, where
the short branch queries the current high-level policy's goal. These selected
illustrations characterize the variant at two settings; its aggregate effect
uses the full grid. The full
Reacher aggregate includes all current-policy-subgoal interior weights. Under the same
endpoint/window rule, the fixed HTB/$\lambda=.95$ no-correction setting on
AntMaze-v1 achieves \MazeFixedDense{}\% dense and \MazeFixedSparse{}\% sparse
success, each over eight seeds. The Maze experiments compare algorithm
variants without correction; they do not constitute paired correction ablations.
The complete Maze/MazeW grid appears in the eight-task comparison, including
the MazeW settings with low success. Specified-subgoal supervision addresses
value updates; selecting useful subgoals and exploring remain complementary
requirements.

\begin{table}[t]
\centering
\caption{\textbf{Current-policy goal queries can also benefit from disabling
correction.} Eight-seed mean final-window success percentages for two current-policy-subgoal/.99
examples. These are illustrative configurations within the full grid.}
\label{tab:mfour-examples}
\begin{tabular}{lrrr}
\toprule
Task & Correction & No correction & Difference (pp)\\
\midrule
Reacher-v0 / dense & 60.00 & 88.13 & +28.12 \\
Reacher-v1 / sparse & 26.00 & 81.63 & +55.62 \\
\bottomrule
\end{tabular}
\end{table}

\paragraph{Horizon-loss scales.}
The loss-scale analysis covers \HorizonLogIdentities{} distinct run identities
with at least 99,000 high-level updates. It averages
losses within each run over early, middle, and late thirds of that run's update
range, then aggregates across seeds. The count measures the number of analyzed runs;
repeated endpoint configurations share the same underlying seed outcomes.
Figure~\ref{fig:horizon-loss}
shows the late-phase signed-residual/$\lambda=.6$ profiles: correction-enabled runs have much
larger long- and short-path critic losses in these settings. The number of seeds differs
between conditions. These profiles describe loss scale, while the cosine
diagnostic measures update direction and the rollout-return comparison
measures critic accuracy. Scalar loss magnitudes alone provide neither
of those measurements nor a causal account of task performance.

\begin{figure}[!htbp]
\centering
\includegraphics[width=\textwidth]{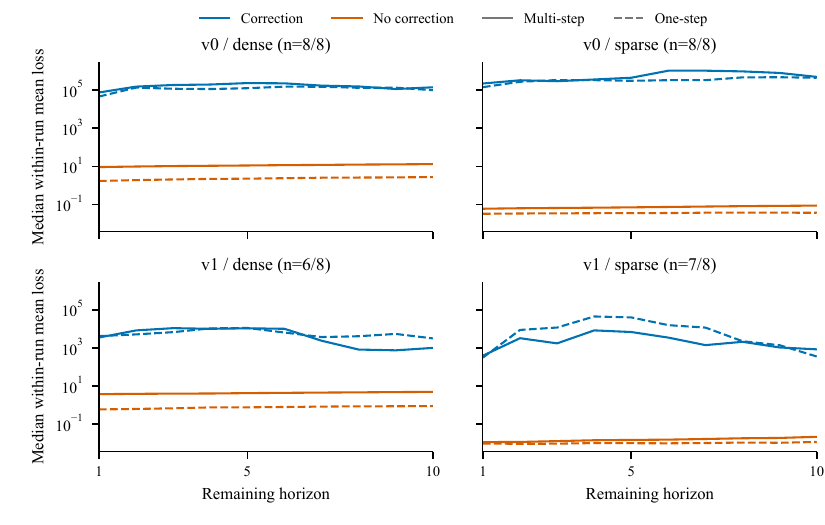}
\caption{\textbf{Late-phase signed-residual/.6 loss scales differ between correction modes.}
Each curve is the median across seed-level phase-mean losses; the vertical axis
is logarithmic. Panel counts give the number of correction/no-correction seeds.
The late phase is the final third of each run's high-level updates.
The profiles measure scalar loss magnitude; gradient alignment is evaluated
in Figure~\ref{fig:gradients}.}
\label{fig:horizon-loss}
\end{figure}
}{}
\section{Component controls, complete task coverage, and learning curves}
\label{app:vthree-evidence}

\paragraph{Task coverage and comparison units.}
The Maze comparison contains \VThreeMazeRows{} runs: two environments,
two rewards, four methods, six mixture weights, and eight seeds. Every curve
reaches exactly one million environment steps and uses the same success metric.
Table~\ref{tab:vthree-fixed} uses all eight seeds in each fixed setting.
The HIRO-without-correction controls comprise \VThreeLiteralRows{}
Reacher runs with the same seed identifiers and tail-window metric. They
use the distinct HIRO high-level policy implementation. The remaining-time
$\lambda=1$ endpoint is reported as a separate baseline. Matched budgets
and seeds support a comparison of complete algorithms, including their
implementation differences. Maze HIRO results are not included because
they do not meet the eight-seed comparison criterion.

\begin{table}[t]
\centering
\caption{\textbf{Fixed-configuration success across eight goal-reaching tasks.}
Success (\%), mean $\pm$ seed SD over eight seeds, with correction disabled.
Multi-step only is the remaining-time baseline at $\lambda=1$; HTB uses $\lambda=.95$.
HIRO uses the local high-level-policy implementation. Dashes mark
conditions without an eight-seed HIRO comparison, not zero success.}
\label{tab:vthree-fixed}
\small
\setlength{\tabcolsep}{5pt}
\begin{tabular}{llrrr}
\toprule
Environment & Reward & HIRO & Multi-step only & HTB/.95\\
\midrule
Reacher-v0 & dense & $\VThreeFixedAAAAMean{}\!\pm\!\VThreeFixedAAAASd{}$ & $\VThreeFixedAAABMean{}\!\pm\!\VThreeFixedAAABSd{}$ & $\VThreeFixedAAACMean{}\!\pm\!\VThreeFixedAAACSd{}$ \\
Reacher-v0 & sparse & $\VThreeFixedABAAMean{}\!\pm\!\VThreeFixedABAASd{}$ & $\VThreeFixedABABMean{}\!\pm\!\VThreeFixedABABSd{}$ & $\VThreeFixedABACMean{}\!\pm\!\VThreeFixedABACSd{}$ \\
Reacher-v1 & dense & $\VThreeFixedACAAMean{}\!\pm\!\VThreeFixedACAASd{}$ & $\VThreeFixedACABMean{}\!\pm\!\VThreeFixedACABSd{}$ & $\VThreeFixedACACMean{}\!\pm\!\VThreeFixedACACSd{}$ \\
Reacher-v1 & sparse & $\VThreeFixedADAAMean{}\!\pm\!\VThreeFixedADAASd{}$ & $\VThreeFixedADABMean{}\!\pm\!\VThreeFixedADABSd{}$ & $\VThreeFixedADACMean{}\!\pm\!\VThreeFixedADACSd{}$ \\
Maze-v1 & dense & --- & $\VThreeFixedAEABMean{}\!\pm\!\VThreeFixedAEABSd{}$ & $\VThreeFixedAEACMean{}\!\pm\!\VThreeFixedAEACSd{}$ \\
Maze-v1 & sparse & --- & $\VThreeFixedAFABMean{}\!\pm\!\VThreeFixedAFABSd{}$ & $\VThreeFixedAFACMean{}\!\pm\!\VThreeFixedAFACSd{}$ \\
MazeW-v2 & dense & --- & $\VThreeFixedAGABMean{}\!\pm\!\VThreeFixedAGABSd{}$ & $\VThreeFixedAGACMean{}\!\pm\!\VThreeFixedAGACSd{}$ \\
MazeW-v2 & sparse & --- & $\VThreeFixedAHABMean{}\!\pm\!\VThreeFixedAHABSd{}$ & $\VThreeFixedAHACMean{}\!\pm\!\VThreeFixedAHACSd{}$ \\
\bottomrule
\end{tabular}
\end{table}

\paragraph{Component averages and figure uncertainty.}
Table~\ref{tab:vthree-components-all} reports mean paired differences,
giving equal weight to each configuration and each seed.
For Figure~\ref{fig:vthree-components}, we first average paired differences
across the fixed evaluated configurations within each seed, then resample
the seed columns jointly. This preserves the
dependence created by reused seed identifiers and controls. A second interval
jointly resamples configuration rows and seed columns using 20,000 draws
with generator seed 20260923. This interval describes sensitivity to the
evaluated grid. Every configuration has equal
weight. The interior grid is $\lambda\in\{.6,.9,.95,.99\}$; the short-branch
contrast uses only the fixed $.95$ setting against $1$. Repeated $\lambda=1$
method aliases are excluded from component aggregates.

\begin{figure}[t]
\centering
\includegraphics[width=\textwidth]{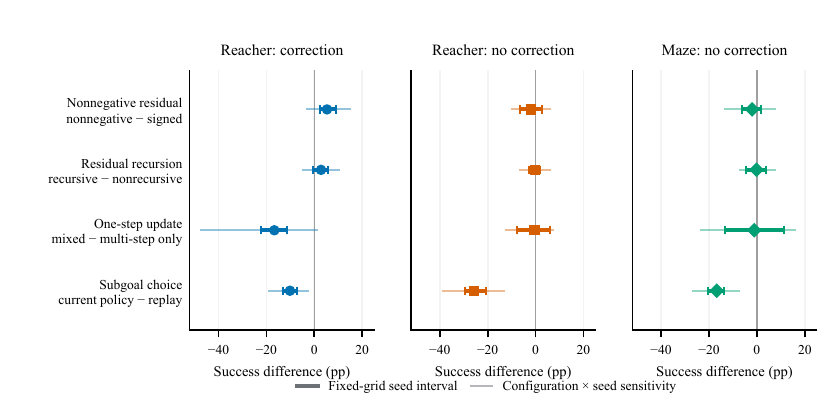}
\caption{\textbf{Component effects depend on the task and query distribution.}
Positive differences favor the left-hand condition. Thick bars resample
eight shared seed blocks on the fixed grid; thin bars also resample
configurations. Residual constraints, recursion, and query contrasts average all sixteen
task--weight settings per panel; the short-path contrast averages four
tasks. Nonnegative and query contrasts use the nonrecursive base.
Intervals are nominal $95\%$.}
\label{fig:vthree-components}
\end{figure}

\begin{table}[p]
\centering
\caption{\textbf{Interior-grid component averages.}
Differences are left minus right. Reacher/Maze use success percentage points;
Push/Fall use return. Correction is on/off. $K$ counts configurations.
Each entry is the mean paired difference across the evaluated configurations.
No family or metric is pooled with another.}
\label{tab:vthree-components-all}
\small
\setlength{\tabcolsep}{4pt}
\begin{tabular}{lllrr}
\toprule
Contrast & Tasks & Corr. & $K$ & Mean difference\\
\midrule
Recursive $-$ nonrecursive & Maze & off & \VThreeSummaryAAN{} & \VThreeSummaryAAMean{} \\
Recursive $-$ nonrecursive & Push/Fall & off & \VThreeSummaryABN{} & \VThreeSummaryABMean{} \\
Recursive $-$ nonrecursive & Push/Fall & on & \VThreeSummaryACN{} & \VThreeSummaryACMean{} \\
Recursive $-$ nonrecursive & Reacher & off & \VThreeSummaryADN{} & \VThreeSummaryADMean{} \\
Recursive $-$ nonrecursive & Reacher & on & \VThreeSummaryAEN{} & \VThreeSummaryAEMean{} \\
Nonnegative $-$ signed & Maze & off & \VThreeSummaryAFN{} & \VThreeSummaryAFMean{} \\
Nonnegative $-$ signed & Push/Fall & off & \VThreeSummaryAGN{} & \VThreeSummaryAGMean{} \\
Nonnegative $-$ signed & Push/Fall & on & \VThreeSummaryAHN{} & \VThreeSummaryAHMean{} \\
Nonnegative $-$ signed & Reacher & off & \VThreeSummaryAIN{} & \VThreeSummaryAIMean{} \\
Nonnegative $-$ signed & Reacher & on & \VThreeSummaryAJN{} & \VThreeSummaryAJMean{} \\
Policy $-$ replay & Maze & off & \VThreeSummaryAKN{} & \VThreeSummaryAKMean{} \\
Policy $-$ replay & Reacher & off & \VThreeSummaryALN{} & \VThreeSummaryALMean{} \\
Policy $-$ replay & Reacher & on & \VThreeSummaryAMN{} & \VThreeSummaryAMMean{} \\
Combined $-$ long-only & Maze & off & \VThreeSummaryANN{} & \VThreeSummaryANMean{} \\
Combined $-$ long-only & Push/Fall & off & \VThreeSummaryAON{} & \VThreeSummaryAOMean{} \\
Combined $-$ long-only & Push/Fall & on & \VThreeSummaryAPN{} & \VThreeSummaryAPMean{} \\
Combined $-$ long-only & Reacher & off & \VThreeSummaryAQN{} & \VThreeSummaryAQMean{} \\
Combined $-$ long-only & Reacher & on & \VThreeSummaryARN{} & \VThreeSummaryARMean{} \\
\bottomrule
\end{tabular}
\end{table}

\paragraph{What the component controls isolate.}
The No next residual versus Signed residual contrast changes the nonnegative
residual output and projected residual target together. It measures their
joint effect on task performance, with both favorable and unfavorable
cases included. Signed value-estimation error is measured separately by
the rollout-return diagnostic.
HTB versus No next residual changes the next-step residual in the implemented recursion.
Current-policy subgoal versus No next residual changes the query from replay to the current high-level policy.
All retain the same value/action-value residual factorization. The $.95$--$1$ contrast
keeps the remaining-time architecture but changes both objective weights.
The original grid therefore tests recursion, query selection, and the
backup mixture within a common factorized architecture.
Appendix~\ref{app:vseven-diagnostics} separately tests direct-$Q$ versus
shared-$V-A$ learning and the output-constraint-by-target-projection
factorial comparison under fixed policies.
Tables~\ref{tab:vthree-projection-reacher},
\ref{tab:vthree-projection-maze}, and~\ref{tab:vthree-projection-pushfall}
give every interior projection contrast; Tables~\ref{tab:vthree-nextgap-reacher},
\ref{tab:vthree-nextgap-maze}, and~\ref{tab:vthree-nextgap-pushfall}
give the recursion contrasts. Tables~\ref{tab:vthree-query-reacher}
and~\ref{tab:vthree-query-maze} give the complete query-distribution controls.

\paragraph{Fixed short-branch control.}
Table~\ref{tab:vthree-short} reports every task and mode at the fixed default.
Fall mean success increases under both modes.
Reacher effects vary across tasks, and some
correction-enabled tasks have lower success with mixed backups.
Eligibility is independent of the effect direction, so the complete table
locates both the benefit on Fall and the task dependence of the mixture.

\begin{table}[p]
\centering
\caption{\textbf{The short branch has task-dependent effects.}
Fixed HTB/$\lambda=.95$ minus Multi-step only at $\lambda=1$.
Success differences use percentage points; returns use the environment's
unscaled reward. All entries pair the same seeds within a task and mode.}
\label{tab:vthree-short}
\small
\setlength{\tabcolsep}{4pt}
\begin{tabular}{llllr}
\toprule
Environment & Reward & Correction & Metric & Mean difference\\
\midrule
Maze-v1 & dense & off & Success & \VThreeShortAAMean{} \\
Maze-v1 & sparse & off & Success & \VThreeShortABMean{} \\
MazeW-v2 & dense & off & Success & \VThreeShortACMean{} \\
MazeW-v2 & sparse & off & Success & \VThreeShortADMean{} \\
Fall & dense & off & Return & \VThreeShortAEMean{} \\
Fall & dense & off & Success & \VThreeShortAFMean{} \\
Fall & dense & on & Return & \VThreeShortAGMean{} \\
Fall & dense & on & Success & \VThreeShortAHMean{} \\
Push & dense & off & Return & \VThreeShortAIMean{} \\
Push & dense & off & Success & \VThreeShortAJMean{} \\
Push & dense & on & Return & \VThreeShortAKMean{} \\
Push & dense & on & Success & \VThreeShortALMean{} \\
Reacher-v0 & dense & off & Success & \VThreeShortAMMean{} \\
Reacher-v0 & dense & on & Success & \VThreeShortANMean{} \\
Reacher-v0 & sparse & off & Success & \VThreeShortAOMean{} \\
Reacher-v0 & sparse & on & Success & \VThreeShortAPMean{} \\
Reacher-v1 & dense & off & Success & \VThreeShortAQMean{} \\
Reacher-v1 & dense & on & Success & \VThreeShortARMean{} \\
Reacher-v1 & sparse & off & Success & \VThreeShortASMean{} \\
Reacher-v1 & sparse & on & Success & \VThreeShortATMean{} \\
\bottomrule
\end{tabular}
\end{table}

\par\noindent\begin{minipage}{\textwidth}
\paragraph{Learning curves and integrated return.}
The fixed-setting success curves in Figure~\ref{fig:vthree-learning}
and return curves in Figure~\ref{fig:vfour-return} show 200 equally spaced
evaluations from 50,000 to ten million environment steps for each of
the thirty runs.
The same five seeds are used for each arm. Bands enumerate the empirical
bootstrap of seed curves at each time point, giving pointwise coverage;
simultaneous coverage is not assessed. All evaluation points are shown without a moving average.
Table~\ref{tab:vthree-auc} compares the time-normalized integral over
$[50{,}000,10{,}000{,}000]$ with the final return. This distinguishes
sample-efficiency behavior from the endpoint comparison.
\end{minipage}\par

\begin{figure}[t]
\centering
\includegraphics[width=\textwidth]{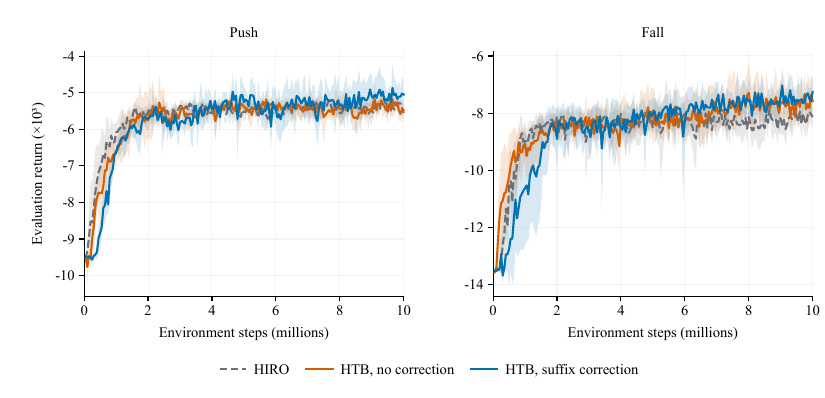}
\caption{\textbf{Final return and early learning can favor different settings.}
Actual returns at all 200 evaluations from $0.05$M to $10$M steps.
Both HTB conditions use $\lambda=.95$. Lines average five matched seeds;
bands are pointwise $95\%$ empirical bootstrap intervals.
No temporal smoothing is applied.}
\label{fig:vfour-return}
\end{figure}

\begin{table}[t]
\centering
\caption{\textbf{Integrated return and final return at the fixed configuration.}
HTB/.95 minus HIRO, five paired seeds. AUC is the trapezoidal return integral
divided by the duration of its observed step window. Each cell gives the
mean paired difference.}
\label{tab:vthree-auc}
\small
\setlength{\tabcolsep}{4pt}
\begin{tabular}{llrr}
\toprule
Task & HTB correction & Normalized AUC difference & Final-return difference\\
\midrule
Push & on & \VThreeAucAAAAMean{} & \VThreeAucAAABMean{} \\
Push & off & \VThreeAucABAAMean{} & \VThreeAucABABMean{} \\
Fall & on & \VThreeAucACAAMean{} & \VThreeAucACABMean{} \\
Fall & off & \VThreeAucADAAMean{} & \VThreeAucADABMean{} \\
\bottomrule
\end{tabular}
\end{table}

\begin{table}[t]
\centering
\caption{\textbf{Fixed HTB settings against the Push/Fall HIRO reference.}
Five-seed means at ten million steps; both HTB modes use $\lambda=.95$.
The HIRO rows use paper-style correction.}
\label{tab:fixed-reference}
\small
\begin{tabular}{llrr}
\toprule
Environment & Setting & Return & Success (\%)\\
\midrule
Push & HIRO & -5553.6 & 0.0 \\
Push & HTB, suffix correction & -5045.3 & 14.0 \\
Push & HTB, no correction & -5478.1 & 0.8 \\
Fall & HIRO & -8137.8 & 9.6 \\
Fall & HTB, suffix correction & -7262.4 & 32.8 \\
Fall & HTB, no correction & -7577.4 & 24.4 \\
\bottomrule
\end{tabular}
\end{table}

\paragraph{Parameter sharing across remaining-time horizons.}
Table~\ref{tab:vthree-sharing} compares \VThreeSharingRows{} runs
of the multi-step-only baseline. Independent horizon actors/critics,
a shared time-conditioned critic, and shared time-conditioned actor/critic
networks all retain suffix expansion and correction. The ten-seed comparison
changes parameter sharing across horizons, whereas the $V-A$ comparison
tests sharing across actions at a fixed remaining time. The eight overlapping
independent-network Reacher seed values are the same observations as the
Multi-step only reference. All curves reach one million environment steps.

\begin{table}[t]
\centering
\caption{\textbf{Horizon-sharing effects vary across goal-reaching tasks.}
Multi-step only algorithms with correction, final-window success (\%), mean $\pm$ SD
over ten seeds. The architecture intervention differs from separating $V$
and $A$ within an action-conditioned value.}
\label{tab:vthree-sharing}
\small
\setlength{\tabcolsep}{4pt}
\begin{tabular}{llrrr}
\toprule
Task & Reward & Independent & Shared critic & Shared actor/critic\\
\midrule
Reacher-v0 & dense & $\VThreeSharingAAAAMean{}\!\pm\!\VThreeSharingAAAASd{}$ & $\VThreeSharingAAABMean{}\!\pm\!\VThreeSharingAAABSd{}$ & $\VThreeSharingAAACMean{}\!\pm\!\VThreeSharingAAACSd{}$ \\
Reacher-v0 & sparse & $\VThreeSharingABAAMean{}\!\pm\!\VThreeSharingABAASd{}$ & $\VThreeSharingABABMean{}\!\pm\!\VThreeSharingABABSd{}$ & $\VThreeSharingABACMean{}\!\pm\!\VThreeSharingABACSd{}$ \\
Maze-v1 & dense & $\VThreeSharingACAAMean{}\!\pm\!\VThreeSharingACAASd{}$ & $\VThreeSharingACABMean{}\!\pm\!\VThreeSharingACABSd{}$ & $\VThreeSharingACACMean{}\!\pm\!\VThreeSharingACACSd{}$ \\
Maze-v1 & sparse & $\VThreeSharingADAAMean{}\!\pm\!\VThreeSharingADAASd{}$ & $\VThreeSharingADABMean{}\!\pm\!\VThreeSharingADABSd{}$ & $\VThreeSharingADACMean{}\!\pm\!\VThreeSharingADACSd{}$ \\
MazeW-v2 & dense & $\VThreeSharingAEAAMean{}\!\pm\!\VThreeSharingAEAASd{}$ & $\VThreeSharingAEABMean{}\!\pm\!\VThreeSharingAEABSd{}$ & $\VThreeSharingAEACMean{}\!\pm\!\VThreeSharingAEACSd{}$ \\
MazeW-v2 & sparse & $\VThreeSharingAFAAMean{}\!\pm\!\VThreeSharingAFAASd{}$ & $\VThreeSharingAFABMean{}\!\pm\!\VThreeSharingAFABSd{}$ & $\VThreeSharingAFACMean{}\!\pm\!\VThreeSharingAFACSd{}$ \\
\bottomrule
\end{tabular}
\end{table}

\paragraph{Full correction sweeps.}
Figure~\ref{fig:pushfall} displays every Push/Fall correction contrast and
Figure~\ref{fig:reacher} gives Reacher strata. Both effect directions are
represented throughout. The complete endpoint
table and Reacher endpoint sensitivities are retained in
Appendix~\ref{app:empirical}.
For full HTB on Push, the paired change in the no-correction-minus-correction
return effect between $\lambda=.6$ and $.9$ is
\VFourFullPushInteractionMean{}
[\VFourFullPushInteractionLo{}, \VFourFullPushInteractionHi{}].
Each interaction pairs the same five seeds across all four conditions.
The corresponding Fall interaction is \VFourFullFallInteractionMean{}
[\VFourFullFallInteractionLo{}, \VFourFullFallInteractionHi{}].

\begin{figure}[p]
\centering
\includegraphics[width=\textwidth]{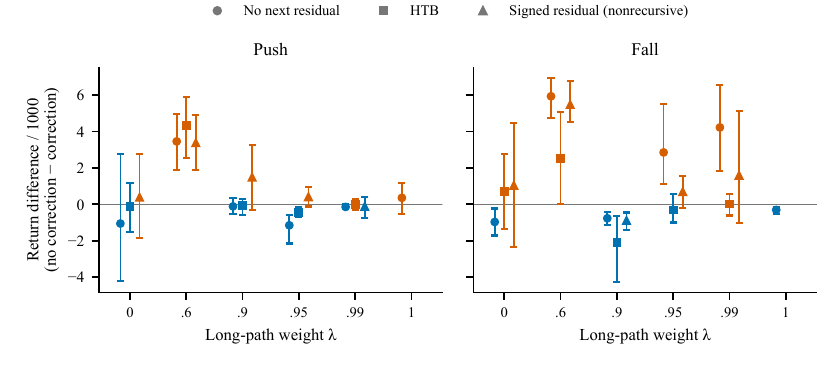}
\caption{\textbf{The benefit of correction reverses with the mixture.}
Positive final-return differences favor disabling correction.
Points average \PFSeeds{} paired seeds; bars enumerate
\PFExactResamples{} empirical bootstrap resamples. Symbols identify
methods; identical $\lambda=1$ results are shown once per environment.}
\label{fig:pushfall}
\end{figure}

\begin{figure}[p]
\centering
\includegraphics[width=\textwidth]{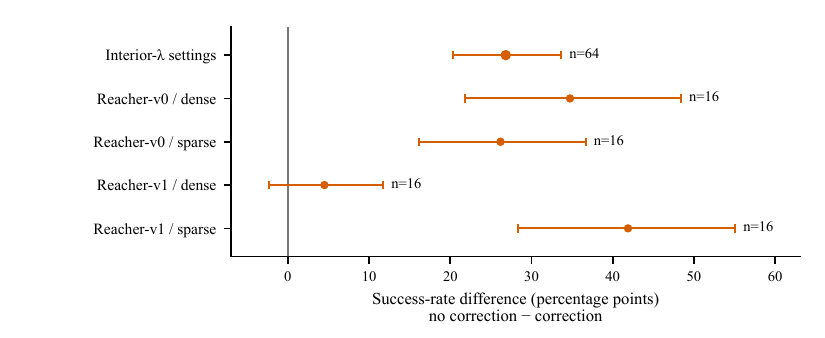}
\caption{\textbf{The correction effect depends on the task.}
Mean final-window success differences on the interior-$\lambda$ grid.
Intervals resample configurations and their \ReacherSeeds{} paired seeds;
$n$ counts configurations. Each configuration has equal weight. This
correction-effect bootstrap differs from the shared-seed-block
component sensitivity defined in this appendix.}
\label{fig:reacher}
\end{figure}

\begin{table}[p]
\centering
\caption{\textbf{Complete Reacher nonnegative-residual contrasts.}
No next residual minus Signed residual; success difference in percentage points,
eight paired seeds. Both conditions use the nonrecursive base.
Correction is on/off. All interior weights and both effect directions appear.}
\label{tab:vthree-projection-reacher}
\small
\begin{tabular}{llllr}
\toprule
Task & Reward & Corr. & $\lambda$ & Mean difference\\
\midrule
Reacher-v0 & dense & off & $0.6$ & \VThreeFullProjectionReacherAAMean{} \\
Reacher-v0 & dense & off & $0.9$ & \VThreeFullProjectionReacherABMean{} \\
Reacher-v0 & dense & off & $0.95$ & \VThreeFullProjectionReacherACMean{} \\
Reacher-v0 & dense & off & $0.99$ & \VThreeFullProjectionReacherADMean{} \\
Reacher-v0 & dense & on & $0.6$ & \VThreeFullProjectionReacherAEMean{} \\
Reacher-v0 & dense & on & $0.9$ & \VThreeFullProjectionReacherAFMean{} \\
Reacher-v0 & dense & on & $0.95$ & \VThreeFullProjectionReacherAGMean{} \\
Reacher-v0 & dense & on & $0.99$ & \VThreeFullProjectionReacherAHMean{} \\
Reacher-v0 & sparse & off & $0.6$ & \VThreeFullProjectionReacherAIMean{} \\
Reacher-v0 & sparse & off & $0.9$ & \VThreeFullProjectionReacherAJMean{} \\
Reacher-v0 & sparse & off & $0.95$ & \VThreeFullProjectionReacherAKMean{} \\
Reacher-v0 & sparse & off & $0.99$ & \VThreeFullProjectionReacherALMean{} \\
Reacher-v0 & sparse & on & $0.6$ & \VThreeFullProjectionReacherAMMean{} \\
Reacher-v0 & sparse & on & $0.9$ & \VThreeFullProjectionReacherANMean{} \\
Reacher-v0 & sparse & on & $0.95$ & \VThreeFullProjectionReacherAOMean{} \\
Reacher-v0 & sparse & on & $0.99$ & \VThreeFullProjectionReacherAPMean{} \\
Reacher-v1 & dense & off & $0.6$ & \VThreeFullProjectionReacherAQMean{} \\
Reacher-v1 & dense & off & $0.9$ & \VThreeFullProjectionReacherARMean{} \\
Reacher-v1 & dense & off & $0.95$ & \VThreeFullProjectionReacherASMean{} \\
Reacher-v1 & dense & off & $0.99$ & \VThreeFullProjectionReacherATMean{} \\
Reacher-v1 & dense & on & $0.6$ & \VThreeFullProjectionReacherAUMean{} \\
Reacher-v1 & dense & on & $0.9$ & \VThreeFullProjectionReacherAVMean{} \\
Reacher-v1 & dense & on & $0.95$ & \VThreeFullProjectionReacherAWMean{} \\
Reacher-v1 & dense & on & $0.99$ & \VThreeFullProjectionReacherAXMean{} \\
Reacher-v1 & sparse & off & $0.6$ & \VThreeFullProjectionReacherAYMean{} \\
Reacher-v1 & sparse & off & $0.9$ & \VThreeFullProjectionReacherAZMean{} \\
Reacher-v1 & sparse & off & $0.95$ & \VThreeFullProjectionReacherBAMean{} \\
Reacher-v1 & sparse & off & $0.99$ & \VThreeFullProjectionReacherBBMean{} \\
Reacher-v1 & sparse & on & $0.6$ & \VThreeFullProjectionReacherBCMean{} \\
Reacher-v1 & sparse & on & $0.9$ & \VThreeFullProjectionReacherBDMean{} \\
Reacher-v1 & sparse & on & $0.95$ & \VThreeFullProjectionReacherBEMean{} \\
Reacher-v1 & sparse & on & $0.99$ & \VThreeFullProjectionReacherBFMean{} \\
\bottomrule
\end{tabular}
\end{table}

\begin{table}[p]
\centering
\caption{\textbf{Complete Maze nonnegative-residual contrasts.}
No next residual minus Signed residual on the nonrecursive base; success difference
in percentage points, eight paired seeds.}
\label{tab:vthree-projection-maze}
\small
\begin{tabular}{llllr}
\toprule
Task & Reward & Corr. & $\lambda$ & Mean difference\\
\midrule
Maze-v1 & dense & off & $0.6$ & \VThreeFullProjectionMazeavailableAAMean{} \\
Maze-v1 & dense & off & $0.9$ & \VThreeFullProjectionMazeavailableABMean{} \\
Maze-v1 & dense & off & $0.95$ & \VThreeFullProjectionMazeavailableACMean{} \\
Maze-v1 & dense & off & $0.99$ & \VThreeFullProjectionMazeavailableADMean{} \\
Maze-v1 & sparse & off & $0.6$ & \VThreeFullProjectionMazeavailableAEMean{} \\
Maze-v1 & sparse & off & $0.9$ & \VThreeFullProjectionMazeavailableAFMean{} \\
Maze-v1 & sparse & off & $0.95$ & \VThreeFullProjectionMazeavailableAGMean{} \\
Maze-v1 & sparse & off & $0.99$ & \VThreeFullProjectionMazeavailableAHMean{} \\
MazeW-v2 & dense & off & $0.6$ & \VThreeFullProjectionMazeavailableAIMean{} \\
MazeW-v2 & dense & off & $0.9$ & \VThreeFullProjectionMazeavailableAJMean{} \\
MazeW-v2 & dense & off & $0.95$ & \VThreeFullProjectionMazeavailableAKMean{} \\
MazeW-v2 & dense & off & $0.99$ & \VThreeFullProjectionMazeavailableALMean{} \\
MazeW-v2 & sparse & off & $0.6$ & \VThreeFullProjectionMazeavailableAMMean{} \\
MazeW-v2 & sparse & off & $0.9$ & \VThreeFullProjectionMazeavailableANMean{} \\
MazeW-v2 & sparse & off & $0.95$ & \VThreeFullProjectionMazeavailableAOMean{} \\
MazeW-v2 & sparse & off & $0.99$ & \VThreeFullProjectionMazeavailableAPMean{} \\
\bottomrule
\end{tabular}
\end{table}

\begin{table}[p]
\centering
\caption{\textbf{Complete Push/Fall nonnegative-residual contrasts.}
No next residual minus Signed residual on the nonrecursive base;
final-return difference, five paired seeds.}
\label{tab:vthree-projection-pushfall}
\small
\begin{tabular}{llllr}
\toprule
Task & Reward & Corr. & $\lambda$ & Mean difference\\
\midrule
Fall & dense & off & $0.6$ & \VThreeFullProjectionPushfallAAMean{} \\
Fall & dense & off & $0.9$ & \VThreeFullProjectionPushfallABMean{} \\
Fall & dense & off & $0.95$ & \VThreeFullProjectionPushfallACMean{} \\
Fall & dense & off & $0.99$ & \VThreeFullProjectionPushfallADMean{} \\
Fall & dense & on & $0.6$ & \VThreeFullProjectionPushfallAEMean{} \\
Fall & dense & on & $0.9$ & \VThreeFullProjectionPushfallAFMean{} \\
Fall & dense & on & $0.95$ & \VThreeFullProjectionPushfallAGMean{} \\
Fall & dense & on & $0.99$ & \VThreeFullProjectionPushfallAHMean{} \\
Push & dense & off & $0.6$ & \VThreeFullProjectionPushfallAIMean{} \\
Push & dense & off & $0.9$ & \VThreeFullProjectionPushfallAJMean{} \\
Push & dense & off & $0.95$ & \VThreeFullProjectionPushfallAKMean{} \\
Push & dense & off & $0.99$ & \VThreeFullProjectionPushfallALMean{} \\
Push & dense & on & $0.6$ & \VThreeFullProjectionPushfallAMMean{} \\
Push & dense & on & $0.9$ & \VThreeFullProjectionPushfallANMean{} \\
Push & dense & on & $0.95$ & \VThreeFullProjectionPushfallAOMean{} \\
Push & dense & on & $0.99$ & \VThreeFullProjectionPushfallAPMean{} \\
\bottomrule
\end{tabular}
\end{table}

\begin{table}[p]
\centering
\caption{\textbf{Complete Reacher residual-recursion contrasts.}
HTB minus No next residual; success difference in percentage points,
eight paired seeds.}
\label{tab:vthree-nextgap-reacher}
\small
\begin{tabular}{llllr}
\toprule
Task & Reward & Corr. & $\lambda$ & Mean difference\\
\midrule
Reacher-v0 & dense & off & $0.6$ & \VThreeFullNextgapReacherAAMean{} \\
Reacher-v0 & dense & off & $0.9$ & \VThreeFullNextgapReacherABMean{} \\
Reacher-v0 & dense & off & $0.95$ & \VThreeFullNextgapReacherACMean{} \\
Reacher-v0 & dense & off & $0.99$ & \VThreeFullNextgapReacherADMean{} \\
Reacher-v0 & dense & on & $0.6$ & \VThreeFullNextgapReacherAEMean{} \\
Reacher-v0 & dense & on & $0.9$ & \VThreeFullNextgapReacherAFMean{} \\
Reacher-v0 & dense & on & $0.95$ & \VThreeFullNextgapReacherAGMean{} \\
Reacher-v0 & dense & on & $0.99$ & \VThreeFullNextgapReacherAHMean{} \\
Reacher-v0 & sparse & off & $0.6$ & \VThreeFullNextgapReacherAIMean{} \\
Reacher-v0 & sparse & off & $0.9$ & \VThreeFullNextgapReacherAJMean{} \\
Reacher-v0 & sparse & off & $0.95$ & \VThreeFullNextgapReacherAKMean{} \\
Reacher-v0 & sparse & off & $0.99$ & \VThreeFullNextgapReacherALMean{} \\
Reacher-v0 & sparse & on & $0.6$ & \VThreeFullNextgapReacherAMMean{} \\
Reacher-v0 & sparse & on & $0.9$ & \VThreeFullNextgapReacherANMean{} \\
Reacher-v0 & sparse & on & $0.95$ & \VThreeFullNextgapReacherAOMean{} \\
Reacher-v0 & sparse & on & $0.99$ & \VThreeFullNextgapReacherAPMean{} \\
Reacher-v1 & dense & off & $0.6$ & \VThreeFullNextgapReacherAQMean{} \\
Reacher-v1 & dense & off & $0.9$ & \VThreeFullNextgapReacherARMean{} \\
Reacher-v1 & dense & off & $0.95$ & \VThreeFullNextgapReacherASMean{} \\
Reacher-v1 & dense & off & $0.99$ & \VThreeFullNextgapReacherATMean{} \\
Reacher-v1 & dense & on & $0.6$ & \VThreeFullNextgapReacherAUMean{} \\
Reacher-v1 & dense & on & $0.9$ & \VThreeFullNextgapReacherAVMean{} \\
Reacher-v1 & dense & on & $0.95$ & \VThreeFullNextgapReacherAWMean{} \\
Reacher-v1 & dense & on & $0.99$ & \VThreeFullNextgapReacherAXMean{} \\
Reacher-v1 & sparse & off & $0.6$ & \VThreeFullNextgapReacherAYMean{} \\
Reacher-v1 & sparse & off & $0.9$ & \VThreeFullNextgapReacherAZMean{} \\
Reacher-v1 & sparse & off & $0.95$ & \VThreeFullNextgapReacherBAMean{} \\
Reacher-v1 & sparse & off & $0.99$ & \VThreeFullNextgapReacherBBMean{} \\
Reacher-v1 & sparse & on & $0.6$ & \VThreeFullNextgapReacherBCMean{} \\
Reacher-v1 & sparse & on & $0.9$ & \VThreeFullNextgapReacherBDMean{} \\
Reacher-v1 & sparse & on & $0.95$ & \VThreeFullNextgapReacherBEMean{} \\
Reacher-v1 & sparse & on & $0.99$ & \VThreeFullNextgapReacherBFMean{} \\
\bottomrule
\end{tabular}
\end{table}

\begin{table}[p]
\centering
\caption{\textbf{Complete Maze residual-recursion contrasts.}
HTB minus No next residual; success difference in percentage points,
eight paired seeds.}
\label{tab:vthree-nextgap-maze}
\small
\begin{tabular}{llllr}
\toprule
Task & Reward & Corr. & $\lambda$ & Mean difference\\
\midrule
Maze-v1 & dense & off & $0.6$ & \VThreeFullNextgapMazeavailableAAMean{} \\
Maze-v1 & dense & off & $0.9$ & \VThreeFullNextgapMazeavailableABMean{} \\
Maze-v1 & dense & off & $0.95$ & \VThreeFullNextgapMazeavailableACMean{} \\
Maze-v1 & dense & off & $0.99$ & \VThreeFullNextgapMazeavailableADMean{} \\
Maze-v1 & sparse & off & $0.6$ & \VThreeFullNextgapMazeavailableAEMean{} \\
Maze-v1 & sparse & off & $0.9$ & \VThreeFullNextgapMazeavailableAFMean{} \\
Maze-v1 & sparse & off & $0.95$ & \VThreeFullNextgapMazeavailableAGMean{} \\
Maze-v1 & sparse & off & $0.99$ & \VThreeFullNextgapMazeavailableAHMean{} \\
MazeW-v2 & dense & off & $0.6$ & \VThreeFullNextgapMazeavailableAIMean{} \\
MazeW-v2 & dense & off & $0.9$ & \VThreeFullNextgapMazeavailableAJMean{} \\
MazeW-v2 & dense & off & $0.95$ & \VThreeFullNextgapMazeavailableAKMean{} \\
MazeW-v2 & dense & off & $0.99$ & \VThreeFullNextgapMazeavailableALMean{} \\
MazeW-v2 & sparse & off & $0.6$ & \VThreeFullNextgapMazeavailableAMMean{} \\
MazeW-v2 & sparse & off & $0.9$ & \VThreeFullNextgapMazeavailableANMean{} \\
MazeW-v2 & sparse & off & $0.95$ & \VThreeFullNextgapMazeavailableAOMean{} \\
MazeW-v2 & sparse & off & $0.99$ & \VThreeFullNextgapMazeavailableAPMean{} \\
\bottomrule
\end{tabular}
\end{table}

\begin{table}[p]
\centering
\caption{\textbf{Complete Push/Fall residual-recursion contrasts.}
HTB minus No next residual; final-return difference, five paired seeds.}
\label{tab:vthree-nextgap-pushfall}
\small
\begin{tabular}{llllr}
\toprule
Task & Reward & Corr. & $\lambda$ & Mean difference\\
\midrule
Fall & dense & off & $0.6$ & \VThreeFullNextgapPushfallAAMean{} \\
Fall & dense & off & $0.9$ & \VThreeFullNextgapPushfallABMean{} \\
Fall & dense & off & $0.95$ & \VThreeFullNextgapPushfallACMean{} \\
Fall & dense & off & $0.99$ & \VThreeFullNextgapPushfallADMean{} \\
Fall & dense & on & $0.6$ & \VThreeFullNextgapPushfallAEMean{} \\
Fall & dense & on & $0.9$ & \VThreeFullNextgapPushfallAFMean{} \\
Fall & dense & on & $0.95$ & \VThreeFullNextgapPushfallAGMean{} \\
Fall & dense & on & $0.99$ & \VThreeFullNextgapPushfallAHMean{} \\
Push & dense & off & $0.6$ & \VThreeFullNextgapPushfallAIMean{} \\
Push & dense & off & $0.9$ & \VThreeFullNextgapPushfallAJMean{} \\
Push & dense & off & $0.95$ & \VThreeFullNextgapPushfallAKMean{} \\
Push & dense & off & $0.99$ & \VThreeFullNextgapPushfallALMean{} \\
Push & dense & on & $0.6$ & \VThreeFullNextgapPushfallAMMean{} \\
Push & dense & on & $0.9$ & \VThreeFullNextgapPushfallANMean{} \\
Push & dense & on & $0.95$ & \VThreeFullNextgapPushfallAOMean{} \\
Push & dense & on & $0.99$ & \VThreeFullNextgapPushfallAPMean{} \\
\bottomrule
\end{tabular}
\end{table}

\begin{table}[p]
\centering
\caption{\textbf{Complete Reacher query-distribution contrasts.}
Current-policy subgoal minus No next residual on the nonrecursive base;
success difference in percentage points, eight paired seeds.
Negative values favor replay goals over current high-level policy queries.}
\label{tab:vthree-query-reacher}
\small
\begin{tabular}{llllr}
\toprule
Task & Reward & Corr. & $\lambda$ & Mean difference\\
\midrule
Reacher-v0 & dense & off & $0.6$ & \VThreeFullQueryReacherAAMean{} \\
Reacher-v0 & dense & off & $0.9$ & \VThreeFullQueryReacherABMean{} \\
Reacher-v0 & dense & off & $0.95$ & \VThreeFullQueryReacherACMean{} \\
Reacher-v0 & dense & off & $0.99$ & \VThreeFullQueryReacherADMean{} \\
Reacher-v0 & dense & on & $0.6$ & \VThreeFullQueryReacherAEMean{} \\
Reacher-v0 & dense & on & $0.9$ & \VThreeFullQueryReacherAFMean{} \\
Reacher-v0 & dense & on & $0.95$ & \VThreeFullQueryReacherAGMean{} \\
Reacher-v0 & dense & on & $0.99$ & \VThreeFullQueryReacherAHMean{} \\
Reacher-v0 & sparse & off & $0.6$ & \VThreeFullQueryReacherAIMean{} \\
Reacher-v0 & sparse & off & $0.9$ & \VThreeFullQueryReacherAJMean{} \\
Reacher-v0 & sparse & off & $0.95$ & \VThreeFullQueryReacherAKMean{} \\
Reacher-v0 & sparse & off & $0.99$ & \VThreeFullQueryReacherALMean{} \\
Reacher-v0 & sparse & on & $0.6$ & \VThreeFullQueryReacherAMMean{} \\
Reacher-v0 & sparse & on & $0.9$ & \VThreeFullQueryReacherANMean{} \\
Reacher-v0 & sparse & on & $0.95$ & \VThreeFullQueryReacherAOMean{} \\
Reacher-v0 & sparse & on & $0.99$ & \VThreeFullQueryReacherAPMean{} \\
Reacher-v1 & dense & off & $0.6$ & \VThreeFullQueryReacherAQMean{} \\
Reacher-v1 & dense & off & $0.9$ & \VThreeFullQueryReacherARMean{} \\
Reacher-v1 & dense & off & $0.95$ & \VThreeFullQueryReacherASMean{} \\
Reacher-v1 & dense & off & $0.99$ & \VThreeFullQueryReacherATMean{} \\
Reacher-v1 & dense & on & $0.6$ & \VThreeFullQueryReacherAUMean{} \\
Reacher-v1 & dense & on & $0.9$ & \VThreeFullQueryReacherAVMean{} \\
Reacher-v1 & dense & on & $0.95$ & \VThreeFullQueryReacherAWMean{} \\
Reacher-v1 & dense & on & $0.99$ & \VThreeFullQueryReacherAXMean{} \\
Reacher-v1 & sparse & off & $0.6$ & \VThreeFullQueryReacherAYMean{} \\
Reacher-v1 & sparse & off & $0.9$ & \VThreeFullQueryReacherAZMean{} \\
Reacher-v1 & sparse & off & $0.95$ & \VThreeFullQueryReacherBAMean{} \\
Reacher-v1 & sparse & off & $0.99$ & \VThreeFullQueryReacherBBMean{} \\
Reacher-v1 & sparse & on & $0.6$ & \VThreeFullQueryReacherBCMean{} \\
Reacher-v1 & sparse & on & $0.9$ & \VThreeFullQueryReacherBDMean{} \\
Reacher-v1 & sparse & on & $0.95$ & \VThreeFullQueryReacherBEMean{} \\
Reacher-v1 & sparse & on & $0.99$ & \VThreeFullQueryReacherBFMean{} \\
\bottomrule
\end{tabular}
\end{table}

\begin{table}[p]
\centering
\caption{\textbf{Complete Maze query-distribution contrasts.}
Current-policy subgoal minus No next residual on the nonrecursive base;
success difference in percentage points, eight paired seeds.}
\label{tab:vthree-query-maze}
\small
\begin{tabular}{llllr}
\toprule
Task & Reward & Corr. & $\lambda$ & Mean difference\\
\midrule
Maze-v1 & dense & off & $0.6$ & \VThreeFullQueryMazeavailableAAMean{} \\
Maze-v1 & dense & off & $0.9$ & \VThreeFullQueryMazeavailableABMean{} \\
Maze-v1 & dense & off & $0.95$ & \VThreeFullQueryMazeavailableACMean{} \\
Maze-v1 & dense & off & $0.99$ & \VThreeFullQueryMazeavailableADMean{} \\
Maze-v1 & sparse & off & $0.6$ & \VThreeFullQueryMazeavailableAEMean{} \\
Maze-v1 & sparse & off & $0.9$ & \VThreeFullQueryMazeavailableAFMean{} \\
Maze-v1 & sparse & off & $0.95$ & \VThreeFullQueryMazeavailableAGMean{} \\
Maze-v1 & sparse & off & $0.99$ & \VThreeFullQueryMazeavailableAHMean{} \\
MazeW-v2 & dense & off & $0.6$ & \VThreeFullQueryMazeavailableAIMean{} \\
MazeW-v2 & dense & off & $0.9$ & \VThreeFullQueryMazeavailableAJMean{} \\
MazeW-v2 & dense & off & $0.95$ & \VThreeFullQueryMazeavailableAKMean{} \\
MazeW-v2 & dense & off & $0.99$ & \VThreeFullQueryMazeavailableALMean{} \\
MazeW-v2 & sparse & off & $0.6$ & \VThreeFullQueryMazeavailableAMMean{} \\
MazeW-v2 & sparse & off & $0.9$ & \VThreeFullQueryMazeavailableANMean{} \\
MazeW-v2 & sparse & off & $0.95$ & \VThreeFullQueryMazeavailableAOMean{} \\
MazeW-v2 & sparse & off & $0.99$ & \VThreeFullQueryMazeavailableAPMean{} \\
\bottomrule
\end{tabular}
\end{table}

\clearpage
\section{Supplementary value-estimation diagnostics}
\label{app:vseven-diagnostics}
These diagnostics isolate the estimation mechanisms behind the design in
Section~\ref{sec:experiments}. Matched-supervision critic fitting tests
whether a shared value improves generalization across primitive actions. A four-arm
comparison separates residual-output and target constraints, and a finite
MDP measures their interaction with maximization feedback. The end-to-end
task-success experiments remain the performance evaluation.

\subsection{Frozen-policy rollout returns}
We use \VSevenPolicySeeds{} trained policies per task from
Reacher-v0 dense and Reacher-v1 sparse, giving \VSevenUnits{} policy units.
The source dense policies use the signed-residual, nonrecursive variant
with relabeling at $\lambda=.9$; the sparse policies use full HTB without
relabeling at $\lambda=.95$. All policies were trained for one
million environment steps. The high-level and low-level actors remain
fixed throughout the diagnostic.

Each unit contains \VSevenEpisodes{} episodes, divided by complete episode
into \VSevenTrainEpisodes{} training, \VSevenCalEpisodes{} calibration,
and \VSevenTestEpisodes{} test episodes. The physical starting state is
fixed and task goals are sampled on reset. We retain
\VSevenSnapshots{} simulator snapshots per unit at early, middle, and late
episode stages, with no
overlap of state--task-goal--remaining-time inputs between splits.
Each snapshot has \VSevenGoalSlots{} candidate subgoal slots: the inherited
subgoal, the current $H$-step high-level proposal, the clipped displacement
to the task goal, and uniformly sampled subgoals within the allowed bounds.
For each candidate subgoal, \VSevenActions{} primitive actions are evaluated: the
deterministic low-level action, a clipped Gaussian perturbation, and a
uniformly sampled action. This gives \VSevenBankRows{} return labels per unit.

We restore the full simulator and wrapper state before each branch,
including episode time and random-generator state. Restoration is checked
against repeated action sequences. Each tested action is applied for one
primitive step from the restored state, which may lie anywhere along an
episode. It is followed
by the frozen low-level policy for the remaining subgoal duration; the
frozen high-level policy selects subsequent subgoals. Branches run to the
original episode termination without a critic bootstrap. Rewards use the
training reward scale $.1$, and discount $.99$ is applied only at high-level
decision boundaries. The environment and evaluation policies are deterministic,
so each label is the complete return for its prescribed branch.
The generalization distribution consists of new episode snapshots with
the same fixed starting state and randomized task goals; new starting
states and out-of-distribution tasks are outside this evaluation.

\subsection{Matched-supervision representation comparison}
The direct estimator has twin $Q$ networks with
\VSevenDirectParams{} trainable parameters; the factorized estimator has
twin $V$ networks and a shared signed residual network with
\VSevenSharedParams{} parameters. Both retain separate remaining-time
branches. They start at the same prediction $Q=0$; the factorized model
uses $V=.1$ and $A=.1$.
Both optimize the same action-value labels and a unit-weight policy-value
anchor. The anchor trains $V$ in the factorized estimator and
$Q(s,G,g,\mu(s,g))$ in the direct estimator, using the same policy-return
label. The controlled intervention is the shared-value learning design:
identical supervision is applied through estimator-specific action-value
and policy-value outputs.

Returns for the policy action and its Gaussian perturbation are available
for training; the uniformly sampled action's return is excluded
from every training and anchor loss. Training uses the same minibatch
stream, optimizer settings, and \VSevenUpdates{} gradient updates.
No actor is updated and the supervised-return experiment uses no target
network. Calibration episodes are kept separate from the final test split.
Table~\ref{tab:vseven-representation} separates two tests.
\emph{Separate test trajectories} evaluates all three action types at
states sampled from trajectories excluded from fitting and calibration.
\emph{Actions excluded from fitting} keeps $(s,G,g,t)$ at the training
inputs and tests the uniformly sampled actions whose returns entered no
loss. Their exact full $(s,G,g,a,t)$ inputs are also absent from fitting.
This second test asks whether supervision at other actions improves
the predicted return for the excluded action at the same input.
Figure~\ref{fig:vseven-estimation} shows the resulting learning curves.
Table~\ref{tab:vseven-neural} gives the complete final comparison.

\begin{figure}[t]
\centering
\includegraphics[width=\textwidth]{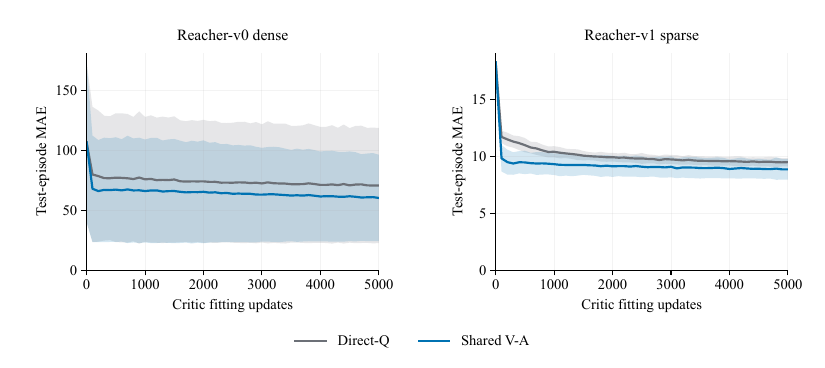}
\caption{\textbf{Value decomposition reduces error under matched supervision.}
MAE on separate test trajectories during fixed-policy fitting, averaged over
\VSevenPolicySeeds{} policy seeds per task. Shading is one standard
deviation across policy seeds. The horizontal axis counts critic-fitting
updates under frozen policies.}
\label{fig:vseven-estimation}
\end{figure}

\subsection{Separating residual output and target constraints}
A second experiment crosses nonnegative versus signed output with
projected versus signed target. These four estimators use a common
nonrecursive residual update, $\lambda=.6$, the same initialization,
balanced remaining-time batches, and target-network updates.
They learn from the original primitive transitions in the training
episodes; the counterfactual rollout labels are reserved for evaluation.
The policies remain fixed, no relabeling or target-action noise is used,
and the rollout bank supplies evaluation only.

For all neural estimators we evaluate the minimum of the twin reconstructed
action values. Subgoal regret compares the rollout value of the
critic-selected candidate with the best candidate's rollout value at the
same snapshot, using the deterministic policy action. It is distinct from
average calibration error. The two representation estimators and the four
projection estimators have different supervision protocols and are
compared within their respective experiment families.

\begin{table}[t]
\centering
\caption{\textbf{Complete neural diagnostic results.}
Mean $\pm$ SD across \VSevenPolicySeeds{} policy seeds, using candidate-slot weighting.
MAE and signed error use separate test trajectories; Action MAE evaluates their
uniformly sampled primitive actions, whose labels are excluded in the
representation comparison. Signed error is prediction minus rollout return.
Regret is measured among the evaluated subgoal candidates.}
\label{tab:vseven-neural}
\small
\setlength{\tabcolsep}{5pt}
\begin{tabular}{llrrrr}
\toprule
Task & Estimator & MAE & Action MAE & Signed error & Regret\\
\midrule
Reacher-v0 dense & Direct-Q & $\VSevenDenseDirectAllMean{}\!\pm\!\VSevenDenseDirectAllSd{}$ & $\VSevenDenseDirectActionMean{}\!\pm\!\VSevenDenseDirectActionSd{}$ & $\VSevenDenseDirectBiasMean{}\!\pm\!\VSevenDenseDirectBiasSd{}$ & $\VSevenDenseDirectRegretMean{}\!\pm\!\VSevenDenseDirectRegretSd{}$ \\
Reacher-v0 dense & Shared V-A & $\VSevenDenseSharedAllMean{}\!\pm\!\VSevenDenseSharedAllSd{}$ & $\VSevenDenseSharedActionMean{}\!\pm\!\VSevenDenseSharedActionSd{}$ & $\VSevenDenseSharedBiasMean{}\!\pm\!\VSevenDenseSharedBiasSd{}$ & $\VSevenDenseSharedRegretMean{}\!\pm\!\VSevenDenseSharedRegretSd{}$ \\
Reacher-v0 dense & Output+target & $\VSevenDenseOutputTargetAllMean{}\!\pm\!\VSevenDenseOutputTargetAllSd{}$ & $\VSevenDenseOutputTargetActionMean{}\!\pm\!\VSevenDenseOutputTargetActionSd{}$ & $\VSevenDenseOutputTargetBiasMean{}\!\pm\!\VSevenDenseOutputTargetBiasSd{}$ & $\VSevenDenseOutputTargetRegretMean{}\!\pm\!\VSevenDenseOutputTargetRegretSd{}$ \\
Reacher-v0 dense & Output only & $\VSevenDenseOutputOnlyAllMean{}\!\pm\!\VSevenDenseOutputOnlyAllSd{}$ & $\VSevenDenseOutputOnlyActionMean{}\!\pm\!\VSevenDenseOutputOnlyActionSd{}$ & $\VSevenDenseOutputOnlyBiasMean{}\!\pm\!\VSevenDenseOutputOnlyBiasSd{}$ & $\VSevenDenseOutputOnlyRegretMean{}\!\pm\!\VSevenDenseOutputOnlyRegretSd{}$ \\
Reacher-v0 dense & Target only & $\VSevenDenseTargetOnlyAllMean{}\!\pm\!\VSevenDenseTargetOnlyAllSd{}$ & $\VSevenDenseTargetOnlyActionMean{}\!\pm\!\VSevenDenseTargetOnlyActionSd{}$ & $\VSevenDenseTargetOnlyBiasMean{}\!\pm\!\VSevenDenseTargetOnlyBiasSd{}$ & $\VSevenDenseTargetOnlyRegretMean{}\!\pm\!\VSevenDenseTargetOnlyRegretSd{}$ \\
Reacher-v0 dense & Signed residual & $\VSevenDenseSignedAllMean{}\!\pm\!\VSevenDenseSignedAllSd{}$ & $\VSevenDenseSignedActionMean{}\!\pm\!\VSevenDenseSignedActionSd{}$ & $\VSevenDenseSignedBiasMean{}\!\pm\!\VSevenDenseSignedBiasSd{}$ & $\VSevenDenseSignedRegretMean{}\!\pm\!\VSevenDenseSignedRegretSd{}$ \\
Reacher-v1 sparse & Direct-Q & $\VSevenSparseDirectAllMean{}\!\pm\!\VSevenSparseDirectAllSd{}$ & $\VSevenSparseDirectActionMean{}\!\pm\!\VSevenSparseDirectActionSd{}$ & $\VSevenSparseDirectBiasMean{}\!\pm\!\VSevenSparseDirectBiasSd{}$ & $\VSevenSparseDirectRegretMean{}\!\pm\!\VSevenSparseDirectRegretSd{}$ \\
Reacher-v1 sparse & Shared V-A & $\VSevenSparseSharedAllMean{}\!\pm\!\VSevenSparseSharedAllSd{}$ & $\VSevenSparseSharedActionMean{}\!\pm\!\VSevenSparseSharedActionSd{}$ & $\VSevenSparseSharedBiasMean{}\!\pm\!\VSevenSparseSharedBiasSd{}$ & $\VSevenSparseSharedRegretMean{}\!\pm\!\VSevenSparseSharedRegretSd{}$ \\
Reacher-v1 sparse & Output+target & $\VSevenSparseOutputTargetAllMean{}\!\pm\!\VSevenSparseOutputTargetAllSd{}$ & $\VSevenSparseOutputTargetActionMean{}\!\pm\!\VSevenSparseOutputTargetActionSd{}$ & $\VSevenSparseOutputTargetBiasMean{}\!\pm\!\VSevenSparseOutputTargetBiasSd{}$ & $\VSevenSparseOutputTargetRegretMean{}\!\pm\!\VSevenSparseOutputTargetRegretSd{}$ \\
Reacher-v1 sparse & Output only & $\VSevenSparseOutputOnlyAllMean{}\!\pm\!\VSevenSparseOutputOnlyAllSd{}$ & $\VSevenSparseOutputOnlyActionMean{}\!\pm\!\VSevenSparseOutputOnlyActionSd{}$ & $\VSevenSparseOutputOnlyBiasMean{}\!\pm\!\VSevenSparseOutputOnlyBiasSd{}$ & $\VSevenSparseOutputOnlyRegretMean{}\!\pm\!\VSevenSparseOutputOnlyRegretSd{}$ \\
Reacher-v1 sparse & Target only & $\VSevenSparseTargetOnlyAllMean{}\!\pm\!\VSevenSparseTargetOnlyAllSd{}$ & $\VSevenSparseTargetOnlyActionMean{}\!\pm\!\VSevenSparseTargetOnlyActionSd{}$ & $\VSevenSparseTargetOnlyBiasMean{}\!\pm\!\VSevenSparseTargetOnlyBiasSd{}$ & $\VSevenSparseTargetOnlyRegretMean{}\!\pm\!\VSevenSparseTargetOnlyRegretSd{}$ \\
Reacher-v1 sparse & Signed residual & $\VSevenSparseSignedAllMean{}\!\pm\!\VSevenSparseSignedAllSd{}$ & $\VSevenSparseSignedActionMean{}\!\pm\!\VSevenSparseSignedActionSd{}$ & $\VSevenSparseSignedBiasMean{}\!\pm\!\VSevenSparseSignedBiasSd{}$ & $\VSevenSparseSignedRegretMean{}\!\pm\!\VSevenSparseSignedRegretSd{}$ \\
\bottomrule
\end{tabular}
\end{table}

The shared representation reduces test-episode MAE in all six policy units;
candidate-selection regret varies across tasks. Joint nonnegative output
and target constraints slightly improve MAE on the dense task and
increase it on the sparse task. These results separate the gain in
action-value estimation from the effects of a conservative constraint
and the resulting subgoal ranking. In the projection experiment,
action-type labels identify evaluation interventions; those actions may
also occur in the separately collected transition bank.

\subsection{Duplicate-candidate sensitivity}
At a high-level boundary, the inherited and current-policy subgoals can
coincide. Each bank contains \VSevenDuplicateGroups{} duplicated
float32 input groups, leaving \VSevenUniqueInputs{} distinct
$(s,G,g,a,t)$ inputs from \VSevenBankRows{} slots.
All duplicate return labels and predictions agree exactly at every
evaluation point during fitting. Candidate sets contain seven or eight distinct
subgoals; different perturbed actions under the same subgoal are not
duplicate inputs.

The primary analysis uses candidate-slot weighting. The sensitivity
analysis merges exactly equal float32 input vectors for prediction error
and exactly equal candidate subgoals within a snapshot for selection
metrics. No merge depends on the outcome.
Table~\ref{tab:vseven-dedup} shows that the representation comparison
retains the same direction in every policy unit.
Policy seeds are the outer replication unit, with query rows and snapshots
nested within them. The three policy seeds per task define the scope of
this comparison.

\begin{table}[t]
\centering
\caption{\textbf{Exact-input deduplication preserves the MAE comparison.}
Shared $V-A$ minus direct-$Q$ test-episode MAE; negative values favor
the shared representation.}
\label{tab:vseven-dedup}
\small
\begin{tabular}{llrr}
\toprule
Task & Policy seed & Slot-weighted difference & Deduplicated difference\\
\midrule
Reacher-v0 dense & \VSevenDenseSeedA{} & \VSevenDenseSeedASlotDifference{} & \VSevenDenseSeedADedupDifference{} \\
Reacher-v0 dense & \VSevenDenseSeedB{} & \VSevenDenseSeedBSlotDifference{} & \VSevenDenseSeedBDedupDifference{} \\
Reacher-v0 dense & \VSevenDenseSeedC{} & \VSevenDenseSeedCSlotDifference{} & \VSevenDenseSeedCDedupDifference{} \\
Reacher-v1 sparse & \VSevenSparseSeedA{} & \VSevenSparseSeedASlotDifference{} & \VSevenSparseSeedADedupDifference{} \\
Reacher-v1 sparse & \VSevenSparseSeedB{} & \VSevenSparseSeedBSlotDifference{} & \VSevenSparseSeedBDedupDifference{} \\
Reacher-v1 sparse & \VSevenSparseSeedC{} & \VSevenSparseSeedCSlotDifference{} & \VSevenSparseSeedCDedupDifference{} \\
\bottomrule
\end{tabular}
\end{table}

\subsection{Finite-MDP maximization-feedback experiment}
We separately test repeated maximization in a stochastic finite MDP with
six states, eight primitive actions, and a high-level duration of three
steps. Subgoals specify repeated primitive actions. Exact dynamic
programming computes the optimal hierarchical values from the transition
matrix and reward function, independently of the learned decomposition.
The factorial comparison crosses a learned maximizing boundary decision
with a fixed optimal boundary policy, and signed with projected residuals.
Both decision conditions have the same optimal-value target. A single
tabular estimator isolates this mechanism from the twin neural
architecture used by full HTB.

The grid uses two, four, or eight candidate subgoals; reward-noise standard
deviations of $0$, $.3$, or $1$; and \VSevenMDPNoiseSeeds{} optimization-noise
seeds per condition. Stochastic transitions remain noisy when reward noise
is zero. All estimators start at the same exact value to isolate subsequent
estimation drift, and use matched samples and update settings.
An additional exact-expectation control removes both noise sources.
The factorial design contains \VSevenMDPRuns{} conditions.

\begin{table}[t]
\centering
\caption{\textbf{Nonnegative projection trades optimism for underestimation.}
Final mean $\pm$ SD across \VSevenMDPNoiseSeeds{} noise seeds with repeated boundary
maximization. Selected error is the critic-selected subgoal's prediction
minus its exact optimal-policy value; MAE averages all evaluated values.
The complete candidate-count and reward-noise grid is shown.}
\label{tab:vseven-mdp}
\small
\begin{tabular}{rrrrrr}
\toprule
 & & \multicolumn{2}{c}{Selected signed error} & \multicolumn{2}{c}{MAE}\\
Candidates & Noise & Signed & Projected & Signed & Projected\\
\midrule
\VSevenMdpTwoZeroCandidates{} & \VSevenMdpTwoZeroNoise{} & $\VSevenMdpTwoZeroSignedSelectedBiasMean{}\!\pm\!\VSevenMdpTwoZeroSignedSelectedBiasSd{}$ & $\VSevenMdpTwoZeroProjectedSelectedBiasMean{}\!\pm\!\VSevenMdpTwoZeroProjectedSelectedBiasSd{}$ & $\VSevenMdpTwoZeroSignedMaeMean{}\!\pm\!\VSevenMdpTwoZeroSignedMaeSd{}$ & $\VSevenMdpTwoZeroProjectedMaeMean{}\!\pm\!\VSevenMdpTwoZeroProjectedMaeSd{}$ \\
\VSevenMdpTwoLowCandidates{} & \VSevenMdpTwoLowNoise{} & $\VSevenMdpTwoLowSignedSelectedBiasMean{}\!\pm\!\VSevenMdpTwoLowSignedSelectedBiasSd{}$ & $\VSevenMdpTwoLowProjectedSelectedBiasMean{}\!\pm\!\VSevenMdpTwoLowProjectedSelectedBiasSd{}$ & $\VSevenMdpTwoLowSignedMaeMean{}\!\pm\!\VSevenMdpTwoLowSignedMaeSd{}$ & $\VSevenMdpTwoLowProjectedMaeMean{}\!\pm\!\VSevenMdpTwoLowProjectedMaeSd{}$ \\
\VSevenMdpTwoHighCandidates{} & \VSevenMdpTwoHighNoise{} & $\VSevenMdpTwoHighSignedSelectedBiasMean{}\!\pm\!\VSevenMdpTwoHighSignedSelectedBiasSd{}$ & $\VSevenMdpTwoHighProjectedSelectedBiasMean{}\!\pm\!\VSevenMdpTwoHighProjectedSelectedBiasSd{}$ & $\VSevenMdpTwoHighSignedMaeMean{}\!\pm\!\VSevenMdpTwoHighSignedMaeSd{}$ & $\VSevenMdpTwoHighProjectedMaeMean{}\!\pm\!\VSevenMdpTwoHighProjectedMaeSd{}$ \\
\VSevenMdpFourZeroCandidates{} & \VSevenMdpFourZeroNoise{} & $\VSevenMdpFourZeroSignedSelectedBiasMean{}\!\pm\!\VSevenMdpFourZeroSignedSelectedBiasSd{}$ & $\VSevenMdpFourZeroProjectedSelectedBiasMean{}\!\pm\!\VSevenMdpFourZeroProjectedSelectedBiasSd{}$ & $\VSevenMdpFourZeroSignedMaeMean{}\!\pm\!\VSevenMdpFourZeroSignedMaeSd{}$ & $\VSevenMdpFourZeroProjectedMaeMean{}\!\pm\!\VSevenMdpFourZeroProjectedMaeSd{}$ \\
\VSevenMdpFourLowCandidates{} & \VSevenMdpFourLowNoise{} & $\VSevenMdpFourLowSignedSelectedBiasMean{}\!\pm\!\VSevenMdpFourLowSignedSelectedBiasSd{}$ & $\VSevenMdpFourLowProjectedSelectedBiasMean{}\!\pm\!\VSevenMdpFourLowProjectedSelectedBiasSd{}$ & $\VSevenMdpFourLowSignedMaeMean{}\!\pm\!\VSevenMdpFourLowSignedMaeSd{}$ & $\VSevenMdpFourLowProjectedMaeMean{}\!\pm\!\VSevenMdpFourLowProjectedMaeSd{}$ \\
\VSevenMdpFourHighCandidates{} & \VSevenMdpFourHighNoise{} & $\VSevenMdpFourHighSignedSelectedBiasMean{}\!\pm\!\VSevenMdpFourHighSignedSelectedBiasSd{}$ & $\VSevenMdpFourHighProjectedSelectedBiasMean{}\!\pm\!\VSevenMdpFourHighProjectedSelectedBiasSd{}$ & $\VSevenMdpFourHighSignedMaeMean{}\!\pm\!\VSevenMdpFourHighSignedMaeSd{}$ & $\VSevenMdpFourHighProjectedMaeMean{}\!\pm\!\VSevenMdpFourHighProjectedMaeSd{}$ \\
\VSevenMdpEightZeroCandidates{} & \VSevenMdpEightZeroNoise{} & $\VSevenMdpEightZeroSignedSelectedBiasMean{}\!\pm\!\VSevenMdpEightZeroSignedSelectedBiasSd{}$ & $\VSevenMdpEightZeroProjectedSelectedBiasMean{}\!\pm\!\VSevenMdpEightZeroProjectedSelectedBiasSd{}$ & $\VSevenMdpEightZeroSignedMaeMean{}\!\pm\!\VSevenMdpEightZeroSignedMaeSd{}$ & $\VSevenMdpEightZeroProjectedMaeMean{}\!\pm\!\VSevenMdpEightZeroProjectedMaeSd{}$ \\
\VSevenMdpEightLowCandidates{} & \VSevenMdpEightLowNoise{} & $\VSevenMdpEightLowSignedSelectedBiasMean{}\!\pm\!\VSevenMdpEightLowSignedSelectedBiasSd{}$ & $\VSevenMdpEightLowProjectedSelectedBiasMean{}\!\pm\!\VSevenMdpEightLowProjectedSelectedBiasSd{}$ & $\VSevenMdpEightLowSignedMaeMean{}\!\pm\!\VSevenMdpEightLowSignedMaeSd{}$ & $\VSevenMdpEightLowProjectedMaeMean{}\!\pm\!\VSevenMdpEightLowProjectedMaeSd{}$ \\
\VSevenMdpEightHighCandidates{} & \VSevenMdpEightHighNoise{} & $\VSevenMdpEightHighSignedSelectedBiasMean{}\!\pm\!\VSevenMdpEightHighSignedSelectedBiasSd{}$ & $\VSevenMdpEightHighProjectedSelectedBiasMean{}\!\pm\!\VSevenMdpEightHighProjectedSelectedBiasSd{}$ & $\VSevenMdpEightHighSignedMaeMean{}\!\pm\!\VSevenMdpEightHighSignedMaeSd{}$ & $\VSevenMdpEightHighProjectedMaeMean{}\!\pm\!\VSevenMdpEightHighProjectedMaeSd{}$ \\
\bottomrule
\end{tabular}
\end{table}

Table~\ref{tab:vseven-mdp} shows a one-sided estimation effect: projection
lowers the estimates and increases MAE throughout this grid. Its
interaction with maximizing versus fixed-policy backups has mixed
directions. The result characterizes the downward constraint on learned
values in both backup settings.
The exact-expectation controls retain the common fixed point.
The comparison evaluates both boundary-policy conditions and their paired
interaction, including underestimation, regret, and calibration against
the learned policy's true value.

\clearpage
\section{Connections to other hierarchical methods}
\label{app:extensions}

HTB separates the choice of a high-level command from its value update.
This distinction connects its evaluator to methods that learn subgoal
representations, discover skills, or plan over latent goals. The connections
below follow from their policy interfaces and the conditional Bellman
construction in Appendix~\ref{app:theory}; they describe possible algorithmic
combinations, not additional performance results.

\paragraph{Learned subgoal representations.}
\citet{nachum2019representation} learn goal representations for a
fixed-duration hierarchy built on HIRO. Their objective concerns which
representation supports effective hierarchical control; HTB concerns how
to evaluate a specified command as its low-level policy changes.
The two mechanisms could be combined by conditioning the time-aware critic
on the learned subgoal. This requires consistent goal semantics in replay,
correct transport of relative goals, and the Markov and action-coverage
conditions of the primitive backup. Adjacency constraints
\citep{zhang2020hrac} and landmark-guided proposals \citep{kim2021higl}
likewise act on subgoal selection and can remain separate from the choice
of subgoal supervised by the critic.

\paragraph{Unsupervised skills.}
DIAYN learns behaviorally diverse skills \citep{eysenbach2019diayn}, while
DADS learns skills with predictable outcomes and a skill-dynamics model
\citep{sharma2020dads}. In a fixed-duration controller over such skills,
a skill identifier can replace the spatial subgoal in the time-aware
value input and remain constant within its execution segment.
The high-level critic would still use accumulated task rewards, independently
of the intrinsic objective used to discover the skills. For stochastic
skill policies, the one-step value target uses the policy expectation in
Equation~\eqref{eq:short-target-expectation}. This provides an evaluation interface
when skills are jointly refined with the high-level policy; frozen skills
do not themselves create low-level-policy nonstationarity.

\paragraph{Recurrent latent-goal hierarchies.}
FeUdal Networks learn directional goals and a recurrent worker conditioned
on a history of manager goals \citep{vezhnevets2017fun}.
This motivates value updates for latent commands, but its pooled-goal
execution differs from holding one goal for a fixed segment.
A time-aware extension must represent the relevant recurrent state and
command history and use the actual command-switching rule in its continuation.
A single spatial goal and countdown do not encode that execution state.

\paragraph{Offline hierarchical policies.}
HIQL learns a latent subgoal policy and a low-level policy from an
action-free offline value function \citep{park2023hiql}.
Although its high-level training target looks several steps ahead,
both policies are queried at every primitive step during execution.
Combining this policy interface with temporally extended HTB evaluation
would therefore require explicit execution boundaries and action-labeled
transitions. Offline action coverage and policy-extraction constraints
remain necessary: the primitive conditional identity does not supply
information about actions absent from the dataset.

\paragraph{Model-based hierarchical control.}
Director selects latent goals at fixed intervals and trains hierarchical
policies through imagined trajectories in a learned world model
\citep{hafner2022director}. Its execution clock offers a natural interface
for remaining-time-conditioned goal evaluation. A replay-based extension
could attach an HTB-style critic to these goals while accounting for the
worker's action distribution and model error. Director's own actor updates
use on-policy imagined trajectories; stationarity of real primitive
transitions does not remove changes in a learned dynamics model.

\paragraph{Options with learned termination.}
Option-Critic already learns primitive-action-conditioned intra-option
values together with option policies and termination functions
\citep{bacon2017optioncritic}. Its termination rule determines whether
the next step continues the option or invokes the high-level policy.
This is the same distinction made explicit by HTB's remaining-time state
for scheduled boundaries. Extending the mixed backups to learned
termination would replace the deterministic boundary switch by the
appropriate termination-weighted continuation, preserving the chosen
task reward and discount convention.

\end{document}